\documentclass{article} %
\usepackage{iclr2027_conference,times}
\usepackage{graphicx}
\usepackage{overpic}
\usepackage{xcolor}

\newcommand{\authormark}[1]{\raisebox{.35ex}{\includegraphics[height=1.8ex]{figures/icons/#1.pdf}}}
\newcommand{\DKmark}{\authormark{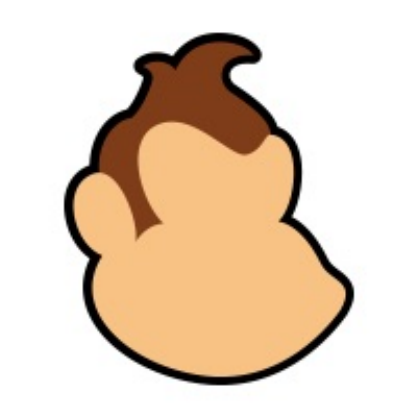}}
\newcommand{\Lmark}{\authormark{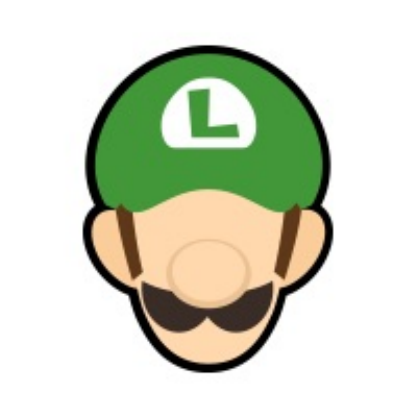}}
\newcommand{\Mmark}{\authormark{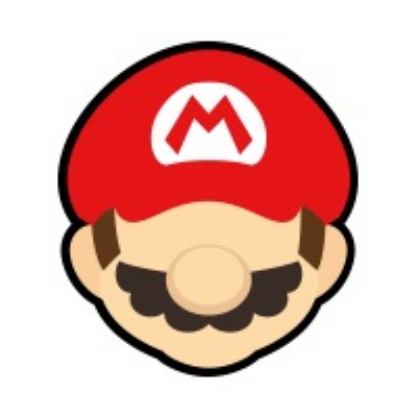}}
\newcommand{\equalcontributionmarks}{%
    \DKmark\,\Lmark\,\Mmark
}

\usepackage{amsmath,amsfonts,bm}

\def\eqref#1{equation~\ref{#1}}

\def\1{\bm{1}}

\DeclareMathAlphabet{\mathsfit}{\encodingdefault}{\sfdefault}{m}{sl}
\SetMathAlphabet{\mathsfit}{bold}{\encodingdefault}{\sfdefault}{bx}{n}

\newcommand{\R}{\mathbb{R}}

\DeclareMathOperator*{\argmax}{arg\,max}
\DeclareMathOperator*{\argmin}{arg\,min}

\DeclareMathOperator{\Tr}{Tr}

\usepackage{amsmath}
\usepackage{xspace}
\usepackage{dsfont}

\providecommand{\Ind}{\mathds{1}}
\providecommand{\norm}[1]{\left\lVert #1 \right\rVert}
\providecommand{\abs}[1]{\left\lvert #1 \right\rvert}
\providecommand{\supp}{\operatorname{supp}}
\providecommand{\Dir}{\operatorname{Dir}}
\providecommand{\diag}{\operatorname{diag}}
\providecommand{\ip}[2]{\left\langle #1, #2 \right\rangle}

\providecommand{\sg}{\operatorname{sg}}

\DeclareRobustCommand{\saeparams}{\boldsymbol{\theta}}

\DeclareRobustCommand{\SAE}{\textsc{sae}\xspace}
\DeclareRobustCommand{\CLS}{\textsc{cls}\xspace}
\DeclareRobustCommand{\SAEs}{\textsc{sae}s\xspace}
\DeclareRobustCommand{\SpatialSAE}{\textsc{Spatial-sae}\xspace}
\DeclareRobustCommand{\SpatialSAEs}{\textsc{Spatial-sae}s\xspace}

\DeclareRobustCommand{\LRH}{\textsc{lrh}\xspace}
\DeclareRobustCommand{\MFLRH}{\textsc{mf-lrh}\xspace}

\DeclareRobustCommand{\MAP}{\textsc{map}\xspace}
\DeclareRobustCommand{\MRF}{\textsc{mrf}\xspace}

\DeclareRobustCommand{\DINO}{\textsc{DINOv2}\xspace}
\DeclareRobustCommand{\ADETwentyK}{\textsc{ADE20K}\xspace}
\DeclareRobustCommand{\COCOStuff}{\textsc{COCO-Stuff}\xspace}
\DeclareRobustCommand{\RDP}{\textsc{RDP}\xspace}
\DeclareRobustCommand{\MRFs}{\textsc{mrf}s\xspace}

\usepackage{amsthm}
\usepackage{marginnote}
\usepackage{etoolbox}
\usepackage{enumitem}
\usepackage{wrapfig}
\usepackage{booktabs}
\usepackage{colortbl}
\usepackage{tabularx}
\usepackage{multirow}
\usepackage{caption}
\usepackage{placeins}
\usepackage{algorithm}
\usepackage{algpseudocode}
\usepackage[skins]{tcolorbox}
\floatstyle{plain}
\newfloat{rubric}{tbp}{lor}
\floatname{rubric}{Rubric}

\definecolor{linkcolor}{HTML}{ca8a04} %
\PassOptionsToPackage{hyphens}{url}
\usepackage[pagebackref,breaklinks,colorlinks,allcolors=linkcolor]{hyperref}
\usepackage{cleveref}
\crefname{rubric}{Rubric}{Rubrics}
\Crefname{rubric}{Rubric}{Rubrics}
\usepackage{url}
\usepackage{etoc}

\newcommand{\intuition}[2]{%
\begin{tcolorbox}[
    enhanced,
    colback=white,
    colframe=black,
    boxrule=0.8pt,
    arc=2mm,
    left=1.1em,
    right=1.1em,
    top=0.65em,
    bottom=0.65em,
    before skip=0.85em,
    after skip=0.95em
]
\noindent\textbf{#1.}\quad #2
\end{tcolorbox}%
}

\theoremstyle{plain}
\newtheorem{theorem}{Theorem}[section]
\newtheorem{proposition}[theorem]{Proposition}
\newtheorem{lemma}[theorem]{Lemma}
\newtheorem{corollary}[theorem]{Corollary}

\theoremstyle{definition}
\newtheorem{definition}[theorem]{Definition}

\theoremstyle{remark}
\newtheorem{remark}[theorem]{Remark}

\crefname{proposition}{Proposition}{Propositions}
\Crefname{proposition}{Proposition}{Propositions}
\crefname{lemma}{Lemma}{Lemmas}
\Crefname{lemma}{Lemma}{Lemmas}
\crefname{corollary}{Corollary}{Corollaries}
\Crefname{corollary}{Corollary}{Corollaries}
\crefname{definition}{Definition}{Definitions}
\Crefname{definition}{Definition}{Definitions}
\crefname{assumption}{Assumption}{Assumptions}
\Crefname{assumption}{Assumption}{Assumptions}
\crefname{remark}{Remark}{Remarks}
\Crefname{remark}{Remark}{Remarks}
\crefname{fact}{Fact}{Facts}
\Crefname{fact}{Fact}{Facts}
\AtBeginEnvironment{proposition}{\crefalias{theorem}{proposition}}
\AtBeginEnvironment{lemma}{\crefalias{theorem}{lemma}}
\AtBeginEnvironment{corollary}{\crefalias{theorem}{corollary}}
\AtBeginEnvironment{definition}{\crefalias{theorem}{definition}}
\AtBeginEnvironment{assumption}{\crefalias{theorem}{assumption}}
\AtBeginEnvironment{remark}{\crefalias{theorem}{remark}}
\AtBeginEnvironment{fact}{\crefalias{theorem}{fact}}

\title{
\begin{center}
    The Independence Prior of SAEs \\ Fragments Visual Concepts
\end{center}
}

\author{
\parbox{\linewidth}{\centering 
        Tommaso Mencattini$^{1,2,}$\DKmark
        \quad
            Giorgos Nikolaou$^{1,2,}$\Lmark \quad
            Donato Crisostomi\Mmark  \\[1ex]
            Thomas Fel$^{3}$ \quad
            Francesco Montagna$^{1}$ \quad
            Emanuele Rodolà$^{4,5}$ \quad
            Francesco Locatello$^{1}$
} \\[3ex]
{\parbox{\linewidth}{%
        \centering
        $^1$ISTA \qquad
        $^2$EPFL \qquad
        $^3$Goodfire \qquad
        $^4$Sapienza University of Rome \qquad
        $^5$Paradigma
      }
}
}

\iclrfinalcopy %
\begin{document}
\etocdepthtag.toc{main}

\maketitle

\ificlrfinal
\begingroup
\renewcommand{\thefootnote}{}
\footnotetext{%
    \begin{tabular}[t]{@{}l@{}}
        \equalcontributionmarks\enspace
        Equal contribution; author order settled via the now SOTA Mario Kart method.\\
        \phantom{\DKmark\,\Lmark\,}\Mmark\enspace Independent researcher; work done while at Sapienza University of Rome.
    \end{tabular}%
}
\endgroup
\fi

\begin{abstract}
Sparse Autoencoders (\SAEs) decompose model activations into sparse combinations of interpretable dictionary atoms. Although \SAEs are grounded in the Linear Representation Hypothesis (\LRH), their objective smuggles in an additional prior: concepts across patches are treated as independent, an assumption clearly violated by natural images and by the activations they induce. We therefore specialize \LRH to vision through the Markov-Field Linear Representation Hypothesis (\MFLRH), which adds the missing spatial dependencies to the \LRH assumptions. We thus propose \SpatialSAE as an amortized MAP estimator under the \MFLRH. \SpatialSAE consistently outperforms standard \SAEs in concept recovery and interpretability. Across four variants, it achieves a $96\%$ average win rate on synthetic concept recovery and improves interpretability on \DINO activations, at a reconstruction cost concentrated in high spatial frequencies.
\end{abstract}

\section{Introduction}
\label{sec:introduction}

Sparse autoencoders (\SAEs) have become the standard tool for extracting
human-interpretable concepts from the activations of foundation models
\citep{bricken2023monosemanticity,cunningham2024saes,gemmascope,qwenscope}. Applied to vision
models, an \SAE learns a dictionary of features, each meant to correspond
to a visual concept
\citep{fel2025rabbit_hull,stevens2025interpretable}. In practice, however,
the features recovered from images' latent representations are often \emph{spatially
fragmented}: a single concept is split across several dictionary atoms, its
activation map has holes inside the object it should cover, and neighboring
patches of the same region are explained by different atoms
(\Cref{fig:wild-qual}) \citep{chanin2024absorption,ayonrinde2024mdl,chanin2025hedging,bhalla2026sparse}. Fragmented features are harder to interpret and
harder to trust, since the atom that nominally detects a concept fails to
fire on parts of it. %

\begin{figure}[t]
    \centering
    \includegraphics[width=0.9\linewidth]{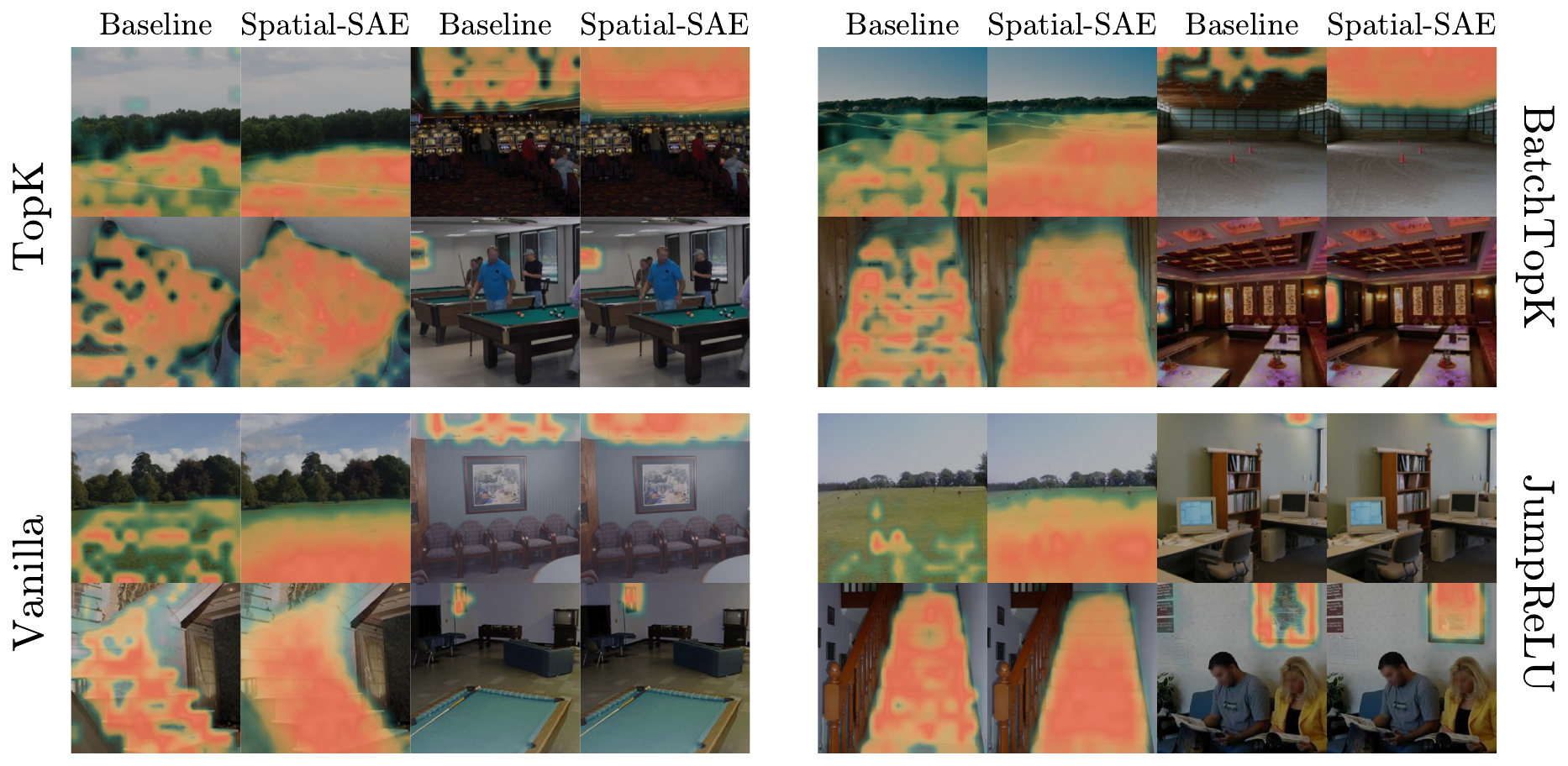}
    \par\medskip
    \includegraphics[width=0.85\linewidth]{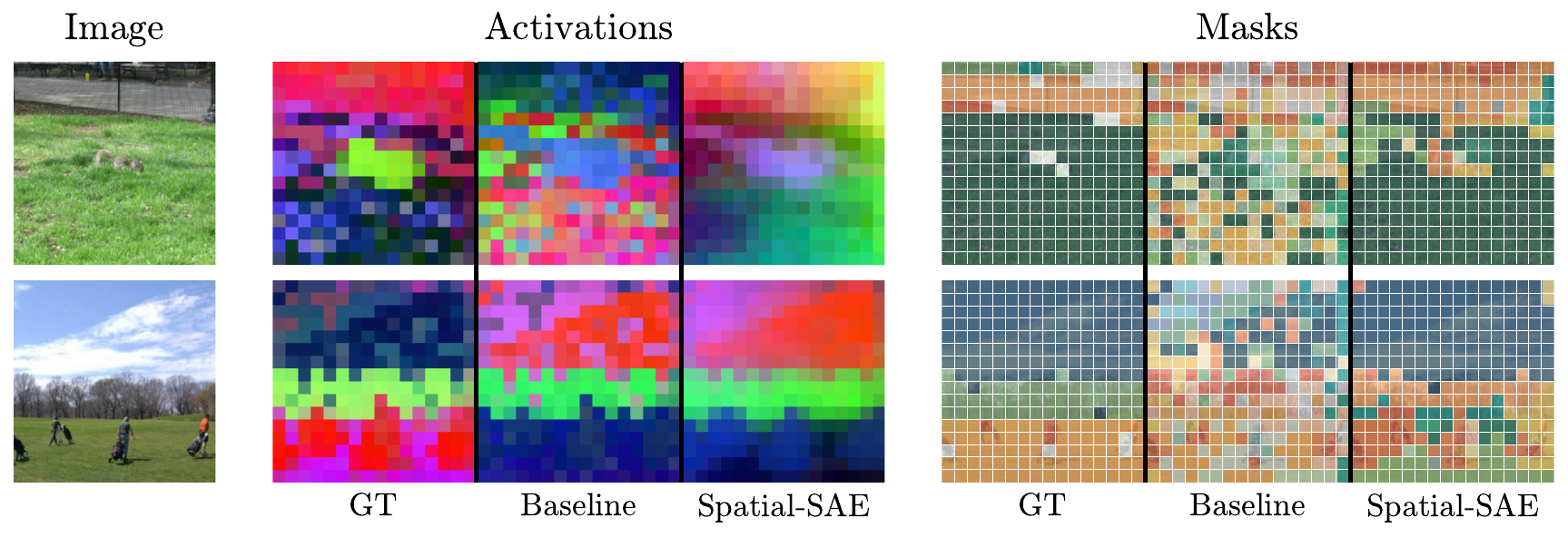}
    \caption{\textbf{Spatial organization at three levels.}
    \textbf{Top:} Baseline and \SpatialSAE feature maps across four sparsifiers,
    illustrating feature support for grass, ceiling, staircase, and painting
    (reading left to right, then top to bottom).
    \textbf{Bottom:} PCA-RGB views of original and reconstructed activations,
    alongside ground-truth segments and strongest-atom assignments.
    Together, these views assess feature support, spatial atom organization,
    and reconstruction structure.}
    \label{fig:wild-qual}
\end{figure}

Two properties of the \SAE objective make fragmentation likely. The first is the rate--distortion--polysemanticity (\RDP) tradeoff~\citep{mencattini2026rate}. At a fixed sparsity budget, the dictionary that best reconstructs
activations can be the one whose atoms are polysemantic, so pieces of a
concept may reconstruct better than the concept itself. The second, which we address in this paper, is an implicit
prior. \SAEs are justified by the Linear Representation Hypothesis (\LRH),
under which each activation is a sparse linear combination of dictionary
directions \citep{elhage2022superposition,park2024lrh,fel2025rabbit_hull}, and \SAE training
amortizes maximum a posteriori (\MAP) inference under this model
\citep{lubana2025priors_time}. 
\LRH, however, describes one
activation at a time; the \SAE objective adds an inductive bias via a prior according to which neighboring patches select their
concepts \emph{independently} \citep{lubana2025priors_time}. This assumption
does not hold for vision-model activations. Objects, textures, and regions
extend across many adjacent patches, and we show that this dependence carries
over from the image to the activations of a frozen vision transformer (\Cref{fig:patch-locality}). An objective that
cannot express these dependencies has no spatial criterion for preferring a feature that covers a region over
one that covers scattered pieces of it, and existing remedies
\citep{fel2025archetypal,saemp}
leave this assumption in place.

We replace the independence prior with one that matches vision-model
activations. The \textbf{Markov-Field Linear Representation Hypothesis
(\MFLRH)} keeps every assumption of \LRH and adds a single spatial one: the
code of a patch is the sum of a \emph{smooth} component, shared with its
neighbors on the patch grid, and a sparse \emph{innovation}\footnote{We borrow the term \emph{innovation} from the state-space
and Kalman-filtering literature, where it denotes the part of an observation
that cannot be predicted from past observations \citep{kailath1968innovations,kalman1960new}.
} component that
carries genuinely local evidence such as boundaries and small objects.

Under \MFLRH, \MAP inference is a graph-regularized
sparse-coding problem, and amortizing it yields \textbf{\SpatialSAE}, whose
encoder estimates the smooth code of a patch from its neighbors and
explains the remainder locally. The spatial prior enters
only through the training loss, so \SpatialSAE is a drop-in extension of
existing \SAE families, where fragmented explanations, previously free,
now carry a cost. On synthetic data with known dictionaries, concept recovery
improves in $288$ out of $300$ settings; on frozen \DINO activations, features
become more concept-aligned and region-shaped for all four sparsifiers we
test, at a reconstruction cost that is concentrated in the high spatial
frequencies where concepts do not live. In practice, \Cref{fig:wild-qual} shows the outcome: tokenwise \SAEs learn
incomplete, incoherent, and uninterpretable features whose supports are fragmented across the
image; \SpatialSAE learns concept-aligned features that are interpretable
at a glance.

\noindent\textbf{Contributions.}
\textbf{(i)} We show that tokenwise \SAEs complete \LRH with a
patch-independence prior, and that vision-model activations violate it
(\Cref{sec:spatial-structure}).
\textbf{(ii)} We introduce \MFLRH, a spatial specialization of \LRH, and
derive \SpatialSAE as amortized \MAP inference under it (\Cref{sec:spatialsae}).
\textbf{(iii)} We show that \SpatialSAE recovers the ground-truth concepts of synthetic data more reliably than standard \SAEs, and learns more concept-aligned features on real \DINO activations (\Cref{sec:wild}).

\section{Background}
\label{sec:background}

\paragraph{Notation.}
Matrices are denoted by uppercase bold symbols $(\mathbf X)$, vectors by lowercase bold
symbols $(\mathbf x)$, scalars by unbolded letters ($x$), and sets by uppercase calligraphic symbols ($\mathcal X$).

\paragraph{Linear Representation Hypothesis.}
The Linear Representation Hypothesis (\LRH) models foundation model activations as sparse
linear combinations of directions in an overcomplete, approximately
non-redundant dictionary
\citep{elhage2022superposition, park2024lrh, LRHorigin, fel2025rabbit_hull}.
Following \citet{saemp,fel2025rabbit_hull}, one can define it formally: 

\begin{definition}[Linear Representation Hypothesis]
\label{def: lrh}
The activation $\mathbf{h}\in\R^d$ satisfies \LRH if there exists a dictionary $\mathbf{D}\in\R^{c\times d}$ and a sparse coefficient vector $\mathbf{z}\in\R^c$ such that
$\mathbf{h}\approx\mathbf{z}\mathbf{D}$ and
\begin{equation}
\label{eq:background-lrh}
\left\{
\begin{array}{lll}
\textnormal{(i) Overcompleteness:}
& c\gg d, \\[0.2em]
\textnormal{(ii) Quasi-orthogonality:}
& \displaystyle
  \mu(\mathbf{D})
  :=
  \max_{a\neq b}
  \abs{\mathbf{d}_a\mathbf{d}_b^\top}
  \leq \varepsilon,
& \norm{\mathbf{d}_a}_2=1
  \quad \forall a, \\[0.2em]
\textnormal{(iii) Sparsity:}
& \abs{\supp(\mathbf{z})}
  \leq k\ll c.
\end{array}
\right.
\end{equation}
\end{definition}
Applied row-wise, this gives
$\mathbf{H}\approx\mathbf{Z}\mathbf{D}$.
\looseness=-1Concept extraction under \LRH is therefore an inverse problem: recover the
dictionary directions $\mathbf{D}$ and the sparse activation matrix
$\mathbf{Z}$ \citep{hindupur2025projecting_assumptions}.

\paragraph{Sparse coding and SAEs.}
Sparse coding addresses this inverse problem by jointly optimizing the
dictionary and one sparse code per activation
\citep{olshausen1996,olshausen1997,elad2010}:
\begin{equation}
\label{eq:background-sparse-coding}
(\mathbf{Z}^{\star},\mathbf{D}^{\star})
\in
\argmin_{\mathbf{Z},\mathbf{D}}
\frac{1}{n}
\sum_{i=1}^{n}
\left[
    \norm{\mathbf{h}_i-\mathbf{z}_i\mathbf{D}}_2^2
    +
    \lambda\mathcal{R}(\mathbf{z}_i)
\right].
\end{equation}
where $\mathcal R$ promotes sparsity, for example $\mathcal{R}(\mathbf{z}) = \|\mathbf{z}\|_1$. Solving Equation~\ref{eq:background-sparse-coding}
requires a separate optimization for every activation. An \SAE amortizes this
cost by training a shared parametric encoder $f_{\saeparams}$ to predict $\mathbf{z}_i = f_{\saeparams}(\mathbf{h}_i)$
in a forward pass \citep{bricken2023monosemanticity, cunningham2024saes}. The encoder replaces repeated optimization with a forward pass, but may return
a suboptimal sparse code; its approximation cost is measured by the
amortization gap \citep{oneill2024compute_optimal}.

\section{Spatial Structure in Vision-Model Activations}
\label{sec:spatial-structure}

The \LRH specifies how each activation is composed, but not how the codes of
different patches should relate. In \Cref{sec:patch-independence}, we identify
the patch-independence prior implicit in tokenwise \SAEs, examine its
consequences for dictionary learning, and contrast it with the spatial
dependencies observed in vision-model activations. In \Cref{sec:mf-lrh},
we introduce \MFLRH, specializing \LRH with a spatial prior over
complementary smooth and innovation fields.

\subsection{\SAEs Add a Patch-Independence Prior}
\label{sec:patch-independence}
The \LRH calls for sparsity and overcompleteness, but the \SAE objective
also encodes an implicit prior of independence between tokens.
\citet{lubana2025priors_time} make this prior explicit through a
probabilistic interpretation of \SAE training. They show that, under
a Gaussian noise model, minimizing the \SAE objective is
equivalent to finding the \textit{maximum a posteriori (\MAP) estimate of the
latent codes} under a sparsity-favoring prior,
$p(\mathbf{Z}) \propto \exp(-\lambda \mathcal{R}(\mathbf{Z}))$ (e.g., $\mathcal R$ could be the $\ell_1$ norm).
Crucially, the sparsity penalty in
Equation~\ref{eq:background-sparse-coding} applies separately to each
token: $\mathcal{R}(\mathbf{Z}) = \sum_i \mathcal{R}(\mathbf{z}_i)$.
The corresponding prior therefore factorizes as
$p(\mathbf{Z}) = \prod_{i=1}^{p} p(\mathbf{z}_i)$,
treating different tokens' latent codes as independent
(Proposition~\ref{prop:app-spatial-independence}).
Thus, beyond encouraging accurate reconstruction and sparse
representations, the objective \emph{imposes a prior that does not account for dependencies across tokens}.
\citet{lubana2025priors_time} argue that this assumption is detrimental for language models, since an i.i.d.\ prior over tokens ignores their temporal structure.

\noindent Similarly, we argue that this prior is mismatched to
dictionary learning on vision models. Although activations may already encode context,
the tokenwise \SAE objective contains no explicit interaction between
neighboring patch codes. For a fixed encoder and dictionary, permuting the
rows of $\mathbf{H}$ before encoding simply permutes $\mathbf{Z}$:
each activation selects the same atoms and incurs
the same cost. The objective therefore assigns no additional cost based on
whether a patch's explanation fits with its neighborhood. Thus, the
independence prior supplies no criterion for preferring explanations that fit
together spatially, making it a mismatched inductive prior for obtaining
spatially consistent explanations of vision activations.

\begin{proposition}[Permutation invariance of the tokenwise objective]
\label{prop:perm-invariance}
Let $\mathbf{P}\in\R^{p\times p}$ be any permutation matrix. For the
row-additive tokenwise \SAE objective defined above,
$\mathcal{L}_{\SAE}
(\mathbf{P}\mathbf{Z},\mathbf{D} \, ; \, \mathbf{P}\mathbf{H})
=
\mathcal{L}_{\SAE}
(\mathbf{Z},\mathbf{D};\mathbf{H})$.
Consequently, a tokenwise \SAE objective does not distinguish a coherent patch
layout from an arbitrarily permuted layout, provided the activations and codes
are permuted together.
\end{proposition}

The proof is given in Appendix~\ref{app:vanilla-map}.
To see why this invariance may matter, consider how recognizing a tree
helps us interpret a nearby green region as foliage rather than grass:
the interpretations of the parts support one another
\citep{wagemans2012gestalt,biederman1972scenes,palmer1975context}.
Just as we interpret image regions, an \SAE interprets their activation
vectors through dictionary features. If these activations retain spatial
dependencies, the features explaining neighboring activations should
inform those selected for the current one. Although each activation
already carries context, the objective does not coordinate
these choices. We next show that these dependencies persist in vision-model activations, motivating replacing the mismatched independence prior implied by \SAEs.

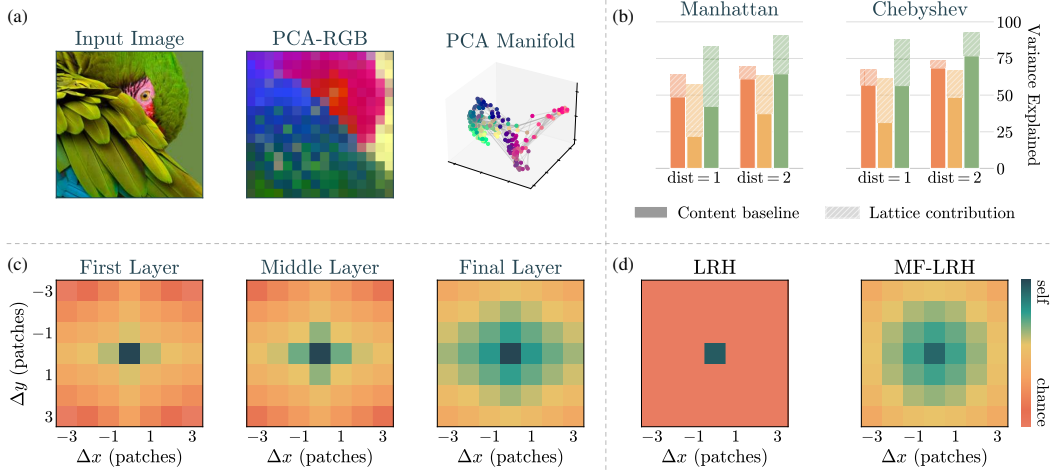
\begin{figure*}[t]
    \centering
    \begin{overpic}[width=\textwidth]{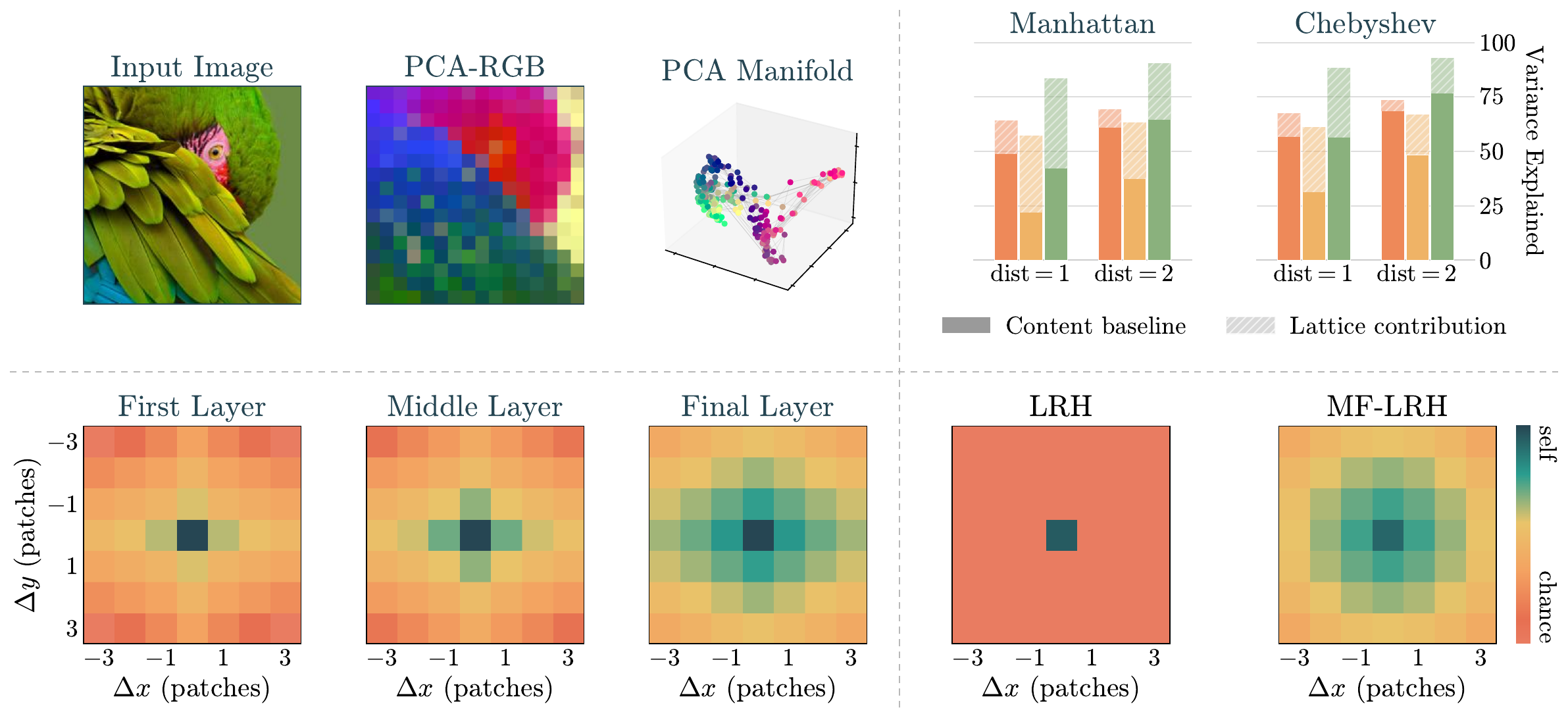}
        \put(0.7,43){\scriptsize (a)}
        \put(58,43){\scriptsize (b)}
        \put(0.7,19.6){\scriptsize (c)}
        \put(58,19.6){\scriptsize (d)}
    \end{overpic}
    \caption{
    \textbf{Patch activations retain spatial structure.}
    (a) PCA-RGB and 3D PCA views of \DINO activations show spatial
    organization. (b) Variance explained by
    Manhattan and Chebyshev neighborhoods. Solid bars are within-image shuffled
    controls; hatching is the additional contribution of the true layout.
    (c) Measured offset-similarity kernels across depth.
    (d) Measured offset-similarity kernels from an \LRH model and the
    \MFLRH model used in our synthetic experiments.
    }
    \label{fig:patch-locality}
\end{figure*}

\paragraph{Patch activations are spatially dependent.} \label{par:patch}
Figure~\ref{fig:patch-locality} provides three complementary pieces of
evidence against independence: 
\textbf{(a) The visualizations reveal recognizable spatial patterns}.
Coloring patches by their activations reveals coherent image regions.
Even after averaging activations across images at each patch position,
their PCA projection still resembles the image grid.
\textbf{(b)} True neighbors explain more of a patch's activation than equally sized neighborhoods of randomly chosen patches from the same image. The shuffled neighborhoods measure what shared image content alone explains, and the \emph{lattice contribution} is the additional variance explained only when neighbors are taken from the true grid layout. This gap shows that spatial proximity provides information beyond shared image content.
\textbf{(c) Nearby patches remain more similar than distant ones}
from the first to the final layer, showing that this spatial structure
persists throughout the network. These observations agree with prior work showing that self-supervised ViT
features organize around object regions \citep{caron2021dino} and that ViTs
trained by gradient descent provably learn to group spatially related patches
through their positional attention \citep{jelassi2022spatial}.
Thus, neighboring patch activations provide predictive evidence about one
another that a factorized prior cannot exploit. Next, we introduce
an extension of \LRH{} that incorporates this dependence while preserving a
separate channel for genuinely local variation.

\subsection{Markov-Field Linear Representation Hypothesis}
\label{sec:mf-lrh}
To replace the independence prior, we need a model of how patch codes depend on each other. In images, this dependence is mostly local: an object or region typically spans several adjacent patches, so a concept present at one patch is likely present at its neighbors.
Markov random fields (\MRFs) have long formalized such local interactions in
classical computer vision \citep{prince2012computer_vision}, and they combine naturally with sparsity: an \MRF prior can favor sparse supports that cluster spatially over equally sparse but scattered ones \citep{cevher2008sparse_mrf}.

\begin{figure}[t]
    \centering
    \includegraphics[width=\textwidth]{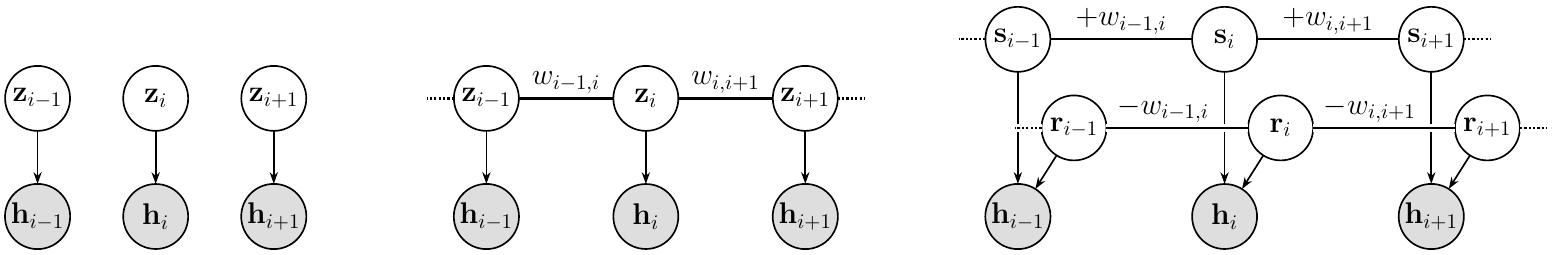}
    \caption{\textbf{From independent codes to attractive--repulsive coupling.}
    \textit{Left:} \LRH{}+\SAE{} treats patch codes $\mathbf z_i$
    independently.
    \textit{Middle:} attractive coupling of the full codes promotes spatial
    persistence but can oversmooth local departures.
    \textit{Right:} \MFLRH decomposes each code as
    $\mathbf z_i=\mathbf s_i+\mathbf r_i$, coupling neighboring smooth codes
    attractively and neighboring innovation codes repulsively.}
    \label{fig:mflrh-graph}
\end{figure}

We therefore introduce the \textbf{Markov-Field Linear Representation Hypothesis (\MFLRH)}, which splits each patch's code into two complementary fields
(\Cref{fig:mflrh-graph}).
Natural images have stronger low-frequency components
\citep{field1987relations}, motivating a spatially coupled
\textbf{smooth field} $\mathbf{S}$ for content shared across patches.
Fine details, however, need not persist spatially, motivating an
\textbf{innovation field} $\mathbf{R}$ for local departures.
We combine them additively as $\mathbf{Z}=\mathbf{S}+\mathbf{R}$,
both for simplicity and to mirror the structure of the residual stream.

\begin{definition}[Markov-Field Linear Representation Hypothesis]
\label{def:mf-lrh}
Let $\mathcal{G}=(\mathcal{V},\mathcal{E},\mathbf{W})$ be a patch graph over the
$p$ image patches with symmetric nonnegative weights
$w_{ij}$, and let $\mathbf{L} \in \mathbb{R}^{p \times p}$ be its graph
Laplacian.
A patch-activation matrix $\mathbf{H}\in\R^{p\times d}$ satisfies \MFLRH if
there exist a dictionary
$\mathbf{D}=(\mathbf{d}_1,\ldots,\mathbf{d}_c)^\top \in\R^{c\times d}$,
latent fields $\mathbf{S},\mathbf{R}\in\R^{p\times c}$, and
$\mathbf{Z}=\mathbf{S}+\mathbf{R}$ such that
(i) the \LRH conditions (\Cref{def: lrh}) hold,
(ii) $\mathbf{H}\approx(\mathbf{S}+\mathbf{R})\mathbf{D}$, and
(iii) $\mathbf{S}$ and $\mathbf{R}$ are respectively attractive and
repulsive pairwise \MRFs on $\mathcal G$. 
For inference, we relax hard sparsity and use the following continuous prior:
\begin{equation}
\label{eq:mflrh-prior}
\resizebox{\linewidth}{!}{$\displaystyle
p(\mathbf{S},\mathbf{R}\mid\mathcal{G})
\propto \exp \Bigg\{
-\lambda_S\norm{\mathbf{S}}_1
-\lambda_R\norm{\mathbf{R}}_1
-\frac{\mu}{2}\Big[
    \kappa_{\mathcal{G}}\|\mathbf{R}\|_{\mathrm{F}}^2
    +\sum_{\{i,j\}\in\mathcal{E}}w_{ij}
    \left(
        \norm{\mathbf{s}_i-\mathbf{s}_j}_2^2
        -\norm{\mathbf{r}_i-\mathbf{r}_j}_2^2
    \right)
\Big]\Bigg\},
$}
\end{equation}
where $\kappa_{\mathcal G}=\lambda_{\max}(\mathbf L)$,
$\lambda_S,\lambda_R>0$, and $\mu\geq0$.
\end{definition}

Conditions (i)--(ii) retain the sparse linear representation of \LRH while
decomposing its code into smooth and innovation fields. Condition (iii)
replaces patch independence with complementary local dependencies:
neighboring smooth codes favor agreement, whereas neighboring innovation
codes favor contrast (\Cref{prop:mflrh-interaction-signs}).
In \Cref{eq:mflrh-prior}, the $\ell_1$ terms promote sparsity,
while the edge terms realize these attractive--repulsive interactions,
and the magnitude term $\kappa_{\mathcal G}\|\mathbf{R}\|_{\mathrm F}^2$
keeps the repulsive quadratic bounded below
(\Cref{cor:app-stabilized-repulsion}). This reflects the intended role of
$\mathbf R$ as the local variation left outside the spatially persistent
component, consistent with our observations in \Cref{fig:patch-locality}b.

At $\mu=0$, the spatial prior disappears and the model factorizes across
patches; ordinary sparse coding is recovered by additionally taking
$\mathbf{S}\equiv\mathbf{0}$ and $\mathbf{R}=\mathbf{Z}$.
Thus, \MFLRH retains \LRH while replacing its implicit patch-independence
completion with an attractive--repulsive Markov prior.

\begin{figure}[t]
    \centering
    \includegraphics[width=\linewidth]{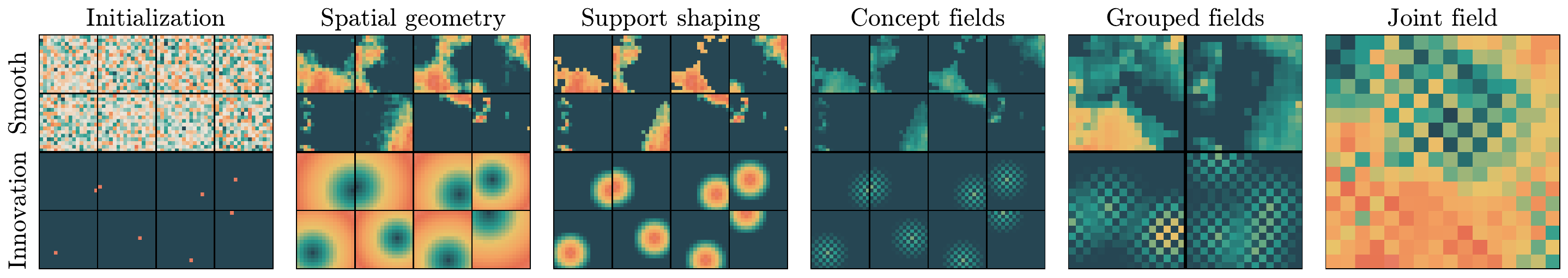}
    \caption{\textbf{Generating synthetic data.} From left to right:
    initialization (Gaussian noise / sampled center), spatial geometry (graph
    diffusion and rectification / distance from center), support shaping
    (quantile thresholding / compact radial envelope), concept construction
    (peak normalization / bipartite masking), grouped aggregation, and the
    final joint field $\mathbf{S} + \mathbf{R}$. In each parenthetical pair,
    the smooth-stream operation is listed first and the innovation-stream
    operation second.}
    \label{fig:synth-smooth}
\end{figure}

\section{Spatial Sparse Autoencoders}
\label{sec:spatialsae}

Having motivated \MFLRH as a model of spatial dependencies between
patch codes, we now develop a method that incorporates this prior.
In \Cref{sec:spatialsae-map}, we extend the Bayesian interpretation
of tokenwise \SAEs to derive \SpatialSAE from \MAP inference under \MFLRH.
In \Cref{sec:ssae-encoder}, we describe the encoder architecture of \SpatialSAE that
makes this objective usable in practice, when the precise decomposition of
$\mathbf Z$ into $(\mathbf S, \mathbf R)$ is not known.

\subsection{\SpatialSAE Is Amortized \MAP Inference Under \MFLRH}
\label{sec:spatialsae-map}

To turn \MFLRH into a practical method, we first need an objective for
inferring its latent codes. We derive it through the Bayesian approach
of \Cref{sec:spatial-structure}, keeping the observation model unchanged
and adopting the spatial prior of \Cref{def:mf-lrh}. The resulting \MAP
objective combines reconstruction and sparsity with two additional terms
that encourage spatial agreement in the smooth field and spatial
disagreement in the innovation field.

\begin{theorem}[Spatial sparse coding is \MAP inference for \MFLRH]
\label{thm:spatial-map}
Let $\mathbf{H}=(\mathbf{S}+\mathbf{R})\mathbf{D}+\boldsymbol{\Psi}$ with
$\boldsymbol{\Psi}_i\sim\mathcal{N}(0,\sigma^2\mathbf{I}_d)$ independent
across patches, and $(\mathbf{S},\mathbf{R})$ following the \MFLRH prior of
\Cref{def:mf-lrh}. For fixed $\mathbf{D}$ and an externally specified
graph $\mathcal{G}$ with Laplacian $\mathbf{L}$, let
$\widetilde{\mathbf{L}}=\kappa_{\mathcal{G}}\mathbf{I}_p-\mathbf{L}$.
Absorbing $\sigma^2$ into the weights, a \MAP estimate of
$(\mathbf{S},\mathbf{R})$ solves the spatial sparse-coding problem:
\begin{equation}
\label{eq:spatial-map}
\resizebox{\linewidth}{!}{$\displaystyle
    (\mathbf{S}^\star,\mathbf{R}^\star)
    \in\argmin_{\mathbf{S},\mathbf{R}}\;
    \tfrac{1}{2}\norm{\mathbf{H}-(\mathbf{S}+\mathbf{R})\mathbf{D}}_F^2
    + \lambda_S\norm{\mathbf{S}}_1
    + \lambda_R\norm{\mathbf{R}}_1
    + \tfrac{\mu}{2}\left[
        \Tr(\mathbf{S}^\top\mathbf{L}\mathbf{S})
        + \Tr(\mathbf{R}^\top\widetilde{\mathbf{L}}\mathbf{R})
    \right].
$}
\end{equation}
\end{theorem}

\vspace{-0.8em}
The proof is in Appendix~\ref{app:continuous-mf-lrh-to-spatial-sae}.
The trace identities in \Cref{lem:app-laplacian-trace} convert the prior's
attractive--repulsive interactions
into the two quadratic terms in \Cref{eq:spatial-map}.
The local dependencies and innovation magnitude penalty encoded by the prior
yield complementary spectral regularization on the patch graph.
The smooth term penalizes variation across graph edges, suppressing high-frequency 
components and promoting agreement between neighboring codes. Conversely, the 
innovation term uses the positive-semidefinite operator $\widetilde{\mathbf L}$ to 
suppress low-frequency components and promote spatial contrast.
This objective connects to graph-regularized sparse coding
\citep{yankelevsky2020theoretical}; related sparsity--smoothness tradeoffs
appear in compressed manifold modes \citep{neumann2014compressed}. Finally, we can define the \SpatialSAE:

\begin{definition}[Spatial Sparse Autoencoder]
\label{def:spatial-sae}
Let $\mathcal{X} = \left\{ \mathbf{H}^{(n)} \in \mathbb{R}^{p \times d} \right\}_{n=1}^N$
be a dataset of patch activations from $N$ images
and $\mathbf{L} \in \mathbb{R}^{p \times p}$ be a fixed patch-graph Laplacian and $\widetilde{\mathbf{L}}=\kappa_{\mathcal{G}}\mathbf{I}_p-\mathbf{L}$.
A \SpatialSAE pairs an encoder $F_{\saeparams}$ predicting both fields,
$(\mathbf{S}_{\saeparams},\mathbf{R}_{\saeparams})=F_{\saeparams}(\mathbf{H} \, ; \, \mathbf{L})$, with a
decoder dictionary $\mathbf{D}\in\R^{c\times d}$ reconstructing
$\widehat{\mathbf{H}}=(\mathbf{S}_{\saeparams}+\mathbf{R}_{\saeparams})\mathbf{D}$, and is
trained by:
\begin{equation}
\label{eq:spatialsae-objective}
\resizebox{\linewidth}{!}{$\displaystyle
    \min_{\saeparams,\mathbf{D}}\;\frac{1}{\abs{\mathcal{X}}}
    \sum_{\mathbf{H}\in\mathcal{X}}\Bigl[
      \tfrac{1}{2}\norm{
        \mathbf{H}-(\mathbf{S}_{\saeparams}+\mathbf{R}_{\saeparams})\mathbf{D}
      }_F^2
    + \lambda_S\norm{\mathbf{S}_{\saeparams}}_1
    + \lambda_R\norm{\mathbf{R}_{\saeparams}}_1
    + \tfrac{\mu}{2}\Bigl(
        \Tr\bigl(\mathbf{S}_{\saeparams}^\top\mathbf{L}\mathbf{S}_{\saeparams}\bigr)
        + \Tr\bigl(\mathbf{R}_{\saeparams}^\top\widetilde{\mathbf{L}}\mathbf{R}_{\saeparams}\bigr)
      \Bigr)\Bigr].
$}
\end{equation}
\end{definition}

\vspace{-0.8em}
Like ordinary sparse coding, Equation~\ref{eq:spatial-map} requires
solving an optimization problem for each image. \SpatialSAE replaces
this with a learned encoder, amortizing \MAP inference under \MFLRH
just as a vanilla \SAE amortizes ordinary sparse coding. Furthermore,
the spatial prior enters through the training loss, making \SpatialSAE a
\textit{plug-and-play extension} compatible with the main \SAE families 
(details in Appendix~\ref{app:real-impl}).
Lastly, we prove that \textit{the graph terms break the permutation invariance
of tokenwise objectives} (\Cref{prop:perm-invariance}) while preserving invariance to patch indexing.
With the graph fixed, the remaining permutation symmetries are those
that preserve weighted adjacency, as formalized in
Proposition~\ref{prop:app-layout-not-order}.
The \MFLRH prior thus encodes the spatial dependencies documented in
\Cref{sec:spatial-structure} but ignored by the patch-independent \SAE prior:
the cost of explaining a patch now depends on how its neighbors are
explained. The graph terms encourage agreement between neighboring
smooth codes and disagreement between neighboring innovation codes.
We next describe how the encoder estimates the smooth and innovation
fields from real-world activations.

\subsection{Estimating the Smooth and Innovation Fields}
\label{sec:ssae-encoder}

\begin{wrapfigure}[12]{r}{0.455\textwidth}
  \vspace{-1.1em}
  \setlength{\columnwidth}{\linewidth}
  \captionsetup{font=normalfont,labelfont=bf,aboveskip=1pt,belowskip=1pt}
  \hrule height 0.8pt
  \vspace{2pt}
  \captionof{algorithm}{\SpatialSAE's forward pass.}
  \label{alg:ssae-forward}
  \hrule height 0.8pt
  \vspace{2pt} 
  \begin{algorithmic}[1]
      \Require Patch activations $\mathbf{H} \in \mathbb{R}^{p \times d}$.
      \For{each patch $i$}
          \State $\mathbf{N} \gets \Call{Stack}{\{\mathbf{H}_{j} : (i,j) \in \mathcal{E}\}}$
          \State $\mathbf{q} \gets \overline{\mathbf{N}} \mathbf{E}_S, \quad \mathbf{KV} \gets \mathbf{N}\mathbf{E}_S$
          \State $\mathbf{S}_i \gets
              \Call{Sparsify}{\mathrm{Att}(\mathbf{q},\mathbf{KV},\mathbf{KV})}$
          \State $\mathbf{R}_i \gets
              \Call{Sparsify}{(\mathbf{H}_i - \mathbf{S}_i \mathbf{D}_S) \mathbf{E}_R}$
          \State $\widehat{\mathbf{H}}_i \gets
              \mathbf{S}_i \mathbf{D}_S + \mathbf{R}_i \mathbf{D}_R$
      \EndFor
      \State \Return $\mathbf{S}$, $\mathbf{R}$, and $\widehat{\mathbf{H}}$
  \end{algorithmic}
  \vspace{2pt}
  \hrule
  \vspace{-0.5\baselineskip}
\end{wrapfigure}
In real-world models, we observe only the activations $\mathbf{H}$.
We therefore need an expressive encoder that infers the smooth and innovation fields
$(\mathbf{S},\mathbf{R})$ from $\mathbf{H}$ alone.
Our approach follows a
simple idea: use neighboring patches to estimate the shared content, then
explain what remains locally. We encode the neighbors' mean activation as an
attention query, summarizing the local context. Attention uses this query to
weight and combine the neighboring activations into the smooth code
$\mathbf{S}_i$, without seeing the target patch. A second encoder estimates
the innovation code $\mathbf{R}_i$ from the residual of the smooth
reconstruction. This builds on the spatial predictability observed in
\Cref{fig:patch-locality}; \Cref{alg:ssae-forward} summarizes the procedure.
In practice, we partition the dictionary between the smooth and
innovation streams, each supported only on its assigned atoms.
\Cref{alg:ssae-forward} uses the corresponding active code blocks and
sub-dictionaries $\mathbf{D}_S$ and $\mathbf{D}_R$.
We compute the residual as $\mathbf{H}_i-\sg(\mathbf{S}_i\mathbf{D}_S)$,
where the stop-gradient prevents the innovation stream from backpropagating
through the smooth reconstruction. 

\section{Experiments}
\label{sec:wild}

We now test whether a spatial prior improves the features obtained by SAEs. Before presenting the experiments, we need to be clear
about what should be evaluated, because the frequently used criterion,
reconstruction \citep{fel2025rabbit_hull,saemp}, is not the right one \citep{chanin2024absorption,chanin2025hedging}. An \SAE communicates each activation
through a sparse bottleneck, and ideally the resulting code is sparse,
faithful, and interpretable at once; standard objectives,
however, optimize only the first two and never reward monosemanticity
directly. \citet{mencattini2026rate} formalize this tension as the
\textbf{rate--distortion--polysemanticity (\RDP) tradeoff}: at a fixed
sparsity budget, better reconstruction can be achieved with more
polysemantic dictionaries, and thus, it does not
necessarily mean a better dictionary. Since our goal is to recover
interpretable concepts rather than minimize reconstruction alone,
we treat interpretability as the primary metric and report reconstruction
as a secondary tradeoff. Under \RDP, a loss in
reconstruction when improving interpretability is therefore expected.

We do so in two settings. In \Cref{sec:spatialsae-synthetic}, we use synthetic
data with known dictionaries and codes to measure polysemanticity
\citep{mencattini2026rate} and concept recovery under spatial structure. Crucially, data is not trivially sampled from the \MFLRH prior itself, but from a spatially structured generator intended to mimic the smooth and local variation we observed in vision activations (\Cref{fig:patch-locality}). In \Cref{sec:wild-eval}, we test whether the same inductive bias helps on frozen \DINO activations. In both settings and for all four sparsifiers we test, \SpatialSAE learns more interpretable dictionaries and more coherent
feature maps, at a modest reconstruction cost.

\subsection{Concept Recovery on Synthetic Data}
\label{sec:spatialsae-synthetic}

To measure concept recovery directly, we generate data whose concepts are
known and compare \SpatialSAE with a standard \SAE on them.

\paragraph{Synthetic spatial benchmark.}
We add spatial structure to existing sparse linear benchmarks
\citep{chanin2026synthsaebench} by constructing two nonnegative latent streams,
$\mathbf S$ and $\mathbf R$, on a $16\times16$ grid, both associated with a
shared dictionary $\mathbf D$. Each stream is formed by normalizing and
aggregating concept-specific subfields, but with different spatial geometry:
smooth subfields are obtained by diffusing, rectifying, and thresholding
Gaussian noise, whereas innovation subfields are compact radial envelopes
centered at sampled locations and restricted to alternating grid sites
(see \Cref{fig:synth-smooth} for an overview).
Our data-generating process deliberately does not sample exactly from the
\MFLRH prior, since in that setting an advantage for the corresponding \MAP
estimator would be unsurprising. Instead, we test whether the prior's inductive
bias remains useful under misspecification, using data that retain the
qualitative spatial structure observed in real concept fields: concepts tend
to persist across neighboring patches, while localized innovations capture
patch-specific deviations. Appendix~\ref{app:synthetic-dgp} gives the full
generating process and its parameters.

\begin{figure}[t]
    \centering
    \begin{overpic}[width=\linewidth]{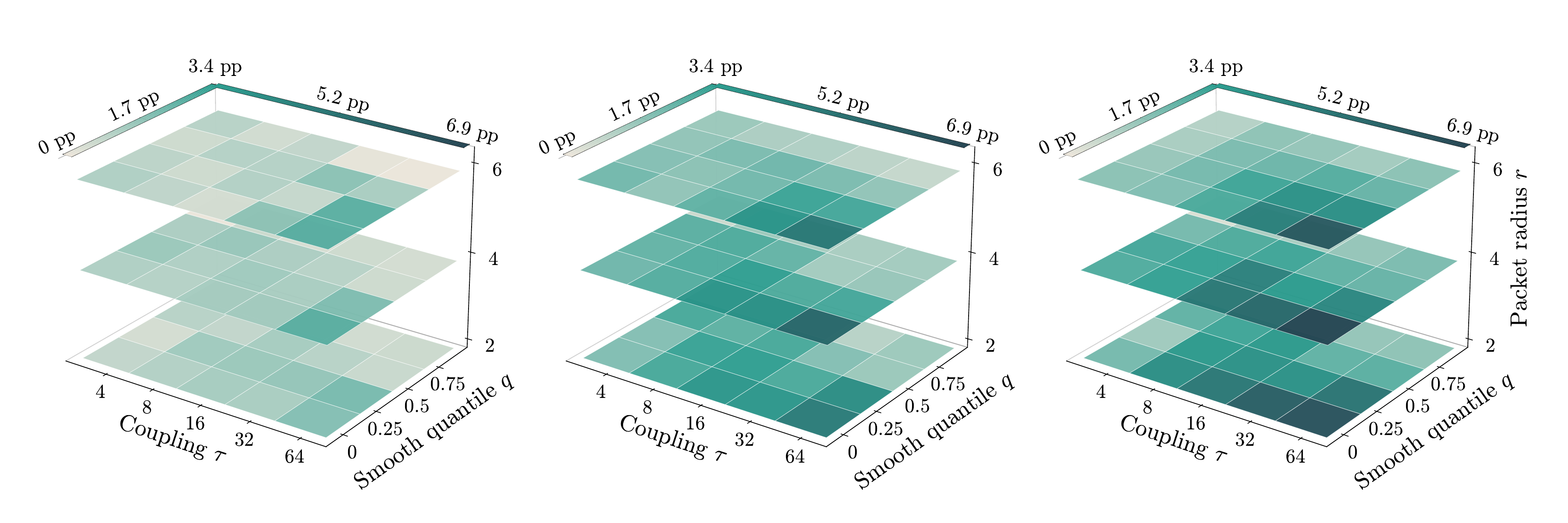}
        \put(0.7,28){\scriptsize (a)}
        \put(34.0,28){\scriptsize (b)}
        \put(67.3,28){\scriptsize (c)}
    \end{overpic}
    \caption{\textbf{Isolating the prior's effect on synthetic data.}
    We compare spectral regularization at (a) $\mu=3\times10^{-4}$,
    (b) $\mu=10^{-3}$, and (c) $\mu=3\times10^{-3}$ with the same two-stream
    model at $\mu=0$. Color intensity shows F1 gain; the prior helps across \textbf{all}
    smoothness and innovation settings.}
    \label{fig:synth-prior-f1-gain}
\end{figure}
\paragraph{Main experimental setting.}
Since the dictionary and codes are known, we can directly evaluate
concept recovery. We measure \textsc{MCC}, support F1, and held-out $R^2$
following \citet{chanin2026synthsaebench}, and polysemanticity
(\textsc{Poly.}) following \citet{mencattini2026rate}. Our main experiment
studies how \SpatialSAE behaves as we vary the DGP parameters governing the
smoothness and innovation structure of the generated concept fields. We
perform a grid sweep over these parameters, yielding $300$ matched runs per
architecture, and report the win rate of \SpatialSAE over the corresponding baseline in
\Cref{tab:synth-wins}. Appendix~\ref{app:synthetic-experiments} reports
additional sweeps over the generating-dictionary geometry and sparsity budget.
It also ablates the two components of \SpatialSAE, the smooth--innovation
decomposition and spectral regularization, to quantify the contribution of
each.

\begin{wraptable}[10]{r}{0.45\textwidth}
    \vspace{-1em}
    \centering
    \footnotesize
    \setlength{\tabcolsep}{3pt}
    \begin{tabular*}{\linewidth}{@{\extracolsep{\fill}}ccccc@{}}
        \toprule
        Sparsifier & \textsc{MCC} & \textsc{Poly.} & F1 & $R^2$ \\
        \midrule
        TopK      & 100\%  & 100\%  & 100\%  & 0\% \\
        BatchTopK & 100\%  & 100\%  & 100\%  & 0\% \\
        ReLU      & 85.7\% & 87.7\% & 85.0\% & 0\% \\
        JumpReLU  & 98.7\% & 96.0\% & 100\%  & 100\% \\
        \bottomrule
    \end{tabular*}
    \captionsetup{font=footnotesize,skip=4pt}
    \caption{
    Win rate of \SpatialSAE against baseline \SAEs on matched pairs
    (same generator configuration and seed);
    \SpatialSAE wins the three identification metrics, conceding
    reconstruction.
    }
    \label{tab:synth-wins}
\end{wraptable}
\paragraph{Results.}
For ReLU, TopK, and BatchTopK, we observe the
rate--distortion--polysemanticity tradeoff predicted by
\citet{mencattini2026rate}: \SpatialSAE improves concept recovery and reduces
polysemanticity while conceding reconstruction. It wins every \textsc{MCC},
F1, and polysemanticity comparison for TopK and BatchTopK, and wins $85.7\%$,
$85.0\%$, and $87.7\%$ of the corresponding ReLU comparisons, respectively,
while never winning on $R^2$. JumpReLU is the notable exception: \SpatialSAE
wins every F1 and $R^2$ comparison, and $98.7\%$ of \textsc{MCC} and
$96.0\%$ of polysemanticity comparisons. Thus, the spatial prior need not
inherently trade reconstruction for concept recovery; with a compatible
sparsifier, it can move the model to a better region of the \RDP frontier.
These gains are not explained by the two-stream decomposition alone; relative
to an otherwise identical unregularized two-stream model, the spatial prior
strictly increases F1 across the same parameter sweep and, for an appropriately tuned $\mu$, \SpatialSAE improves up to $6.9$ percentage points
(\Cref{fig:synth-prior-f1-gain}).

\subsection{Concept Discovery in Vision-Model Activations}
\label{sec:wild-eval}

\Cref{sec:spatialsae-synthetic} tests \SpatialSAE on synthetic data
containing both smooth and local concepts. We next turn to real vision-model
activations, which need not follow \MFLRH exactly and for which neither the
ground-truth dictionary nor the latent codes are observed.

\paragraph{Setup and evaluation.}
We compare \SpatialSAE against standard \SAEs on frozen \DINO ViT-B/14
activations from \ADETwentyK across four sparsifiers with matched dictionary
widths and sparsity budgets. We use dense annotations to measure
semantic and spatial alignment, and a blinded autointerpretability protocol to
assess feature interpretability. For each class, we identify its best
detector and normalize its F1 by that of the corresponding linear probe
(\emph{Normalized F1}); we measure spatial alignment by comparing within- and
across-segment code similarity at matched distances (\emph{Region Coherence});
and we assess feature interpretability through blinded LLM ratings of activation
support (\emph{Map Quality}) and consistency across images
(\emph{Monosemanticity}). Complete training
details, general evaluation
protocols, and additional \COCOStuff results appear
in Appendix~\ref{sec:experiment-details}.

\begin{figure}[t]
    \centering
    \includegraphics[width=\linewidth]{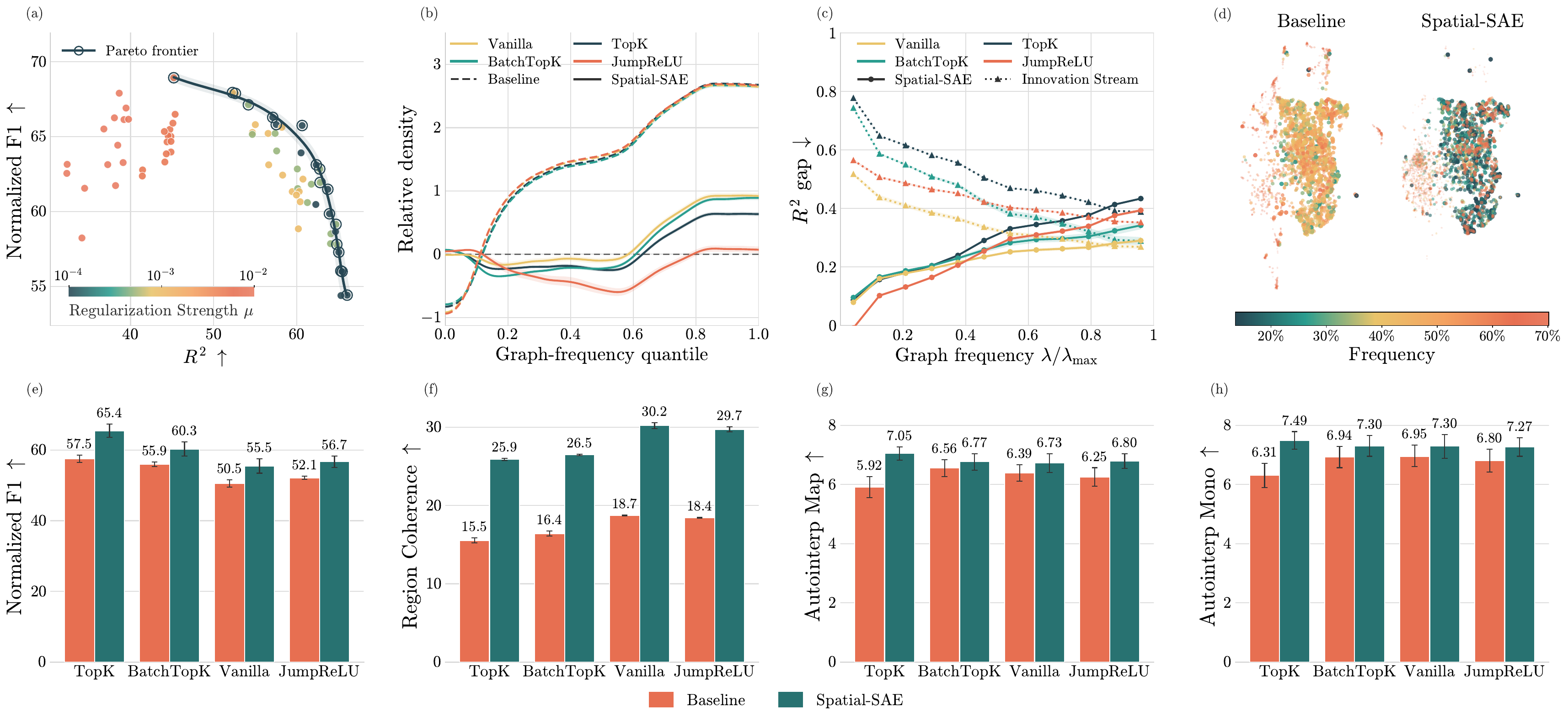}
    \caption{\textbf{Concept alignment and spectral organization.}
    (a) Normalized F1--$R^2$ frontier over \SpatialSAE
    configurations.
    (b) Code spectrum relative to concept masks: zero denotes
    agreement with the masks, dashed curves the baseline, and solid curves
    \SpatialSAE.
    (c) The baseline's $R^2$ advantage over \SpatialSAE (solid) and its
    innovation stream alone (dotted), by frequency.
    (d) TopK decoder atoms in a joint \textsc{umap}, colored by the
    graph frequency of their activation maps.
    (e--h) Normalized F1, Region Coherence, and automated map-quality
    and monosemanticity ratings. Error bars show one stddev over
    5 seeds in (e--f) and $95\%$ class-bootstrap intervals from
    one seed in (g--h).
    }
    \label{fig:wild-main}
\end{figure}

\paragraph{Concept alignment and spatial organization.}
\SpatialSAE improves all interpretability measures for all
sparsifiers (\Cref{fig:wild-main}e--h). For TopK, Normalized F1 increases from
$57.5$ to $65.4$ and Region Coherence from $15.5$ to $25.9$, while Map Quality
and Monosemanticity improve for every sparsifier. These gains trade off against
reconstruction, exactly as \RDP predicts; varying the regularization strength
traces a clear Normalized-F1--$R^2$ frontier rather than producing an
isolated operating point (\Cref{fig:wild-main}a). The spectral analysis then connects
this tradeoff to fragmentation: low graph frequencies encode concepts that
persist across neighboring patches, yet every baseline underrepresents them,
instead increasingly overrepresenting higher frequencies. Conversely, \SpatialSAE brings the code
spectrum substantially closer to the concept mask spectrum
(\Cref{fig:wild-main}b). The reconstruction tradeoff is concentrated in the
same high-frequency regime: the baseline's $R^2$ advantage over the combined \SpatialSAE code
grows from $0.11$ at low frequencies to $0.37$ at high frequencies, where its
spectral excess is largest. Meanwhile, its advantage over the innovation stream
alone narrows, consistent with that stream specializing in local, high-frequency
variation (\Cref{fig:wild-main}c). At the feature level,
\SpatialSAE contains distinct low- and high-frequency populations, whereas the
baseline is concentrated at higher frequencies and lacks a comparable
low-frequency population (\Cref{fig:wild-main}d). The qualitative results in
\Cref{fig:wild-qual} show the same structure: \SpatialSAE better preserves
regional organization in PCA-RGB reconstructions, produces more stable
strongest-atom assignments within segments, and learns class-best maps with
significantly more continuous support.

\section{Conclusion}
\label{sec:conclusion}
We showed that tokenwise \SAEs implicitly assume independence between patch codes, an assumption that vision-model activations violate. To replace it, we introduced \MFLRH, which specializes the \LRH by coupling smooth codes across neighboring patches while retaining sparse innovations for local details, and derived \SpatialSAE as amortized \MAP inference under this hypothesis. On synthetic data, \SpatialSAE improves concept recovery and reduces polysemanticity; on real vision-model activations, where these quantities cannot be measured without a ground-truth dictionary, it improves every proxy we evaluate. 
\SpatialSAE therefore moves along the \RDP frontier toward interpretability, trading reconstruction for more interpretable features and coherent activation maps.%

\newpage
\subsection*{AI Use Statement}
In this work, we used generative AI tools to implement methods (training, evaluation, and data-generation code), to design and provide feedback on experiments, and to assist in the writing of proofs. A language model is also part of our evaluation methodology: Claude Haiku 4.5 serves as the blinded rater in the automated interpretability protocol (\Cref{sec:wild-eval} and Appendix~\ref{app:real-autointerp}), and the full rater instructions are included in the supplementary material. We have not used generative AI tools to develop the theoretical framework, formulate mathematical claims, propose hypotheses, or interpret results, and no data were generated, translated, or cleaned by generative AI (the synthetic datasets are produced by our own generation code). Additionally, we used generative AI tools for brainstorming, to create and edit software code, create scientific figures, draft parts of the paper, and edit the paper for readability. We have reviewed all AI-assisted work. All AI-generated code was checked and tested for correctness by the authors, every AI-assisted proof was verified step by step by the authors, every experimental result was verified by the authors, and all AI-assisted text was reviewed and revised by the authors against the underlying derivations and results. We take responsibility for the final content of this work, including text, claims or artifacts produced with the aid of generative AI.

\subsection*{Ethics Statement}
This work does not involve human subjects: no human annotators or participants were recruited, and all interpretability ratings were produced by a language model under a blinded protocol (Appendix~\ref{app:real-autointerp}). All experiments use publicly available models and datasets (DINOv2, ADE20K, and COCO-Stuff), used under their respective licenses and terms of use for research. Some of these images contain people; we do not attempt to identify individuals, and we release no new images or annotations. The backbone and datasets may encode demographic and cultural biases, and dictionary features learned from their activations can inherit them; our method aims to make such internal structure easier to inspect, not to correct it. We foresee no direct harmful applications beyond those of interpretability research on vision models in general, and we believe more faithful concept extraction primarily benefits the auditing and understanding of deployed models.

\subsection*{Reproducibility Statement}
We will publicly release the code for training and evaluating \SpatialSAE and all baselines, the synthetic data generator and datasets, and the trained models upon publication. The implementation extends the public \texttt{overcomplete} library, and the software versions and compute used are listed in Appendix~\ref{sec:experiment-details}. On the theoretical side, all assumptions are stated in the main text and the complete proofs of \Cref{prop:perm-invariance} and \Cref{thm:spatial-map}, along with the supporting results, are given in Appendices~\ref{app:vanilla-map}--\ref{app:continuous-mf-lrh-to-spatial-sae}, with the required background on patch graphs and Markov random fields in Appendices~\ref{app:notation}--\ref{app:undirected-graphical-models}. For the synthetic benchmark, Appendix~\ref{app:synthetic-dgp} specifies the data-generating process, the dictionary and stream construction, the sweeps, and the training and evaluation protocol. For natural images, we use only the publicly available \DINO backbone and the \ADETwentyK and \COCOStuff datasets; Appendix~\ref{sec:experiment-details} describes activation extraction and preprocessing, the patch graph, all architectures and hyperparameters (\Cref{app:real-impl}), the evaluation metrics (\Cref{app:metrics}), and the model-selection and seeding protocol (\Cref{app:real-ablations}). The automated interpretability protocol, including panel construction, rater assignment, and aggregation, is described in Appendix~\ref{app:real-autointerp}, and the full rater instructions are included in the supplementary material.

\bibliography{iclr2027_conference}

@article{yankelevsky2020theoretical,
title = {Theoretical guarantees for graph sparse coding},
journal = {Applied and Computational Harmonic Analysis},
volume = {49},
number = {2},
pages = {698-725},
year = {2020},
issn = {1063-5203},
doi = {10.1016/j.acha.2019.03.003},
url = {https://www.sciencedirect.com/science/article/pii/S106352031830232X},
author = {Yael Yankelevsky and Michael Elad}
}

@article{neumann2014compressed,
author = {Neumann, T. and Varanasi, K. and Theobalt, C. and Magnor, M. and Wacker, M.},
title = {Compressed Manifold Modes for Mesh Processing},
journal = {Computer Graphics Forum},
volume = {33},
number = {5},
pages = {35-44},
doi = {10.1111/cgf.12429},
url = {https://onlinelibrary.wiley.com/doi/abs/10.1111/cgf.12429},
eprint = {https://onlinelibrary.wiley.com/doi/pdf/10.1111/cgf.12429},
year = {2014}
}

@misc{bricken2023monosemanticity,
  author = {Trenton Bricken and others},
  title = {Towards Monosemanticity: Decomposing Language Models with Dictionary Learning},
  howpublished = {Transformer Circuits Thread},
  year = {2023},
  url = {https://transformer-circuits.pub/2023/monosemantic-features/index.html}
}

@inproceedings{caron2021dino,
  author = {Mathilde Caron and Hugo Touvron and Ishan Misra and Herv{\'e} J{\'e}gou and Julien Mairal and Piotr Bojanowski and Armand Joulin},
  title = {Emerging Properties in Self-Supervised Vision Transformers},
  booktitle = {Proceedings of the IEEE/CVF International Conference on Computer Vision (ICCV)},
  year = {2021},
pages     = {9630--9640},
  doi       = {10.1109/ICCV48922.2021.00951}
}

@inproceedings{cunningham2024saes,
  author = {Robert Huben and Hoagy Cunningham and Logan Riggs Smith and Aidan Ewart and Lee Sharkey},
  title = {Sparse Autoencoders Find Highly Interpretable Features in Language Models},
  booktitle = {International Conference on Learning Representations (ICLR)},
  year = {2024},
  url = {https://openreview.net/forum?id=F76bwRSLeK}
}

@book{elad2010,
  author = {Michael Elad},
  title = {Sparse and Redundant Representations: From Theory to Applications in Signal and Image Processing},
  publisher = {Springer},
  year = {2010},
  doi = {10.1007/978-1-4419-7011-4}
}

@inproceedings{fel2025archetypal,
  author = {Thomas Fel and Ekdeep Singh Lubana and Jacob Prince and Matthew Kowal and Victor Boutin and Isabel Papadimitriou and Binxu Wang and Martin Wattenberg and Demba Ba and Talia Konkle},
  title = {Archetypal {SAE}: Adaptive and Stable Dictionary Learning for Concept Extraction in Large Vision Models},
  booktitle = {International Conference on Machine Learning (ICML)},
  year = {2025},
series    = {Proceedings of Machine Learning Research},
  volume    = {267},
  pages     = {16543--16572},
  publisher = {PMLR}
}

@inproceedings{jelassi2022spatial,
  author = {Samy Jelassi and Michael E. Sander and Yuanzhi Li},
  title = {Vision Transformers Provably Learn Spatial Structure},
  booktitle = {Advances in Neural Information Processing Systems (NeurIPS)},
  volume = {35},
  year = {2022},
pages     = {37822--37836},
  doi       = {10.52202/068431-2741}
}

@article{oquab2023dinov2,
  author = {Maxime Oquab and others},
  title = {{DINOv2}: Learning Robust Visual Features without Supervision},
  journal = {Transactions on Machine Learning Research},
  year = {2024},
  url = {https://openreview.net/forum?id=a68SUt6zFt}
}

@inproceedings{darcet2024registers,
  author = {Timoth\'{e}e Darcet and Maxime Oquab and Julien Mairal and Piotr Bojanowski},
  title = {Vision Transformers Need Registers},
  booktitle = {International Conference on Learning Representations (ICLR)},
  year = {2024},
  url = {https://openreview.net/forum?id=2dnO3LLiJ1}
}

@article{olshausen1996,
  author = {Bruno A. Olshausen and David J. Field},
  title = {Emergence of Simple-Cell Receptive Field Properties by Learning a Sparse Code for Natural Images},
  journal = {Nature},
  volume = {381},
  number = {6583},
  pages = {607--609},
  year = {1996},
  doi = {10.1038/381607a0}
}

@article{olshausen1997,
  author = {Bruno A. Olshausen and David J. Field},
  title = {Sparse Coding with an Overcomplete Basis Set: A Strategy Employed by {V1}?},
  journal = {Vision Research},
  volume = {37},
  number = {23},
  pages = {3311--3325},
  year = {1997}
}

@inproceedings{lubana2025priors_time,
  title     = {Priors in Time: Missing Inductive Biases for Language Model Interpretability},
  author    = {Ekdeep Singh Lubana and Can Rager and Sai Sumedh R. Hindupur and Val{\'e}rie Costa and Greta Tuckute and Oam Patel and Sonia Krishna Murthy and Thomas Fel and Daniel Wurgaft and Eric J. Bigelow and Johnny Lin and Demba Ba and Martin Wattenberg and Fernanda Viegas and Melanie Weber and Aaron Mueller},
  booktitle = {International Conference on Learning Representations (ICLR)},
  year      = {2026},
  url       = {https://arxiv.org/abs/2511.01836},
  pages     = {42231--42290}
}

@inproceedings{fel2025rabbit_hull,
  title     = {Into the Rabbit Hull: From Task-Relevant Concepts in {DINO} to {Minkowski} Geometry},
  author    = {Thomas Fel and Binxu Wang and Michael A. Lepori and Matthew Kowal and Andrew Lee and Randall Balestriero and Sonia Joseph and Ekdeep S. Lubana and Talia Konkle and Demba Ba and Martin Wattenberg},
  booktitle = {International Conference on Learning Representations (ICLR)},
  year      = {2026},
  url       = {https://arxiv.org/abs/2510.08638}
}

@inproceedings{hindupur2025projecting_assumptions,
  title     = {Projecting Assumptions: The Duality Between Sparse Autoencoders and Concept Geometry},
  author    = {Sai Sumedh R. Hindupur and Ekdeep Singh Lubana and Thomas Fel and Demba Ba},
  booktitle = {Advances in Neural Information Processing Systems (NeurIPS)},
  volume    = {38},
  year      = {2025},
  url       = {https://arxiv.org/abs/2503.01822},
pages     = {11649--11699},
  doi       = {10.52202/085713-0391}
}

@inproceedings{oneill2024compute_optimal,
  title     = {Compute Optimal Inference and Provable Amortisation Gap in Sparse Autoencoders},
  author    = {Charles O'Neill and Alim Gumran and David Klindt},
  booktitle = {Proceedings of the 42nd International Conference on Machine Learning (ICML)},
  series    = {Proceedings of Machine Learning Research},
  volume    = {267},
  pages     = {46877--46896},
  publisher = {PMLR},
  year      = {2025},
  url       = {https://arxiv.org/abs/2411.13117}
}

@inproceedings{cevher2008sparse_mrf,
  title     = {Sparse Signal Recovery Using {Markov} Random Fields},
  author    = {Volkan Cevher and Marco F. Duarte and Chinmay Hegde and Richard G. Baraniuk},
  booktitle = {Advances in Neural Information Processing Systems},
  volume    = {21},
  year      = {2008},
  url       = {https://papers.neurips.cc/paper/3487-sparse-signal-recovery-using-markov-random-fields}
}

@inproceedings{belkin2001laplacian_eigenmaps,
  title     = {{Laplacian} Eigenmaps and Spectral Techniques for Embedding and Clustering},
  author    = {Mikhail Belkin and Partha Niyogi},
  booktitle = {Advances in Neural Information Processing Systems},
  volume    = {14},
  pages     = {585--591},
  year      = {2001},
  url       = {https://papers.nips.cc/paper/1961-laplacian-eigenmaps-and-spectral-techniques-for-embedding-and-clustering}
}

@book{prince2012computer_vision,
  title     = {Computer Vision: Models, Learning, and Inference},
  author    = {Simon J. D. Prince},
  publisher = {Cambridge University Press},
  address   = {Cambridge},
  year      = {2012},
  isbn      = {9781107011793},
  doi       = {10.1017/CBO9780511996504}
}

@inproceedings{margossian2023amortized,
  title     = {Amortized Variational Inference: When and Why?},
  author    = {Margossian, Charles C. and Blei, David M.},
  booktitle = {Proceedings of the 40th Conference on Uncertainty in Artificial Intelligence (UAI)},
  series    = {Proceedings of Machine Learning Research},
  volume    = {244},
  pages     = {2434--2449},
  publisher = {PMLR},
  year      = {2024}
}

@article{mencattini2026rate,
  title={The Rate-Distortion-Polysemanticity Tradeoff in {SAEs}},
  author={Mencattini, Tommaso and Montagna, Francesco and Locatello, Francesco},
  journal={arXiv preprint arXiv:2605.14694},
  year={2026}
}

@inproceedings{chanin2024absorption,
  title     = {A is for Absorption: Studying Feature Splitting and Absorption in Sparse Autoencoders},
  author    = {Chanin, David and Wilken-Smith, James and Dulka, Tom{\'a}{\v{s}} and Bhatnagar, Hardik and Golechha, Satvik and Bloom, Joseph},
  booktitle = {Advances in Neural Information Processing Systems (NeurIPS)},
  volume    = {38},
  year      = {2025},
  url       = {https://arxiv.org/abs/2409.14507},
pages     = {82318--82355},
  doi       = {10.52202/085713-2755}
}

@article{chanin2025hedging,
  title={Feature Hedging: Correlated Features Break Narrow Sparse Autoencoders},
  author={Chanin, David and Dulka, Tom{\'a}{\v{s}} and Garriga-Alonso, Adri{\`a}},
  journal={arXiv preprint arXiv:2505.11756},
  year={2025}
}

@article{ayonrinde2024mdl,
  title={Interpretability as Compression: Reconsidering {SAE} Explanations of Neural Activations with {MDL-SAEs}},
  author={Ayonrinde, Kola and Pearce, Michael T. and Sharkey, Lee},
  journal={arXiv preprint arXiv:2410.11179},
  year={2024}
}

@article{chanin2026synthsaebench,
  title={{SynthSAEBench}: Evaluating Sparse Autoencoders on Scalable Realistic Synthetic Data},
  author={Chanin, David and Garriga-Alonso, Adri{\`a}},
  journal={arXiv preprint arXiv:2602.14687},
  year={2026}
}

@misc{elhage2022superposition,
  title={Toy Models of Superposition},
  author={Elhage, Nelson and Hume, Tristan and Olsson, Catherine and Schiefer, Nicholas and Henighan, Tom and Kravec, Shauna and Hatfield-Dodds, Zac and Lasenby, Robert and Drain, Dawn and Chen, Carol and Grosse, Roger and McCandlish, Sam and Kaplan, Jared and Amodei, Dario and Wattenberg, Martin and Olah, Christopher},
  year={2022},
  howpublished={Transformer Circuits Thread},
  note={\url{https://transformer-circuits.pub/2022/toy_model/index.html}, arXiv:2209.10652}
}

@inproceedings{park2024lrh,
  title={The Linear Representation Hypothesis and the Geometry of Large Language Models},
  author={Park, Kiho and Choe, Yo Joong and Veitch, Victor},
  booktitle={Proceedings of the 41st International Conference on Machine Learning},
  year={2024},
series    = {Proceedings of Machine Learning Research},
  volume    = {235},
  pages     = {39643--39666},
  publisher = {PMLR}
}

@inproceedings{LRHorigin,
  title     = {On the Origins of Linear Representations in Large Language Models},
  author    = {Yibo Jiang and Goutham Rajendran and Pradeep Ravikumar and Bryon Aragam and Victor Veitch},
  booktitle = {Proceedings of the 41st International Conference on Machine Learning (ICML)},
  series    = {Proceedings of Machine Learning Research},
  volume    = {235},
  publisher = {PMLR},
  year      = {2024},
  url       = {https://arxiv.org/abs/2403.03867},
pages     = {21879--21911}
}

@inproceedings{saemp,
  title     = {From Flat to Hierarchical: Extracting Sparse Representations with Matching Pursuit},
  author    = {Val\'erie Costa and Thomas Fel and Ekdeep Singh Lubana and Bahareh Tolooshams and Demba Ba},
  booktitle = {Advances in Neural Information Processing Systems (NeurIPS)},
  volume    = {38},
  year      = {2025},
  url       = {https://arxiv.org/abs/2506.03093},
pages     = {34381--34424},
  doi       = {10.52202/085713-1154}
}

@article{field1987relations,
  title={Relations between the statistics of natural images and the response properties of cortical cells},
  author={Field, David J.},
  journal={Journal of the Optical Society of America A},
  volume={4},
  number={12},
  pages={2379--2394},
  year={1987},
  publisher={Optical Society of America}
}

@misc{stevens2025interpretable,
  title         = {Interpretable and Testable Vision Features via Sparse Autoencoders},
  author        = {Stevens, Samuel and Chao, Wei-Lun and Berger-Wolf, Tanya and Su, Yu},
  year          = {2025},
  eprint        = {2502.06755},
  archivePrefix = {arXiv},
  primaryClass  = {cs.CV},
  url           = {https://arxiv.org/abs/2502.06755}
}

@misc{anthropic2025haiku45,
  author = {{Anthropic}},
  title  = {Introducing {Claude Haiku 4.5}},
  year   = {2025},
  month  = oct,
  url    = {https://www.anthropic.com/news/claude-haiku-4-5}
}

@article{wagemans2012gestalt,
  author  = {Wagemans, Johan and Elder, James H. and Kubovy, Michael
             and Palmer, Stephen E. and Peterson, Mary A.
             and Singh, Manish and von der Heydt, R{\"u}diger},
  title   = {A century of {Gestalt} psychology in visual perception:
             {I. Perceptual} grouping and figure--ground organization},
  journal = {Psychological Bulletin},
  year    = {2012},
  volume  = {138},
  number  = {6},
  pages   = {1172--1217},
  doi     = {10.1037/a0029333}
}

@article{biederman1972scenes,
  author  = {Biederman, Irving},
  title   = {Perceiving real-world scenes},
  journal = {Science},
  year    = {1972},
  volume  = {177},
  number  = {4043},
  pages   = {77--80},
  doi     = {10.1126/science.177.4043.77}
}

@article{palmer1975context,
  author  = {Palmer, Stephen E.},
  title   = {The effects of contextual scenes on the identification of objects},
  journal = {Memory \& Cognition},
  year    = {1975},
  volume  = {3},
  number  = {5},
  pages   = {519--526},
  doi     = {10.3758/BF03197524}
}

@misc{qwenscope,
      title={{Qwen-Scope}: Turning Sparse Features into Development Tools for Large Language Models}, 
      author={Boyi Deng and Xu Wang and Yaoning Wang and Yu Wan and Yubo Ma and Baosong Yang and Haoran Wei and Jialong Tang and Huan Lin and Ruize Gao and Tianhao Li and Qian Cao and Xuancheng Ren and Xiaodong Deng and An Yang and Fei Huang and Dayiheng Liu and Jingren Zhou},
      year={2026},
      eprint={2605.11887},
      archivePrefix={arXiv},
      primaryClass={cs.CL},
      url={https://arxiv.org/abs/2605.11887}, 
}

@inproceedings{gemmascope,
  title     = {{Gemma Scope}: Open Sparse Autoencoders Everywhere All At Once on {Gemma} 2},
  author    = {Tom Lieberum and Senthooran Rajamanoharan and Arthur Conmy and Lewis Smith and Nicolas Sonnerat and Vikrant Varma and J{\'a}nos Kram{\'a}r and Anca Dragan and Rohin Shah and Neel Nanda},
  booktitle = {Proceedings of the 7th BlackboxNLP Workshop: Analyzing and Interpreting Neural Networks for NLP},
  publisher = {Association for Computational Linguistics},
  pages     = {278--300},
  year      = {2024},
  doi       = {10.18653/v1/2024.blackboxnlp-1.19}
}

@inproceedings{bhalla2026sparse,
  title={Do Sparse Autoencoders Capture Concept Manifolds?},
  author={Bhalla, Usha and Fel, Thomas and Rager, Can and Feucht, Sheridan and Haklay, Tal and Wurgaft, Daniel and Boppana, Siddharth and Kowal, Matthew and Merullo, Jack and Geiger, Atticus and Lubana, Ekdeep Singh},
  booktitle={Conference on Language Modeling (COLM)},
  year={2026},
  url={https://arxiv.org/abs/2604.28119}
}

@article{kalman1960new,
  title={A new approach to linear filtering and prediction problems},
  author={Kalman, Rudolf Emil},
  journal={Journal of Basic Engineering},
  volume={82},
  number={1},
  pages={35--45},
  year={1960},
  publisher={American Society of Mechanical Engineers}
}

@article{bach2019submodular,
  author  = {Bach, Francis},
  title   = {Submodular Functions: From Discrete to Continuous Domains},
  journal = {Mathematical Programming},
  volume  = {175},
  pages   = {419--459},
  year    = {2019},
  doi     = {10.1007/s10107-018-1248-6},
  url     = {https://arxiv.org/abs/1511.00394}
}

@inproceedings{weller2015revisiting,
  author    = {Weller, Adrian},
  title     = {Revisiting the Limits of {MAP} Inference by {MWSS} on Perfect Graphs},
  booktitle = {Proceedings of the Eighteenth International Conference on Artificial Intelligence and Statistics},
  series    = {Proceedings of Machine Learning Research},
  volume    = {38},
  pages     = {1061--1069},
  year      = {2015},
  publisher = {PMLR},
  url       = {https://proceedings.mlr.press/v38/weller15.html}
}

@inproceedings{gotovos2015sampling,
  author    = {Gotovos, Alkis and Hassani, S. Hamed and Krause, Andreas},
  title     = {Sampling from Probabilistic Submodular Models},
  booktitle = {Advances in Neural Information Processing Systems},
  volume    = {28},
  year      = {2015},
  url       = {https://papers.nips.cc/paper/5744-sampling-from-probabilistic-submodular-models}
}

@article{zhou2019semantic,
  title={Semantic understanding of scenes through the {ADE20K} dataset},
  author={Zhou, Bolei and Zhao, Hang and Puig, Xavier and Xiao, Tete and Fidler, Sanja and Barriuso, Adela and Torralba, Antonio},
  journal={International Journal of Computer Vision},
  volume={127},
  number={3},
  pages={302--321},
  year={2019},
  doi={10.1007/s11263-018-1140-0},
  publisher={Springer}
}

@inproceedings{cocostuff,
  author    = {Caesar, Holger and Uijlings, Jasper and Ferrari, Vittorio},
  title     = {{COCO-Stuff}: Thing and Stuff Classes in Context},
  booktitle = {Proceedings of the IEEE/CVF Conference on Computer Vision and Pattern Recognition (CVPR)},
  pages     = {1209--1218},
  year      = {2018},
  doi       = {10.1109/CVPR.2018.00132}
}

@article{kailath1968innovations,
  title   = {An Innovations Approach to Least-Squares Estimation---{Part I}: Linear Filtering in Additive White Noise},
  author  = {Kailath, Thomas},
  journal = {IEEE Transactions on Automatic Control},
  volume  = {13},
  number  = {6},
  pages   = {646--655},
  year    = {1968},
  doi     = {10.1109/TAC.1968.1099025}
}

@article{shuman2013emerging,
  title   = {The Emerging Field of Signal Processing on Graphs: Extending High-Dimensional Data Analysis to Networks and Other Irregular Domains},
  author  = {Shuman, David I. and Narang, Sunil K. and Frossard, Pascal and Ortega, Antonio and Vandergheynst, Pierre},
  journal = {IEEE Signal Processing Magazine},
  volume  = {30},
  number  = {3},
  pages   = {83--98},
  year    = {2013},
  doi     = {10.1109/MSP.2012.2235192}
}
\bibliographystyle{iclr2027_conference}

\vfill\pagebreak
\appendix

\pdfbookmark[0]{Appendix}{appendix}
\section*{Appendix Overview}
\etocdepthtag.toc{appendix}
\etocsettagdepth{main}{none}
\etocsettagdepth{appendix}{2}
\etocsettocstyle{\vspace{-1em}}{}
\tableofcontents

\vfill\pagebreak
\vfill\pagebreak
\section{Preliminaries}

This appendix develops the mathematical language used to express spatial
dependence in \MFLRH. In \Cref{app:notation}, we represent image patches as
vertices of a graph and introduce the Laplacian and Dirichlet energies that
distinguish spatial agreement from spatial contrast. In
\Cref{app:undirected-graphical-models}, we turn these graph interactions into
probabilistic priors by reviewing undirected graphical models, pairwise Markov
random fields, and \MAP estimation. Together, these preliminaries support the
theoretical development of \SpatialSAE in \Cref{sec:theory-of-sae}.

\subsection{Patch Graphs, Laplacian, and Dirichlet Energy}
\label{app:notation}

This subsection introduces the graph notation used to express spatial locality. 
The purpose is simple: patch activations in a vision model are not indexed by an 
arbitrary set, but by positions in an image. To make this spatial structure 
available to the prior, we represent the patch layout as a weighted graph. The
graph specifies which patches interact. These interactions can favor agreement
or contrast between neighboring codes.

\intuition{Why a graph}
{A graph encodes which patches are neighbors. Once patches are vertices and
adjacency is encoded by edge weights, we can measure how much a latent concept
field changes across edges. Penalizing these differences encourages smoothness;
favoring them, while controlling code magnitude, encourages spatial contrast.}

We first define the object that carries spatial information. Generally, the edge weights may depend on patch adjacency, 
attention, feature affinity, or a mixture of spatial and representation-space 
similarity. The formal arguments below only require an undirected graph with 
symmetric nonnegative weights.

\begin{definition}[Patch graph and graph Laplacian]
\label{def:app-patch-graph}
Let the image patches be equipped with an undirected weighted graph
\[
    \mathcal{G}=(\mathcal{V},\mathcal{E},\mathbf{W}),
\]
where $\mathcal{V}=\{1,\ldots,p\}$ is the set of patch indices and
$\mathcal{E}$ consists of unordered pairs of distinct vertices. The weight
matrix $\mathbf{W}=(w_{ij})\in\R^{p\times p}$ satisfies
$w_{ij}=w_{ji}\ge 0$, with $w_{ii}=0$ and $w_{ij}=0$ whenever
$\{i,j\}\notin\mathcal{E}$. Each undirected edge is counted once in sums
and products over $\mathcal{E}$. The degree matrix is
\[
    \boldsymbol{\Delta}=\diag(\Delta_{11},\ldots,\Delta_{pp}),
    \qquad
    \Delta_{ii}=\sum_{j=1}^p w_{ij}.
\]
The graph Laplacian associated with $\mathcal{G}$ is
\[
    \mathbf{L}=\boldsymbol{\Delta}-\mathbf{W}.
\]
\end{definition}

The Laplacian measures differences between neighboring patch codes through the
graph Dirichlet energy. Large edge weights give the corresponding differences
more influence. This measure of spatial variation applies to any latent
field; its role in the objective depends on whether variation is penalized or
favored.

\begin{definition}[Graph Dirichlet energy]
\label{def:app-dirichlet-energy}
For a matrix $\mathbf{X}\in\R^{p\times c}$, whose row $\mathbf{x}_i$ denotes the
latent code at patch $i$, the graph Dirichlet energy of $\mathbf{X}$ on
$\mathcal{G}$ is
\[
    \Dir(\mathbf{X})
    :=
    \Tr(\mathbf{X}^\top \mathbf{L}\mathbf{X}).
\]
\end{definition}

This definition is compact, but its meaning is not obvious from the trace form 
alone. The quantity $\Dir(\mathbf{X})$ measures how much the concept field
varies over the graph: it is small when neighboring patches have similar
codes, and large when the code changes abruptly across high-weight edges. In this
sense, it is the graph analogue of the classical continuous Dirichlet energy
\[
    \int_{\Omega} \norm{\nabla s(x)}_2^2\,dx,
\]
which measures the spatial variation of a function $s$ over a continuous domain 
$\Omega$. The next lemma expresses both the Laplacian quadratic and its shifted
counterpart in terms of squared differences between neighboring patch codes.
These identities describe attractive and repulsive interactions on the graph.

\begin{lemma}[Laplacian trace identities]
\label{lem:app-laplacian-trace}
Let $\mathcal{G}=(\mathcal{V},\mathcal{E},\mathbf{W})$ be a patch graph as in
\Cref{def:app-patch-graph}, and let
$\mathbf{L}=\boldsymbol{\Delta}-\mathbf{W}$ be its graph Laplacian. Then, for 
every $\mathbf{X}\in\R^{p\times c}$ and $\beta\in\R$,
\[
\begin{aligned}
    \Tr(\mathbf{X}^\top \mathbf{L}\mathbf{X})
    &=\sum_{\{i,j\}\in\mathcal{E}}
    w_{ij}\norm{\mathbf{x}_i-\mathbf{x}_j}_2^2, \\
    \Tr\bigl(\mathbf{X}^\top(\beta\mathbf{I}_p-\mathbf{L})\mathbf{X}\bigr)
    &=\beta\norm{\mathbf{X}}_F^2
    -\sum_{\{i,j\}\in\mathcal{E}}
    w_{ij}\norm{\mathbf{x}_i-\mathbf{x}_j}_2^2.
\end{aligned}
\]
\end{lemma}

\begin{proof}
Using the definition of the graph Laplacian,
\[
    \Tr(\mathbf{X}^\top \mathbf{L}\mathbf{X})
    =
    \Tr(\mathbf{X}^\top \boldsymbol{\Delta}\mathbf{X})
    -
    \Tr(\mathbf{X}^\top \mathbf{W}\mathbf{X}).
\]
The first term expands as
\[
    \Tr(\mathbf{X}^\top \boldsymbol{\Delta}\mathbf{X})
    =
    \sum_{i=1}^p
    \Delta_{ii}\norm{\mathbf{x}_i}_2^2,
\]
and the second term expands as
\[
    \Tr(\mathbf{X}^\top \mathbf{W}\mathbf{X})
    =
    \sum_{i,j=1}^p
    w_{ij}\ip{\mathbf{x}_i}{\mathbf{x}_j}.
\]
Since $\Delta_{ii}=\sum_j w_{ij}$ and $\mathbf{W}$ is symmetric, we can rewrite 
the difference as
\[
\begin{aligned}
    \Tr(\mathbf{X}^\top \mathbf{L}\mathbf{X})
    &=
    \sum_i \Delta_{ii}\norm{\mathbf{x}_i}_2^2
    -
    \sum_{i,j}w_{ij}\ip{\mathbf{x}_i}{\mathbf{x}_j} \\
    &=
    \frac12
    \sum_{i,j}
    w_{ij}
    \left(
        \norm{\mathbf{x}_i}_2^2
        +
        \norm{\mathbf{x}_j}_2^2
        -
        2\ip{\mathbf{x}_i}{\mathbf{x}_j}
    \right) \\
    &=
    \frac12
    \sum_{i,j}
    w_{ij}
    \norm{\mathbf{x}_i-\mathbf{x}_j}_2^2 \\
    &=\sum_{\{i,j\}\in\mathcal{E}}
    w_{ij}\norm{\mathbf{x}_i-\mathbf{x}_j}_2^2.
\end{aligned}
\]
Each undirected edge contributes twice to the ordered sum, giving the edge-sum
form of the first identity. For the second identity, use
\[
    \Tr\bigl(\mathbf{X}^\top(\beta\mathbf{I}_p-\mathbf{L})\mathbf{X}\bigr)
    =\beta\norm{\mathbf{X}}_F^2
      -\Tr(\mathbf{X}^\top\mathbf{L}\mathbf{X})
\]
and substitute the first identity.
\end{proof}

\paragraph{Proof intuition.}
The degree and adjacency terms regroup into weighted squared differences across
edges. The shifted quadratic reverses these edge contributions and adds a term
proportional to the squared field magnitude.

Minimizing $\Dir(\mathbf{X})$ encourages
neighboring patches to have similar concept activations. This is the graph
smoothness penalty used in spectral graph methods and graph signal processing
\citep{belkin2001laplacian_eigenmaps,shuman2013emerging}.
Conversely, at a fixed magnitude, larger $\Dir(\mathbf{X})$ favors differences between
neighboring codes. A negative Dirichlet term alone, however, would be unbounded
below on any graph with a positive-weight edge, since increasing the magnitude
of a varying field would decrease the energy without limit.

\begin{corollary}[Stabilized repulsion]
\label{cor:app-stabilized-repulsion}
Under the assumptions of \Cref{lem:app-laplacian-trace}, if
$\beta\ge\lambda_{\max}(\mathbf{L})$, then for every
$\mathbf{X}\in\R^{p\times c}$,
\[
    \Tr\bigl(\mathbf{X}^\top(\beta\mathbf{I}_p-\mathbf{L})\mathbf{X}\bigr)
    =\beta\norm{\mathbf{X}}_F^2-\Dir(\mathbf{X})
    \ge0.
\]
\end{corollary}

\begin{proof}
The equality follows from \Cref{lem:app-laplacian-trace}. For nonnegativity, set
$\mathbf{A}=\beta\mathbf{I}_p-\mathbf{L}$. Its eigenvalues are
$\beta-\lambda_k\ge0$ under the assumed bound on $\beta$, so
$\mathbf{A}\succeq0$. Write $\mathbf{X}=[\mathbf{v}_1,\ldots,\mathbf{v}_c]$
column-wise, where $\mathbf{v}_a\in\R^p$. Then
\[
    \Tr(\mathbf{X}^\top\mathbf{A}\mathbf{X})
    =\sum_{a=1}^c(\mathbf{X}^\top\mathbf{A}\mathbf{X})_{aa}
    =\sum_{a=1}^c\mathbf{v}_a^\top\mathbf{A}\mathbf{v}_a
    \ge0,
\]
because each summand is nonnegative when $\mathbf{A}\succeq0$.
\end{proof}

For $\beta\ge\lambda_{\max}(\mathbf{L})$, the negative edge term favors
contrast, while the magnitude term keeps the shifted quadratic nonnegative.
The two operators give complementary spectral preferences: $\mathbf{L}$
penalizes a graph mode with eigenvalue $\lambda_k$ in proportion to $\lambda_k$,
whereas $\beta\mathbf{I}_p-\mathbf{L}$ penalizes it in proportion to
$\beta-\lambda_k$. Thus, at a fixed field magnitude, the Laplacian quadratic
favors low graph frequencies and the shifted quadratic favors high graph
frequencies.

The graph therefore supports two complementary spatial preferences. The next
subsection explains how attractive and repulsive interactions are expressed
through potentials in an undirected graphical model.

\subsection{Undirected Graphical Models and Markov Random Fields}
\label{app:undirected-graphical-models}

This subsection introduces the probabilistic language needed to turn spatial
locality into a prior.  The graph preliminaries above tell us which patches are
neighbors.  We now explain how a graph can also describe dependencies between
random variables attached to those patches.  This is the role of undirected
graphical models. We introduce local dependence,
potential functions, pairwise Markov random fields, and the \MAP principle that
turns a probabilistic model into an optimization objective.

\intuition{Why graphical models}
{A patch graph says which patches are adjacent.  An undirected graphical model
uses this adjacency structure to describe which latent variables interact
directly. For the quadratic difference interactions used here, attraction favors
neighboring agreement, while repulsion favors contrast. Both preferences
can be expressed on the same patch adjacency structure.}

Following the classical computer-vision view of graphical models
\citep{prince2012computer_vision}, the point of the graph is not merely to draw
connections between variables.  The graph is a modeling assumption: it says that
the joint distribution should be built from local interactions rather than from
arbitrary global dependencies.  This is particularly natural for images, where
pixels or patches live on a spatial grid.

\begin{definition}[Undirected graphical model]
\label{def:app-undirected-graphical-model}
An undirected graphical model consists of an undirected graph
$\mathcal{G}=(\mathcal{V},\mathcal{E})$ and a joint distribution over its
vertex-indexed variables $\{X_i\}_{i\in\mathcal{V}}$ satisfying the Markov
property in \Cref{def:app-markov-random-field}. The graph specifies which variables
may interact directly: an edge $\{i,j\}\in\mathcal{E}$ means that $X_i$ and $X_j$ may
have a direct probabilistic interaction.
\end{definition}

In this appendix, the undirected graphical model of interest is a Markov random
field.  The Markov property makes the locality assumption explicit: once the
neighbors of a variable are known, the rest of the graph does not directly
change the conditional distribution of that variable.

\begin{definition}[Markov random field]
\label{def:app-markov-random-field}
Let $\mathcal{G}=(\mathcal{V},\mathcal{E})$ be an undirected graph, and let
$\{X_i\}_{i\in\mathcal{V}}$ be random variables indexed by its vertices.  The
collection $\{X_i\}_{i\in\mathcal{V}}$ is a Markov random field on
$\mathcal{G}$ if, for every $i\in\mathcal{V}$, the following equality holds
almost surely:
\[
    p(x_i \mid x_{\mathcal{V}\setminus\{i\}})
    =
    p(x_i \mid x_{\mathcal{N}(i)}),
\]
where
\[
    \mathcal{N}(i)
    =
    \{j\in\mathcal{V}:\{i,j\}\in\mathcal{E}\}
\]
is the set of neighbors of $i$.
\end{definition}

The definition says that the prior is local.  A patch code may interact directly
with neighboring patch codes, but it does not need a separate direct interaction
with every other patch in the image.  This is the same modeling principle used in
classical grid-based vision models. Locality specifies which variables interact;
the interaction potentials determine whether they favor spatial agreement or
contrast.

To specify which local configurations are likely or unlikely, graphical models
use potential functions.  There are two equivalent conventions.  One can use
positive compatibility potentials, whose larger values make a configuration more
probable.  Or one can use energy potentials, whose larger values make a
configuration less probable through a Gibbs density.  In this appendix we use the
energy convention because it connects directly to \MAP minimization.

\begin{definition}[Pairwise Markov random fields and potentials]
\label{def:app-potential-functions}
\label{def:app-pairwise-mrf-energy}
Let $\mathcal{G}=(\mathcal{V},\mathcal{E})$ be an undirected graph with variables
$\{X_i\}_{i\in\mathcal{V}}$.

A \emph{compatibility potential} is a strictly positive function.  Unary
compatibility potentials $\psi_i(x_i)>0$ score individual states, and pairwise
compatibility potentials $\psi_{ij}(x_i,x_j)>0$ score joint states of neighboring
variables. In sums and products over $\mathcal{E}$, each undirected edge is
counted once. In compatibility form, a pairwise Markov random field has density
\[
    p(\mathbf{x})
    =
    \frac{1}{Z}
    \prod_{i\in\mathcal{V}}\psi_i(x_i)
    \prod_{\{i,j\}\in\mathcal{E}}\psi_{ij}(x_i,x_j).
\]
Larger compatibility values increase the probability of the corresponding
configuration.

Equivalently, define unary and pairwise \emph{energy potentials}
\[
    U_i(x_i)=-\log\psi_i(x_i),
    \qquad
    P_{ij}(x_i,x_j)=-\log\psi_{ij}(x_i,x_j).
\]
Then
\[
    p(\mathbf{x})
    =
    \frac{1}{Z}
    \exp\!\left(-E(\mathbf{x})\right),
    \qquad
    E(\mathbf{x})
    =
    \sum_{i\in\mathcal{V}}U_i(x_i)
    +
    \sum_{\{i,j\}\in\mathcal{E}}P_{ij}(x_i,x_j).
\]
In energy form, lower potential values are preferred by the prior, while higher
potential values decrease probability. Individual energy potentials need not be
nonnegative: a repulsive potential can decrease as neighboring codes differ.
The complete energy must define a finite, positive normalizing constant
$0<Z<\infty$ for the density to be well-defined.
\end{definition}

Since \MAP inference minimizes negative log probability, the energy convention
directly yields optimization objectives for both attractive and repulsive
interactions.

Unary potentials encode preferences about individual variables, such as
sparsity or magnitude. Pairwise potentials determine how neighboring variables
interact. To define attraction and repulsion for continuous variables, we first
recall submodularity on products of ordered domains
\citep[Sections~2.1--2.2]{bach2019submodular}.

\begin{definition}[Submodularity and supermodularity]
\label{def:app-continuous-submodularity}
Let $\mathcal{X}=\prod_{k=1}^{n} I_k$, where each $I_k\subseteq\R$ is an
interval. For $\mathbf{x},\mathbf{y}\in\mathcal{X}$, let
$\mathbf{x}\wedge\mathbf{y}$ and $\mathbf{x}\vee\mathbf{y}$ denote their
coordinatewise minimum and maximum. A function $f:\mathcal{X}\to\R$ is
\emph{submodular} if
\[
    f(\mathbf{x})+f(\mathbf{y})
    \geq
    f(\mathbf{x}\wedge\mathbf{y})+f(\mathbf{x}\vee\mathbf{y})
    \qquad
    \text{for all }\mathbf{x},\mathbf{y}\in\mathcal{X}.
\]
It is \emph{supermodular} if the inequality is reversed, and \emph{modular} if
equality holds for all $\mathbf{x},\mathbf{y}$. A positive function
$\psi:\mathcal{X}\to(0,\infty)$ is \emph{log-submodular} or
\emph{log-supermodular} when $\log\psi$ has the corresponding property.
\end{definition}

The following definition adapts the binary attraction and repulsion conventions
of \citet[Section~3.1.1]{weller2015revisiting} and
\citet{gotovos2015sampling} to continuous variables, using the same relation
between energy and log-compatibility.

\begin{definition}[Attractive and repulsive pairwise \MRFs]
\label{def:app-attractive-repulsive-mrf}
Consider a continuous pairwise \MRF on
$\mathcal{G}=(\mathcal{V},\mathcal{E})$ with density
\[
    q(\mathbf{x}\mid\mathcal{G})
    =\frac{1}{Z}
    \exp\!\left(
        -\sum_{i\in\mathcal{V}}U_i(x_i)
        -\sum_{\{i,j\}\in\mathcal{E}}P_{ij}(x_i,x_j)
    \right),
\]
on a product of intervals, where $0<Z<\infty$ and each undirected edge is
counted once. The functions $U_i$ and $P_{ij}$ are unary and pairwise energy
potentials, respectively.

The field is \emph{attractive} if every $P_{ij}$ is submodular, and
\emph{repulsive} if every $P_{ij}$ is supermodular. Equivalently, the edge
compatibility factors $\exp(-P_{ij})$ are respectively log-supermodular and
log-submodular. We include zero interactions as a boundary case.
\end{definition}

Attraction lets neighboring patches reinforce a feature: holding the other
variables fixed, a larger feature value at one patch makes larger values at
its neighbor relatively more likely. Repulsion reverses this preference.
For nonnegative codes, it discourages strong neighboring activations of the
same feature.

For twice continuously differentiable pairwise energies, attraction and
repulsion are equivalently characterized by
\[
    \frac{\partial^2 P_{ij}}{\partial x_i\,\partial x_j}\leq0
    \quad\text{and}\quad
    \frac{\partial^2 P_{ij}}{\partial x_i\,\partial x_j}\geq0,
\]
respectively \citep[Section~2.2]{bach2019submodular}. Scalar unary terms are
separable and therefore modular: adding them to an energy changes neither classification.
The lattice definition itself requires no differentiability, so this observation
also applies to nonsmooth unary penalties such as $\lambda|x_i|$.

For the negative Dirichlet energy considered here,
\Cref{cor:app-stabilized-repulsion} supplies a unary magnitude term that makes
the combined quadratic nonnegative. Adding a positive
$\ell_1$ penalty controls directions with zero quadratic cost and yields a
normalizable density.

The last probabilistic object we need is \MAP inference. This is the bridge from
a graphical model to an optimization objective. Once the prior and
likelihood are written in energy form, \MAP inference becomes minimization of a
data-fit term plus a prior energy.

\begin{definition}[\MAP estimator]
\label{def:app-map-estimator}
Given observed data $\mathbf{y}$ and latent variables $\mathbf{x}$, the
maximum a posteriori estimator is any configuration
\[
    \mathbf{x}_{\mathrm{MAP}}
    \in
    \argmax_{\mathbf{x}}
    p(\mathbf{x}\mid \mathbf{y}).
\]
Equivalently, if
\[
    p(\mathbf{x}\mid \mathbf{y})
    \propto
    \exp\!\left(-E(\mathbf{x};\mathbf{y})\right),
\]
then
\[
    \mathbf{x}_{\mathrm{MAP}}
    \in
    \argmin_{\mathbf{x}}
    E(\mathbf{x};\mathbf{y}).
\]
\end{definition}

On a product domain, with respect to a product reference measure,
a prior whose energy contains only unary terms factorizes across variables.
Pairwise potentials can introduce local dependencies, which appear as
graph-dependent terms in the \MAP objective. For continuous patch codes, the
quadratic energies above can express either spatial agreement or spatial
contrast, with unary terms promoting sparsity and controlling magnitude.

An undirected graphical model therefore specifies which variables interact,
potential functions specify which configurations are preferred, and \MAP
inference turns the posterior into an energy minimization problem.

\vfill\pagebreak
\section{Theory of \SpatialSAEs}
\label{sec:theory-of-sae}

Having established the required graph and probabilistic preliminaries, we now
give the full theoretical account of \SpatialSAE. In \Cref{app:lrh-to-sae}, we
connect the \LRH to sparse coding and its amortization by an \SAE. In
\Cref{app:vanilla-map}, we derive the \MAP interpretation of the vanilla \SAE
objective and expose its factorized prior over patch codes. We then replace
that prior with the spatially coupled \MFLRH prior in
\Cref{app:mflrh-prior}. In
\Cref{app:continuous-mf-lrh-to-spatial-sae}, we derive the resulting
\SpatialSAE objective, relate it to amortized inference, and show that it
depends on spatial layout rather than an arbitrary ordering of the patches.

\subsection{From \LRH to Sparse Coding and Sparse Autoencoders}
\label{app:lrh-to-sae}

This subsection fixes the algebraic starting point of the appendix.  We first
state the Linear Representation Hypothesis in the notation used here.  We then
explain why the hypothesis turns concept extraction into sparse dictionary
learning, and finally define the vanilla sparse autoencoder as the amortized
version of that sparse-coding problem.  At this stage no spatial assumption is
being made; the goal is only to make clear what classical \LRH and vanilla
\SAEs do and do not specify.

\intuition{From representations to inference}
{The \LRH says that an activation is a sparse linear combination of dictionary
atoms.  Once this is assumed, interpretability becomes an inverse problem: infer
the sparse coefficients and the dictionary.  A sparse autoencoder replaces the
per-example sparse-coding solve by a learned encoder $f_{\saeparams}$.}

We begin by restating the \LRH as stated in \citet[Def.~2.1]{saemp}, in the notation of \citet[Def.~1]{fel2025rabbit_hull}:

\begin{definition}[Linear Representation Hypothesis]
\label{def:app-lrh}
A representation $\mathbf{h}\in\R^d$ satisfies the Linear Representation
Hypothesis, abbreviated \LRH, if there exists a sparsity constant $k$, an
overcomplete dictionary
$\mathbf{D}=(\mathbf{d}_1,\ldots,\mathbf{d}_c)\in\R^{c\times d}$, and a
coefficient vector $\mathbf{z}\in\R^c$ such that
$\mathbf{h}\approx \mathbf{z}\mathbf{D}$, under the following conditions:
\[
\left\{
\begin{array}{lll}
\text{(i) Overcompleteness:}
& c\gg d, \\
\text{(ii) Quasi-orthogonality:}
& \mu(\mathbf{D}) := \max_{a\ne b}\abs{\mathbf{d}_a\mathbf{d}_b^\top}\le \varepsilon,
& \norm{\mathbf{d}_a}_2=1\ \forall a, \\
\text{(iii) Sparsity:}
& \abs{\supp(\mathbf{z})}\le k\ll c.
\end{array}
\right.
\]
\end{definition}

The definition is local: it describes one activation vector $\mathbf{h}\in\R^{d}$.  In practice, the
object of interest is a collection of activations $\mathbf{H}\in\R^{n\times d}$.  If
$\mathbf{H}$ collects them row by row, applying
\Cref{def:app-lrh} to each row gives a sparse code matrix
$\mathbf{Z}\in\R^{n\times c}$ and the approximation
$\mathbf{H}\approx \mathbf{Z}\mathbf{D}$.  Thus, \textit{under \LRH, concept extraction
becomes the problem of recovering both the dictionary and the sparse code matrix}.

This inverse problem is a dictionary-learning problem. As argued by \citet{fel2025rabbit_hull}, different concept
extraction methods correspond to different constraints on the code
$\mathbf{Z}$ and dictionary $\mathbf{D}$: clustering methods constrain codes to
select prototypical examples as atoms, orthogonal factorization constrains the dictionary to be
orthogonal, nonnegative factorization imposes nonnegativity, and sparse autoencoders
constrain codes to be produced by an encoder with a sparsity-inducing mechanism
\citep{hindupur2025projecting_assumptions}. \textit{These are not
neutral algorithmic choices: each constraint encodes a different prior over what
kind of decomposition should count as a good explanation.}

Under \LRH, the relevant assumptions are \textbf{overcompleteness} and \textbf{sparsity}. Therefore
the dictionary-learning problem aligned with \LRH is sparse coding: learn an
overcomplete dictionary and assign each activation a sparse coefficient vector.

\begin{definition}[Sparse coding induced by \LRH]
\label{def:app-sparse-coding}
Given activations $\mathbf{H}\in\R^{n\times d}$, a latent width $c$, and a
sparsity regularizer $\mathcal{R}$, the sparse-coding problem induced by \LRH is
\[
    (\mathbf{Z}^\star,\mathbf{D}^\star)
    \in
    \argmin_{\mathbf{Z},\mathbf{D}}
    \frac{1}{n}\sum_{i=1}^n
    \left[
    \norm{\mathbf{h}_i-\mathbf{z}_i\mathbf{D}}_2^2
    +
    \lambda\mathcal{R}(\mathbf{z}_i)
    \right],
\]
where $\mathbf{z}_i$ is the sparse code assigned to activation $\mathbf{h}_i$,
$\mathbf{D}$ is the learned dictionary, and $\lambda\ge 0$ controls the sparsity
strength.
\end{definition}

Sparse coding is faithful to the algebraic content of \LRH: it searches for an
overcomplete dictionary and sparse coefficients that reconstruct the observed
activations. However, it is computationally expensive at scale. For a fixed
dictionary $\mathbf{D}$, the code $\mathbf{z}_i$ of each activation $\mathbf{h}_i$ is itself the solution of a
sparse inference problem:
\[
    \mathbf{z}^{\star}(\mathbf{h}_i;\mathbf{D})
    \in
    \argmin_{\mathbf{z}\in\R^c}
    \left[
    \norm{\mathbf{h}_i-\mathbf{z}\mathbf{D}}_2^2
    +
    \lambda\mathcal{R}(\mathbf{z})
    \right].
\]
Thus sparse coding requires solving, or approximately solving, a new optimization
problem for every activation. This is the computational motivation for
amortized inference: instead of solving the sparse inference problem separately
for each $\mathbf{h}_i$, one learns a shared function $f_{\saeparams}$ that predicts the code directly from the activation, $\mathbf z_i = f_{\saeparams}(\mathbf h_i)$.

\begin{definition}[Amortized sparse inference]
\label{def:app-amortized-inference}
Let
\[
    \mathcal{Q}(\mathbf{z};\mathbf{h},\mathbf{D})
    :=
    \norm{\mathbf{h}-\mathbf{z}\mathbf{D}}_2^2
    +
    \lambda\mathcal{R}(\mathbf{z})
\]
be the sparse inference objective for the code $\mathbf{z}$ of an activation
$\mathbf{h}$ under dictionary $\mathbf{D}$. The exact sparse-coding solution is
\[
    \mathbf{z}^{\star}(\mathbf{h};\mathbf{D})
    \in
    \argmin_{\mathbf{z}\in\R^c}
    \mathcal{Q}(\mathbf{z};\mathbf{h},\mathbf{D}).
\]
Sparse inference is \emph{amortized} when this per-activation optimization is
replaced by a shared inference function $f_{\saeparams}$,
\[
    \mathbf{z}_i=f_{\saeparams}(\mathbf{h}_i)
    \qquad \forall i.
\]
The parameters $\saeparams$ of $f_{\saeparams}$ are learned across the dataset, so the cost
of inference is shared across activations.
\end{definition}

This definition adapts the standard idea of amortized inference to sparse
coding. In amortized variational inference, one learns a common inference
function that maps each observation to the approximate latent object associated
with that observation \citep{margossian2023amortized}. In the sparse-coding
setting, the latent object is not a posterior distribution but a sparse
coefficient vector. The benefit is speed: once $f_{\saeparams}$ is learned, codes can
be produced by a forward pass rather than by solving the sparse-coding problem
above. The cost is approximation: the encoder may not return the exact sparse
coding solution. Following \citet{oneill2024compute_optimal}, we can measure this
cost by the amortization gap
\[
    \Delta_{\mathrm{amort}}(\saeparams;\mathbf{h},\mathbf{D})
    :=
    \mathcal{Q}(f_{\saeparams}(\mathbf{h});\mathbf{h},\mathbf{D})
    -
    \min_{\mathbf{z}\in\R^c}
    \mathcal{Q}(\mathbf{z};\mathbf{h},\mathbf{D})
    \ge 0.
\]
For a dataset, the corresponding average gap is $\Delta_{\mathrm{amort}}(\saeparams;\mathbf{H},\mathbf{D}) = \frac{1}{n}\sum_{i=1}^n \Delta_{\mathrm{amort}}(\saeparams;\mathbf{h}_i,\mathbf{D})$.
This gap is zero only when the encoder family is expressive enough, and trained
well enough, to return an optimal sparse code for each activation. In standard
\SAEs, the encoder is usually a restricted linear--nonlinear map, so the
amortization gap can be nonzero even when better sparse codes exist
\citep{oneill2024compute_optimal}.

We can now formalize the central point: a sparse autoencoder is an amortized
solution to the sparse inference problem induced by \LRH.  Instead of solving
\[
    \argmin_{\mathbf{z}\in\R^c}
    \mathcal{Q}(\mathbf{z};\mathbf{h}_i,\mathbf{D})
\]
separately for every activation $\mathbf{h}_i$, an \SAE learns an encoder
$f_{\saeparams}$ that returns the code directly.  The bilevel formulation of
\citet{hindupur2025projecting_assumptions} makes this precise: the outer problem
is the sparse dictionary-learning objective, while the inner constraint
represents the architecture-dependent encoder map that determines which codes
can be returned by the \SAE.

\begin{definition}[Sparse autoencoder as amortized sparse inference]
\label{def:app-sae}
A sparse autoencoder, abbreviated \SAE, for activations
$\mathbf{H}\in\R^{n\times d}$ consists of an encoder
$f_{\saeparams}:\R^d\to\R^c$ and a decoder dictionary
$\mathbf{D}\in\R^{c\times d}$. It solves the constrained sparse-coding problem
\[
    \argmin_{\mathbf{D},\mathbf{Z},\saeparams}
    \frac{1}{n}\sum_{i=1}^n
    \left[
    \norm{\mathbf{h}_i-\mathbf{z}_i\mathbf{D}}_2^2
    +
    \lambda\mathcal{R}(\mathbf{z}_i)
    \right],
    \qquad
    \text{s.t.}\quad
    \mathbf{z}_i
    =
    f_{\saeparams}(\mathbf{h}_i)
    \in
    \argmin_{\boldsymbol{\pi}\in\mathcal{S}_{\saeparams}}
    F_{\saeparams}(\boldsymbol{\pi};\mathbf{h}_i)
    \quad \forall i,
\]
under the following conditions:
\[
\left\{
\begin{array}{lll}
\text{(i) Overcomplete latent width:}
& c\gg d, \\[0.25em]
\text{(ii) Amortized inference:}
& \mathbf{z}_i=f_{\saeparams}(\mathbf{h}_i),
& \forall i\in\{1,\ldots,n\}, \\[0.25em]
\text{(iii) Encoder-induced constraint:}
& \mathbf{z}_i\in
\argmin_{\boldsymbol{\pi}\in\mathcal{S}_{\saeparams}}
F_{\saeparams}(\boldsymbol{\pi};\mathbf{h}_i),
& \forall i\in\{1,\ldots,n\}, \\[0.25em]
\text{(iv) Sparse codes:}
& \mathcal{R}(\mathbf{z}_i)\ \text{penalizes dense codes},
& \text{for example } \mathcal{R}(\mathbf{z}_i)=\norm{\mathbf{z}_i}_1, \\[0.25em]
\text{(v) Linear decoding:}
& \widehat{\mathbf{h}}_i=\mathbf{z}_i\mathbf{D}.
\end{array}
\right.
\]
Here $F_{\saeparams}$ is a variational formulation of the encoder, and
$\mathcal{S}_{\saeparams}$ is the architecture-induced feasible set for the encoder
output. For a standard one-layer \SAE, this constraint can be instantiated by
\[
    f_{\saeparams}(\mathbf{h})=g(\mathbf{h}\mathbf{W}+\mathbf{b}),
\]
with the choice of nonlinearity $g$ determining the encoder family (e.g., ReLU or TopK \SAEs).
\end{definition}

The definition makes precise how an \SAE relates to sparse coding. The outer
objective is the same reconstruction-versus-sparsity tradeoff that appears in
sparse dictionary learning. The difference is the constraint
$\mathbf{z}_i=f_{\saeparams}(\mathbf{h}_i)$. This constraint makes inference fast,
because the code is obtained by a learned encoder rather than by solving
$\argmin_{\mathbf{z}}\mathcal{Q}(\mathbf{z};\mathbf{h}_i,\mathbf{D})$ for every
activation. But it also changes the solution class. The encoder can be a
computationally restricted sparse-inference procedure, so its output can differ
from the code selected by exact or more expressive sparse inference. This is the
amortization cost of \SAEs: speed is obtained by restricting inference to the
range of $f_{\saeparams}$ \citep{oneill2024compute_optimal}.

This subsection introduced the standard setting behind \SAEs and the \LRH. \LRH motivates sparse,
overcomplete dictionary learning; sparse coding is the dictionary-learning
solution aligned with that hypothesis; and an \SAE is an amortized approximation
to sparse coding. The next subsection asks what probabilistic prior is implicitly
imposed by the vanilla \SAE objective. \textit{We want to understand which additional assumptions it makes beyond the \LRH, and whether they are realistic}.\footnote{This is the key passage of the first part of the theory: \SAEs arise as an approximation to the natural solution of concept discovery under \LRH (namely sparse coding), as sparse coding matches the assumptions of \LRH on the data-generating process. However, \SAEs add hidden assumptions on top of those of the \LRH, which are unrealistic for vision and make \SAEs no longer the natural solution. These assumptions become apparent only when a Bayesian approach is adopted.} The next subsection builds on \citet{lubana2025priors_time}.

\subsection{The \MAP Interpretation of Vanilla \SAEs}
\label{app:vanilla-map}
This subsection gives the probabilistic version of the previous subsection's
claim.  There, sparse coding was introduced as the dictionary-learning objective
aligned with \LRH: under \Cref{def:app-lrh}, activations are assumed to be sparse
linear combinations of atoms in an overcomplete dictionary, and sparse coding is
the optimization problem that recovers such codes.  Here we show that this same
objective arises as the \MAP estimator, in the sense of
\Cref{def:app-map-estimator}, for a noisy version of \LRH.\footnote{Noise is required for a Bayesian interpretation of the objective: it turns the deterministic approximation into a likelihood.} \textit{The \MAP perspective is the precise Bayesian sense in which sparse coding, and consequently \SAEs as amortized
approximations to sparse coding, are natural tools for \LRH.} The Bayesian reading is also useful because it exposes an assumption that is
hidden in the purely algebraic view. \textit{\LRH asks for sparsity and
overcompleteness, but the vanilla \SAE objective adds another choice: its
sparsity prior is factorized over indexed activations, imposing an independence assumption between feature activations of the same input}.  Thus the \MAP view both
justifies why sparse coding matches \LRH and reveals why vanilla \SAEs may still
be mismatched when the indexed activations have structure, such as time in
language or space in vision.

\intuition{\MAP alignment and a hidden independence prior}
{Sparse coding is the \MAP formalization of the optimization problem suggested by
\LRH.  A noisy linear reconstruction model gives the data-fit term, and a
sparsity-promoting prior gives the sparse-code term.  But once the sparsity cost
is written as a sum over indexed activations, the corresponding prior becomes a
product prior.  For patch tokens, this means the vanilla \SAE prior does not know
which patches are neighbors.}

We start from the noisy version of \LRH.  The deterministic statement
$\mathbf{H}\approx \mathbf{Z}\mathbf{D}$ is replaced by a generative model in
which activations are noisy linear reconstructions from sparse latent codes:
\[
    \mathbf{h}_i
    =
    \mathbf{z}_i\mathbf{D}
    +
    \boldsymbol{\psi}_i,
    \qquad
    \boldsymbol{\psi}_i\sim\mathcal{N}(0,\sigma^2 \mathbf{I}_d),
    \qquad
    i=1,\ldots,n.
\]
Equivalently, at the matrix level,
\[
    \mathbf{H}
    =
    \mathbf{Z}\mathbf{D}
    +
    \boldsymbol{\Psi}.
\]
The next proposition shows that sparse coding is the \MAP estimator of this
noisy \LRH model when the prior over codes is sparsity-promoting.

\begin{proposition}[Sparse coding as \MAP inference for noisy \LRH]
\label{prop:app-sparse-coding-map}
Assume the noisy \LRH model
\[
    \mathbf{h}_i
    =
    \mathbf{z}_i\mathbf{D}
    +
    \boldsymbol{\psi}_i,
    \qquad
    \boldsymbol{\psi}_i\sim\mathcal{N}(0,\sigma^2 \mathbf{I}_d),
\]
with conditionally independent noise across indexed activations.  Treat
$\mathbf{D}$ as fixed, and assume a sparsity-promoting prior over the code
matrix, namely a prior whose density decreases as the sparsity penalty
$\mathcal{R}(\mathbf{Z})$ increases.  For example, we may take the Gibbs prior
\[
    p(\mathbf{Z})
    \propto
    \exp\!\left(-\lambda\mathcal{R}(\mathbf{Z})\right),
    \qquad
    \lambda\ge 0.
\]
Then the \MAP estimator of $\mathbf{Z}$ is
\[
    \mathbf{Z}_{\mathrm{MAP}}
    \in
    \argmin_{\mathbf{Z}}
    \frac{1}{2\sigma^2}
    \norm{\mathbf{H}-\mathbf{Z}\mathbf{D}}_F^2
    +
    \lambda\mathcal{R}(\mathbf{Z}).
\]
After absorbing $\sigma^2$ into the regularization weight, this is the
sparse-coding objective in \Cref{def:app-sparse-coding}.
\end{proposition}

\begin{proof}
Because the noise variables are conditionally independent across indexed
activations, the likelihood factorizes over rows:
\[
    p(\mathbf{H}\mid \mathbf{Z},\mathbf{D})
    =
    \prod_{i=1}^n
    p(\mathbf{h}_i\mid \mathbf{z}_i,\mathbf{D}).
\]
For a single activation, the Gaussian noise assumption gives
\[
    p(\mathbf{h}_i\mid \mathbf{z}_i,\mathbf{D})
    =
    \frac{1}{(2\pi\sigma^2)^{d/2}}
    \exp\!\left(
    -\frac{1}{2\sigma^2}
    \norm{\mathbf{h}_i-\mathbf{z}_i\mathbf{D}}_2^2
    \right).
\]
Multiplying these terms over $i=1,\ldots,n$ gives
\[
    p(\mathbf{H}\mid \mathbf{Z},\mathbf{D})
    =
    \frac{1}{(2\pi\sigma^2)^{nd/2}}
    \exp\!\left(
    -\frac{1}{2\sigma^2}
    \sum_{i=1}^n
    \norm{\mathbf{h}_i-\mathbf{z}_i\mathbf{D}}_2^2
    \right).
\]
Since
\[
    \sum_{i=1}^n
    \norm{\mathbf{h}_i-\mathbf{z}_i\mathbf{D}}_2^2
    =
    \norm{\mathbf{H}-\mathbf{Z}\mathbf{D}}_F^2,
\]
the likelihood can be written as
\[
    p(\mathbf{H}\mid \mathbf{Z},\mathbf{D})
    \propto
    \exp\!\left(
    -\frac{1}{2\sigma^2}
    \norm{\mathbf{H}-\mathbf{Z}\mathbf{D}}_F^2
    \right).
\]
By Bayes' rule,
\[
    p(\mathbf{Z}\mid \mathbf{H},\mathbf{D})
    \propto
    p(\mathbf{H}\mid \mathbf{Z},\mathbf{D})p(\mathbf{Z}).
\]
Substituting the likelihood and the Gibbs sparsity prior gives
\[
    p(\mathbf{Z}\mid \mathbf{H},\mathbf{D})
    \propto
    \exp\!\left(
    -\frac{1}{2\sigma^2}
    \norm{\mathbf{H}-\mathbf{Z}\mathbf{D}}_F^2
    -
    \lambda\mathcal{R}(\mathbf{Z})
    \right).
\]
Therefore maximizing the posterior is equivalent to minimizing the negative log
posterior:
\[
    -\log p(\mathbf{Z}\mid \mathbf{H},\mathbf{D})
    =
    \frac{1}{2\sigma^2}
    \norm{\mathbf{H}-\mathbf{Z}\mathbf{D}}_F^2
    +
    \lambda\mathcal{R}(\mathbf{Z})
    +
    \mathrm{const}.
\]
Dropping the constant independent of $\mathbf{Z}$ gives the stated \MAP
objective.
\end{proof}

\paragraph{Proof intuition.}
The proof only translates noisy \LRH into a Bayesian inverse problem.  Gaussian
noise says that reconstructions with smaller squared error are more likely.  A
sparsity-promoting prior says that code matrices with lower sparsity penalty are
more likely.  The posterior combines these two preferences, so its negative
logarithm is exactly a reconstruction term plus a sparsity term.  This is the
probabilistic sense in which sparse coding is the natural optimization tool for
\LRH.

This is also the point at which the Bayesian reading becomes useful for critique.
From the algebraic perspective, the intended assumptions are sparsity and
overcompleteness.  From the Bayesian perspective, we can inspect the full prior
that the objective imposes on $\mathbf{Z}$.  The next proposition shows that the
standard tokenwise sparsity term does more than impose sparsity: it also imposes
independence across indexed activations.

We now specialize to the case where the indexed activations are \textbf{patch tokens from
one image}.  In this case, $\mathbf{H}\in\R^{p\times d}$ and
$\mathbf{Z}\in\R^{p\times c}$, where $p$ is the number of patches.

\begin{proposition}[Spatial independence prior of tokenwise \SAEs]
\label{prop:app-spatial-independence}
Let $\mathbf{H}\in\R^{p\times d}$ be the patch-token activation matrix of one
image, and let $\mathbf{Z}\in\R^{p\times c}$ be the corresponding code matrix.
Suppose the vanilla \SAE objective has the row-additive form
\[
    \mathcal{L}_{\SAE}(\mathbf{Z},\mathbf{D} \, ; \, \mathbf{H})
    =
    \sum_{i=1}^p
    \ell(\mathbf{h}_i,\mathbf{z}_i;\mathbf{D})
    +
    \lambda
    \sum_{i=1}^p
    \mathcal{R}(\mathbf{z}_i),
\]
where $\ell$ is the reconstruction loss for patch $i$.  If this objective is
interpreted as a \MAP objective, then the induced prior over patch codes
factorizes as
\[
    p(\mathbf{Z})
    =
    p(\mathbf{z}_1,\ldots,\mathbf{z}_p)
    =
    \prod_{i=1}^p p(\mathbf{z}_i).
\]
If the same regularizer is used at every patch, the marginal prior has the same
form at every patch.
\end{proposition}

The proof is the spatial analogue of the temporal factorization argument in
\citet[App.~I.1]{lubana2025priors_time}; we adapt their proof strategy from the
time index of language-model activations to the patch index of vision-model
activations.

\begin{proof}
In a \MAP objective, the part of the loss that does not depend on the observed
activation is the negative log prior.  In the row-additive \SAE objective, this
prior contribution is
\[
    \lambda
    \sum_{i=1}^p
    \mathcal{R}(\mathbf{z}_i).
\]
Thus the corresponding negative log prior has the form
\[
    -\log p(\mathbf{Z})
    =
    \lambda
    \sum_{i=1}^p
    \mathcal{R}(\mathbf{z}_i)
    +
    \mathrm{const}.
\]
Exponentiating both sides gives
\[
    p(\mathbf{Z})
    \propto
    \exp\!\left(
    -\lambda
    \sum_{i=1}^p
    \mathcal{R}(\mathbf{z}_i)
    \right).
\]
Since the exponential of a sum is the product of exponentials,
\[
    p(\mathbf{Z})
    \propto
    \prod_{i=1}^p
    \exp\!\left(
    -\lambda\mathcal{R}(\mathbf{z}_i)
    \right).
\]
Define the one-patch marginal prior
\[
    p(\mathbf{z}_i)
    \propto
    \exp\!\left(
    -\lambda\mathcal{R}(\mathbf{z}_i)
    \right).
\]
Then
\[
    p(\mathbf{Z})
    =
    \prod_{i=1}^p p(\mathbf{z}_i),
\]
up to the normalizing constant.  Therefore the patch codes are independent under
the prior.  If the same $\mathcal{R}$ is used for every patch, these marginals
have the same functional form.
\end{proof}

\paragraph{Proof intuition.}
The proof uses the same energy-to-probability conversion as the \MAP derivation.
A sum of one-patch sparsity costs becomes a product of one-patch priors.  Thus
the prior knows how to make each individual patch code sparse, but it has no term
that can couple neighboring patches.

\begin{remark}[Credit and scope of the argument]
This is the spatial analogue of the temporal-prior argument of
\citet{lubana2025priors_time}.  Their result concerns the time index in
language-model activations: standard \SAEs impose priors that are independent
across time.  Here the same Bayesian calculation is applied to the patch index
of a vision model.  The conclusion is therefore spatial: a tokenwise vision
\SAE imposes an independent prior over patch codes.  For the SAE-wise prior
characterization, see Appendix~I.2 of \citet{lubana2025priors_time}, which analyzes
ReLU, TopK, BatchTopK, and JumpReLU \SAEs.
\end{remark}

The previous proposition used a row-additive penalty for clarity.  The same
structural point applies more broadly to common \SAE variants.  ReLU \SAEs impose
a sparsity prior through an $\ell_1$ penalty.  TopK and JumpReLU \SAEs impose
sparsity through constraints or pseudo-hard thresholding.  BatchTopK couples
examples through a batch-level sparsity rule.  These variants differ in how they
shape sparsity inside each code vector, but none of these tokenwise mechanisms
introduces a spatial edge interaction between neighboring patches.

The same limitation can also be seen from a symmetry perspective.  If the
objective only contains rowwise reconstruction and rowwise sparsity terms, then
it cannot distinguish the original patch layout from a layout in which all patch
rows have been permuted together.

This is \Cref{prop:perm-invariance} in the main text, which we restate and prove
here.  Let $\mathbf{P}\in\R^{p\times p}$ be a permutation matrix, and write
$\mathcal{L}_{\SAE}(\mathbf{Z},\mathbf{D};\mathbf{H})$ for the tokenwise
objective.  If the reconstruction and sparsity terms are row-additive, then
\[
    \mathcal{L}_{\SAE}(\mathbf{P}\mathbf{Z},\mathbf{D};\mathbf{P}\mathbf{H})
    =
    \mathcal{L}_{\SAE}(\mathbf{Z},\mathbf{D};\mathbf{H}),
\]
so a tokenwise \SAE objective does not distinguish a coherent patch layout from
an arbitrarily permuted layout, provided the activations and codes are permuted
together.

\begin{proof}[Proof of \Cref{prop:perm-invariance}]
The Frobenius reconstruction term is invariant under left multiplication by a
permutation matrix:
\[
    \norm{\mathbf{P}\mathbf{H}-\mathbf{P}\mathbf{Z}\mathbf{D}}_F^2
    =
    \norm{\mathbf{P}(\mathbf{H}-\mathbf{Z}\mathbf{D})}_F^2
    =
    \norm{\mathbf{H}-\mathbf{Z}\mathbf{D}}_F^2.
\]
The row-additive sparsity term is also unchanged because $\mathbf{P}$ only
reorders rows:
\[
    \sum_i
    \mathcal{R}((\mathbf{P}\mathbf{Z})_i)
    =
    \sum_i
    \mathcal{R}(\mathbf{z}_i).
\]
Combining the two equalities proves the claim.
\end{proof}

\paragraph{Proof intuition.}
A permutation changes the names and order of the patch rows, but not the multiset
of patch losses.  Since the vanilla objective contains no graph term, it has no
way to know which patches were neighbors before the permutation.

This subsection explains what was hidden in the previous one.  Sparse coding is
the \MAP formalization of the noisy \LRH model, and vanilla \SAEs approximate
that sparse-coding estimator through amortized inference.  However, the Bayesian
reading shows that the vanilla objective also adds a factorized prior over
indexed codes.  For vision, this means an independent prior over patch codes.
The next subsection replaces that spatially factorized prior with a Markov-field
prior over neighboring patch codes.

\subsection{The \MFLRH Prior}
\label{app:mflrh-prior}

The previous subsection identified the factorized prior supplied by the
vanilla tokenwise objective. The \MFLRH of \Cref{def:mf-lrh} replaces that
choice with local interactions on the patch graph, while retaining the linear
representation $\mathbf{H}\approx(\mathbf{S}+\mathbf{R})\mathbf{D}$.
The smooth field models spatially persistent content, and the innovation
field models local departures. Both fields are continuous, with attractive
interactions for $\mathbf{S}$ and stabilized repulsive interactions for
$\mathbf{R}$.

For the patch graph of \Cref{def:app-patch-graph}, write
$\kappa_{\mathcal{G}}=\lambda_{\max}(\mathbf{L})$ and
$\widetilde{\mathbf{L}}=\kappa_{\mathcal{G}}\mathbf{I}_p-\mathbf{L}$.
By the trace identities in \Cref{lem:app-laplacian-trace}, the prior in
\Cref{eq:mflrh-prior} is
$p(\mathbf{S},\mathbf{R}\mid\mathcal{G})\propto
\exp[-\mathcal{J}_{\mathbf{L}}(\mathbf{S},\mathbf{R})]$, where
\begin{equation}
\label{eq:app-mf-prior-energy}
    \mathcal{J}_{\mathbf{L}}(\mathbf{S},\mathbf{R})
    =\lambda_S\norm{\mathbf{S}}_1
       +\lambda_R\norm{\mathbf{R}}_1
       +\frac{\mu}{2}\left[
           \Tr(\mathbf{S}^\top\mathbf{L}\mathbf{S})
           +\Tr(\mathbf{R}^\top\widetilde{\mathbf{L}}\mathbf{R})
       \right],
    \quad \lambda_S,\lambda_R>0,\quad \mu\geq0.
\end{equation}
The first quadratic penalizes variation between neighbors. The second favors
contrast at a fixed field magnitude, with the shift by
$\kappa_{\mathcal{G}}$ keeping its energy nonnegative. The $\ell_1$ terms
relax the hard sparsity condition of \LRH: the continuous prior can yield
sparse posterior modes, but its samples are not exactly sparse.
\Cref{prop:mflrh-interaction-signs} below verifies normalizability and the
interaction signs on both the real code space and the nonnegative orthant.

\begin{proposition}[Attraction and repulsion in \MFLRH]
\label{prop:mflrh-interaction-signs}
Consider the continuous prior in \Cref{eq:mflrh-prior}, with
$\lambda_S,\lambda_R>0$, $\mu\geq0$, symmetric nonnegative graph weights,
and $\kappa_{\mathcal{G}}=\lambda_{\max}(\mathbf{L})$, where each undirected
edge is counted once. The fields $\mathbf{S},\mathbf{R}\in\R^{p\times c}$
range over the full real domain, or both are restricted to the nonnegative
orthant, and the $\ell_1$ penalties are entrywise. The prior is normalizable,
and its smooth and innovation fields are respectively attractive and repulsive
pairwise \MRFs, coordinatewise in the dictionary-code basis.
\end{proposition}

\intuition{Shared content and local departures}
{In \MFLRH, these preferences encourage the two fields to specialize.
The smooth field $\mathbf{S}$ explains content shared across patches,
such as objects, materials, and regions, so attraction encourages its
features to persist spatially. The innovation field $\mathbf{R}$ explains
local departures from that shared content. For example, neighboring
patches of a wall suggest a common background, while patches intersecting
a thin pole require an additional local explanation. An innovation feature
that activates at the pole and falls away in the surrounding wall expresses
precisely this spatial contrast.

Together with sparsity, stabilized repulsion encourages the innovation
stream to concentrate on such departures. The spatial penalties thus favor
persistent structure in $\mathbf{S}$ and local variation in $\mathbf{R}$,
guiding how information is allocated between the streams. This is a
modeling preference for complementary explanations; reconstruction supplies
the evidence for which innovations are needed.}

\begin{proof}
The prior energy separates between streams and across feature coordinates.
For each feature coordinate $a$, the scalar edge energies are
\[
    P^S_{ij}(s_{ia},s_{ja})
    =\frac{\mu w_{ij}}{2}(s_{ia}-s_{ja})^2,
    \qquad
    P^R_{ij}(r_{ia},r_{ja})
    =-\frac{\mu w_{ij}}{2}(r_{ia}-r_{ja})^2.
\]
Their mixed derivatives satisfy
\[
    \frac{\partial^2 P^S_{ij}}{\partial s_{ia}\,\partial s_{ja}}
    =-\mu w_{ij}\leq0,
    \qquad
    \frac{\partial^2 P^R_{ij}}{\partial r_{ia}\,\partial r_{ja}}
    =\mu w_{ij}\geq0.
\]
Thus the smooth edge energies are submodular and the innovation edge energies
are supermodular. The remaining contributions are unary:
\[
    U^S_i(s_{ia})=\lambda_S|s_{ia}|,
    \qquad
    U^R_i(r_{ia})=\lambda_R|r_{ia}|
        +\frac{\mu\kappa_{\mathcal{G}}}{2}r_{ia}^2.
\]
These scalar unary terms are modular and do not change either classification.
Each feature coordinate therefore has attractive smooth interactions and
repulsive innovation interactions on $\mathcal{G}$. Grouping coordinates by
patch preserves the pairwise Markov structure. Zero interactions, including
the case $\mu=0$, satisfy both inequalities.

It remains to verify normalizability. By \Cref{lem:app-laplacian-trace}, the
combined quadratic contribution to the prior energy is
\[
    Q(\mathbf{S},\mathbf{R})
    =\frac{\mu}{2}\left[
        \Tr(\mathbf{S}^\top\mathbf{L}\mathbf{S})
        +\Tr\bigl(\mathbf{R}^\top
            (\kappa_{\mathcal{G}}\mathbf{I}_p-\mathbf{L})\mathbf{R}\bigr)
    \right].
\]
The Laplacian term is nonnegative by its squared-difference identity, and the
shifted term is nonnegative by \Cref{cor:app-stabilized-repulsion}. Since
$\mu\geq0$, we have $Q(\mathbf{S},\mathbf{R})\geq0$. Writing the full energy as
$\mathcal{J}_{\mathbf{L}}(\mathbf{S},\mathbf{R})=
\lambda_S\norm{\mathbf{S}}_1+\lambda_R\norm{\mathbf{R}}_1
+Q(\mathbf{S},\mathbf{R})$, the unnormalized density satisfies
\[
    0<\exp\!\bigl(-\mathcal{J}_{\mathbf{L}}(\mathbf{S},\mathbf{R})\bigr)
    \leq
    \exp\!\left(-\lambda_S\norm{\mathbf{S}}_1
                -\lambda_R\norm{\mathbf{R}}_1\right).
\]
The upper bound is integrable over the full real domain and hence over the
nonnegative orthant. Strict positivity gives a positive normalizing constant,
so $0<Z<\infty$ on either domain.
\end{proof}

\subsection{From \MFLRH to \SpatialSAE}
\label{app:continuous-mf-lrh-to-spatial-sae}

We now keep the Gaussian observation model from \Cref{app:vanilla-map} and
derive the inference objective under the \MFLRH prior. We then obtain the
practical \SpatialSAE objective by amortizing this estimator and establish
that its graph dependence is invariant to arbitrary relabelings of the
patches.

\begin{proof}[Proof of \Cref{thm:spatial-map}]
Fix the dictionary $\mathbf{D}$ and the externally specified graph
$\mathcal{G}$. Under
$\mathbf{H}=(\mathbf{S}+\mathbf{R})\mathbf{D}+\boldsymbol{\Psi}$ with
independent $\boldsymbol{\psi}_i\sim\mathcal{N}(0,\sigma^2\mathbf{I}_d)$,
Bayes' rule gives the negative log posterior, up to a constant independent
of the fields, as
\begin{equation}
\label{eq:app-mf-map-energy}
    E_{\mathbf{L}}(\mathbf{S},\mathbf{R};\mathbf{H},\mathbf{D})
    :=\frac{1}{2\sigma^2}
    \norm{\mathbf{H}-(\mathbf{S}+\mathbf{R})\mathbf{D}}_F^2
    +\mathcal{J}_{\mathbf{L}}(\mathbf{S},\mathbf{R}).
\end{equation}
The independent Gaussian likelihood supplies the reconstruction term, and
\Cref{eq:app-mf-prior-energy} supplies the prior term. The \MAP estimator
therefore minimizes $E_{\mathbf{L}}$. Multiplying this energy by
$\sigma^2>0$ preserves its minimizers. Absorbing that factor into
$\lambda_S$, $\lambda_R$, and $\mu$ gives exactly \Cref{eq:spatial-map}.
\end{proof}

\paragraph{Amortizing the spatial estimator.}
As in \Cref{def:app-amortized-inference}, an encoder replaces a separate
optimization for each image by a shared prediction,
$(\mathbf{S}_{\saeparams},\mathbf{R}_{\saeparams})
=F_{\saeparams}(\mathbf{H};\mathbf{L})$.
Substituting these fields into the rescaled posterior energy and averaging
over images gives the \SpatialSAE objective in
\Cref{eq:spatialsae-objective}. The dictionary is learned with unit-norm
atoms, $\norm{\mathbf{d}_a}_2=1$. The encoder restricts the available codes,
so its predictions approximate the spatial \MAP estimator with a possible
amortization gap, just as in vanilla sparse coding.

The dictionary partition in \Cref{sec:ssae-encoder} restricts the two
fields to complementary coordinate blocks, with zeros elsewhere. This gives
$\abs{\supp(\mathbf{z}_i)}
=\abs{\supp(\mathbf{s}_i)}+\abs{\supp(\mathbf{r}_i)}$;
such additivity need not hold in general.
The exact \MAP derivation above uses positive $\ell_1$ coefficients.
TopK and BatchTopK instead enforce sparsity through the encoder;
\Cref{app:real-impl} details these variants and the width normalization
of the spatial penalties.

\paragraph{Graph layout and patch indexing.}
The preceding derivation makes the graph dependence of the objective
explicit. We now prove the symmetry claim from \Cref{sec:spatialsae-map}:
this dependence is on the spatial relations encoded by the graph rather than
on the arbitrary indices assigned to its vertices.

\begin{proposition}[Spatial layout and patch indexing]
\label{prop:app-layout-not-order}
Let $E_{\mathbf{L}}$ be the objective in \Cref{eq:app-mf-map-energy},
with $c\geq1$ and the loop-free patch graph of
\Cref{def:app-patch-graph}. The fields range over the real code space or
the nonnegative orthant. For any permutation matrix
$\mathbf{P}\in\R^{p\times p}$, simultaneously relabeling the activations,
fields, and graph preserves the energy:
\[
    E_{\mathbf{P}\mathbf{L}\mathbf{P}^\top}
    (\mathbf{P}\mathbf{S},\mathbf{P}\mathbf{R};
     \mathbf{P}\mathbf{H},\mathbf{D})
    =E_{\mathbf{L}}(\mathbf{S},\mathbf{R};\mathbf{H},\mathbf{D}).
\]
If $\mu>0$ and the graph is held fixed, the identity
\[
    E_{\mathbf{L}}(\mathbf{P}\mathbf{S},\mathbf{P}\mathbf{R};
                   \mathbf{P}\mathbf{H},\mathbf{D})
    =E_{\mathbf{L}}(\mathbf{S},\mathbf{R};\mathbf{H},\mathbf{D})
\]
holds for every activation matrix and pair of fields if and only if
$\mathbf{P}^\top\mathbf{L}\mathbf{P}=\mathbf{L}$. Equivalently,
$\mathbf{P}$ preserves weighted adjacency:
$\mathbf{P}^\top\mathbf{W}\mathbf{P}=\mathbf{W}$.
\end{proposition}

\begin{proof}
Permutation of rows preserves the Frobenius reconstruction loss and the
entrywise $\ell_1$ penalties. Relabeling the graph also preserves its
spectrum, so for $\mathbf{L}'=\mathbf{P}\mathbf{L}\mathbf{P}^\top$,
\[
    \widetilde{\mathbf{L}}'
    =\kappa_{\mathcal{G}}\mathbf{I}_p-\mathbf{L}'
    =\mathbf{P}\widetilde{\mathbf{L}}\mathbf{P}^\top.
\]
For either $\mathbf{A}=\mathbf{L}$ or
$\mathbf{A}=\widetilde{\mathbf{L}}$, orthogonality of $\mathbf{P}$ gives
\[
    \Tr\!\left((\mathbf{P}\mathbf{X})^\top
        (\mathbf{P}\mathbf{A}\mathbf{P}^\top)
        (\mathbf{P}\mathbf{X})\right)
    =\Tr(\mathbf{X}^\top\mathbf{A}\mathbf{X}).
\]
This proves invariance under simultaneous relabeling. With the graph
fixed, $\mathbf{P}^\top\mathbf{L}\mathbf{P}=\mathbf{L}$ also implies
$\mathbf{P}^\top\widetilde{\mathbf{L}}\mathbf{P}
=\widetilde{\mathbf{L}}$, so both spatial penalties are unchanged.

Conversely, suppose the fixed-graph identity holds for every pair of
fields. Set $\mathbf{R}=0$ and let a single column of $\mathbf{S}$ be
$\mathbf{x}$, with all other columns zero. The reconstruction and
sparsity terms are invariant under the simultaneous row permutation.
Since $\mu>0$, equality of the spatial penalties requires
\[
    \mathbf{x}^\top
    (\mathbf{P}^\top\mathbf{L}\mathbf{P}-\mathbf{L})
    \mathbf{x}=0
\]
for every admissible $\mathbf{x}$. Testing the coordinate vectors
$\mathbf{e}_i$ and their sums $\mathbf{e}_i+\mathbf{e}_j$ forces every
entry of this symmetric matrix to vanish. These tests are valid on both
domains. Finally, for a loop-free graph, the off-diagonal entries of
$\mathbf{L}$ are $-w_{ij}$ and its diagonal entries are the weighted
degrees. Thus preservation of $\mathbf{L}$ is equivalent to preservation
of weighted adjacency.
\end{proof}

At $\mu=0$, all row permutations preserve the objective even with the
graph fixed. For $\mu>0$, the proposition characterizes symmetries of the
whole objective; particular fields may still have equal energy under
additional permutations. Together with the \MAP derivation, this explains
how \SpatialSAE uses spatial relations without depending on patch indexing.

\vfill\pagebreak
\section{Synthetic Experiments}
\label{app:synthetic-experiments}

\begin{figure}[t]
    \centering
    \includegraphics[width=\linewidth]{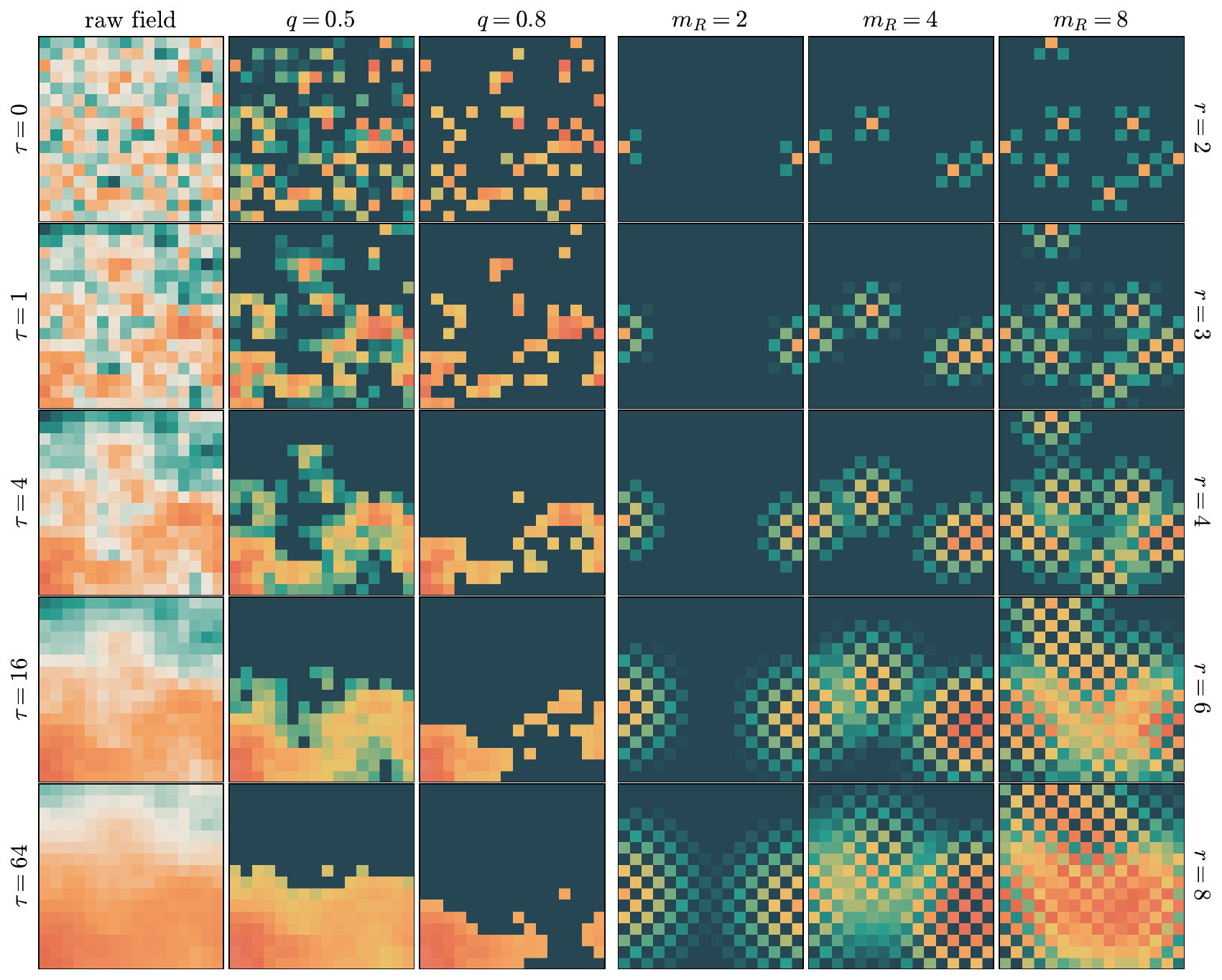}
    \caption{\textbf{Controlling smooth- and innovation-stream structure.}
    Left: rows increase the coupling $\tau$, making the raw field
    progressively smoother, while the two thresholded columns show how a
    larger support quantile $q$ reduces its active area. Right: rows increase
    the packet radius $r$, expanding each innovation packet, while columns
    increase the number $m_R$ of active innovation concepts, producing more
    packets and greater overlap. Random draws and packet centers are held fixed
    across comparable panels.}
    \label{fig:app-dgp-stream-structure}
\end{figure}
Real vision-model activations do not reveal the dictionary or latent codes that
generated them, so they cannot directly tell us whether an \SAE recovers the
underlying concepts. We therefore construct a controlled benchmark with known
concept directions and firing patterns. We describe the benchmark in
\Cref{app:synthetic-dgp} and report the component, geometry, and capacity
results in \Cref{app:synth-results}.

\subsection{Experimental Setup}
\label{app:synthetic-dgp}

The synthetic activations contain smooth concepts and sparse local innovations,
matching the qualitative structure observed in real vision representations
(as discussed in \Cref{par:patch}), without being sampled from the exact \MFLRH prior. This
tests whether the spatial prior remains useful when the data exhibit the
relevant structure without matching its assumed distribution exactly.

\paragraph{Concept dictionary and data dimensions.}
\label{app:synth-dictionary}

We add spatial structure to existing synthetic \SAE benchmarks
\citep{chanin2026synthsaebench} as follows: for each dataset, we sample $c$
concept directions independently from
$\mathcal N(0,\mathbf I_d)$, normalize them to unit length, and fix the
resulting dictionary $\mathbf D\in\R^{c\times d}$ across images in the
linear observation model of \Cref{thm:spatial-map}. Each image has $p=256$
patches on a $16\times16$ grid, giving
$\mathbf H\in\R^{p\times d}$ and smooth and innovation codes
$\mathbf S,\mathbf R\in\R_{\geq0}^{p\times c}$.
We use the unweighted four-neighbor grid graph and denote its combinatorial
Laplacian by $\mathbf L$.
Each experiment therefore has activation dimension $d$, true dictionary width
$c$, and learned \SAE width $w$. We vary the superposition ratio
$\rho=c/d$, equivalently the overcompleteness of the true dictionary, and
narrowness $\nu=w/c$ independently.
Since we fix $c=256$ and vary $d$,
increasing $\rho$ packs the same number of true concepts into a lower-dimensional
activation space, while decreasing $\nu$ restricts the number of learned atoms
available to recover them.

\begin{figure}[t]
    \centering
    \includegraphics[width=\linewidth]{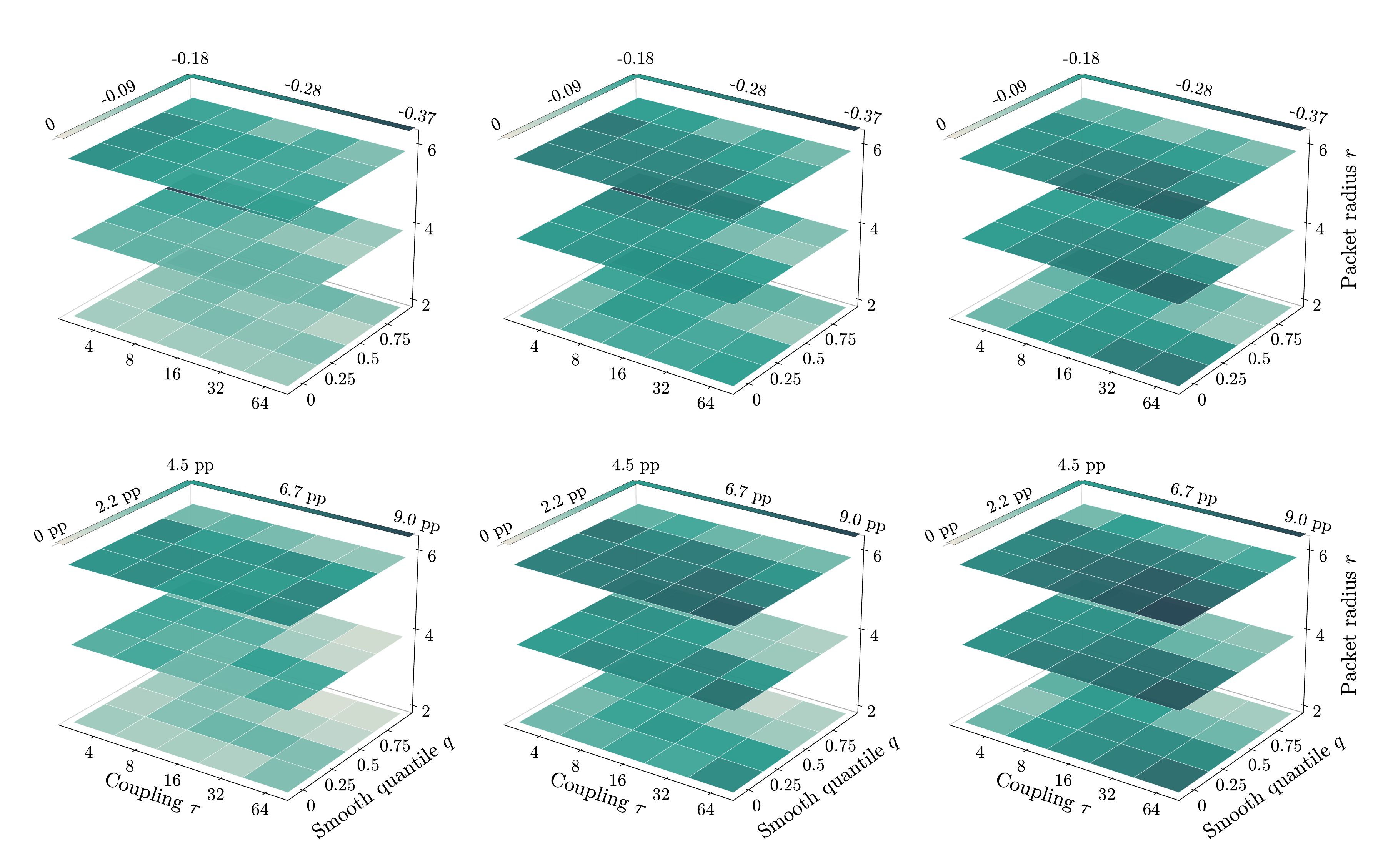}
    \caption{\textbf{Geometry sweep relative to \textbf{B}.} Differences
    between $\textbf{DT}$ and $\textbf{B}$ across the $60$ geometry cells.
    Columns show $\mathrm{L}$, $\mathrm{M}$, and $\mathrm{H}$; the upper row
    shows the polysemanticity difference and the lower row the F1 difference
    in percentage points. Each horizontal slice fixes $r$, while the in-plane
    axes vary $\tau$ and $q$. Teal denotes an improvement, while the pale
    reference color denotes no change. The predominantly teal surfaces show
    that $\textbf{DT}$ improves substantially over $\textbf{B}$ throughout
    the grid.}
    \label{fig:app-synth-geometry-vs-b}
\end{figure}

\begin{figure}[t]
    \centering
    \includegraphics[width=\linewidth]{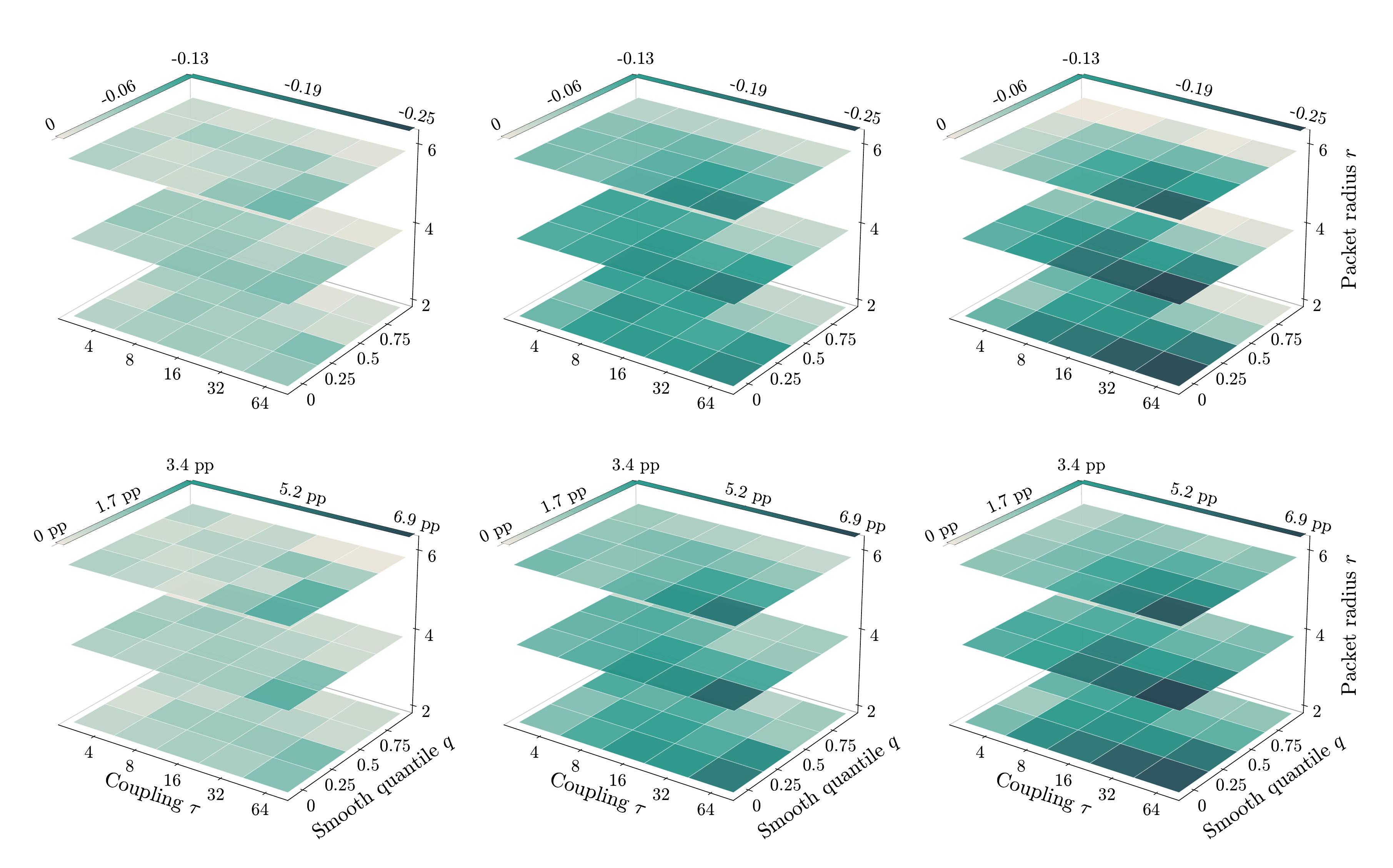}
    \caption{\textbf{Geometry sweep relative to \textbf{D}.} Differences
    between $\textbf{DT}$ and $\textbf{D}$ across the $60$ geometry cells.
    Columns show $\mathrm{L}$, $\mathrm{M}$, and $\mathrm{H}$; the upper row
    shows the polysemanticity difference and the lower row the F1 difference
    in percentage points. Each horizontal slice fixes $r$, while the in-plane
    axes vary $\tau$ and $q$. Teal denotes an improvement, while the pale
    reference color denotes no change. The variation across $\mathrm{L}$,
    $\mathrm{M}$, and $\mathrm{H}$ shows that the benefit over $\textbf{D}$
    depends on the trace strength.}
    \label{fig:app-synth-geometry-vs-d}
\end{figure}

\begin{figure}[t]
    \centering
    \includegraphics[width=\linewidth]{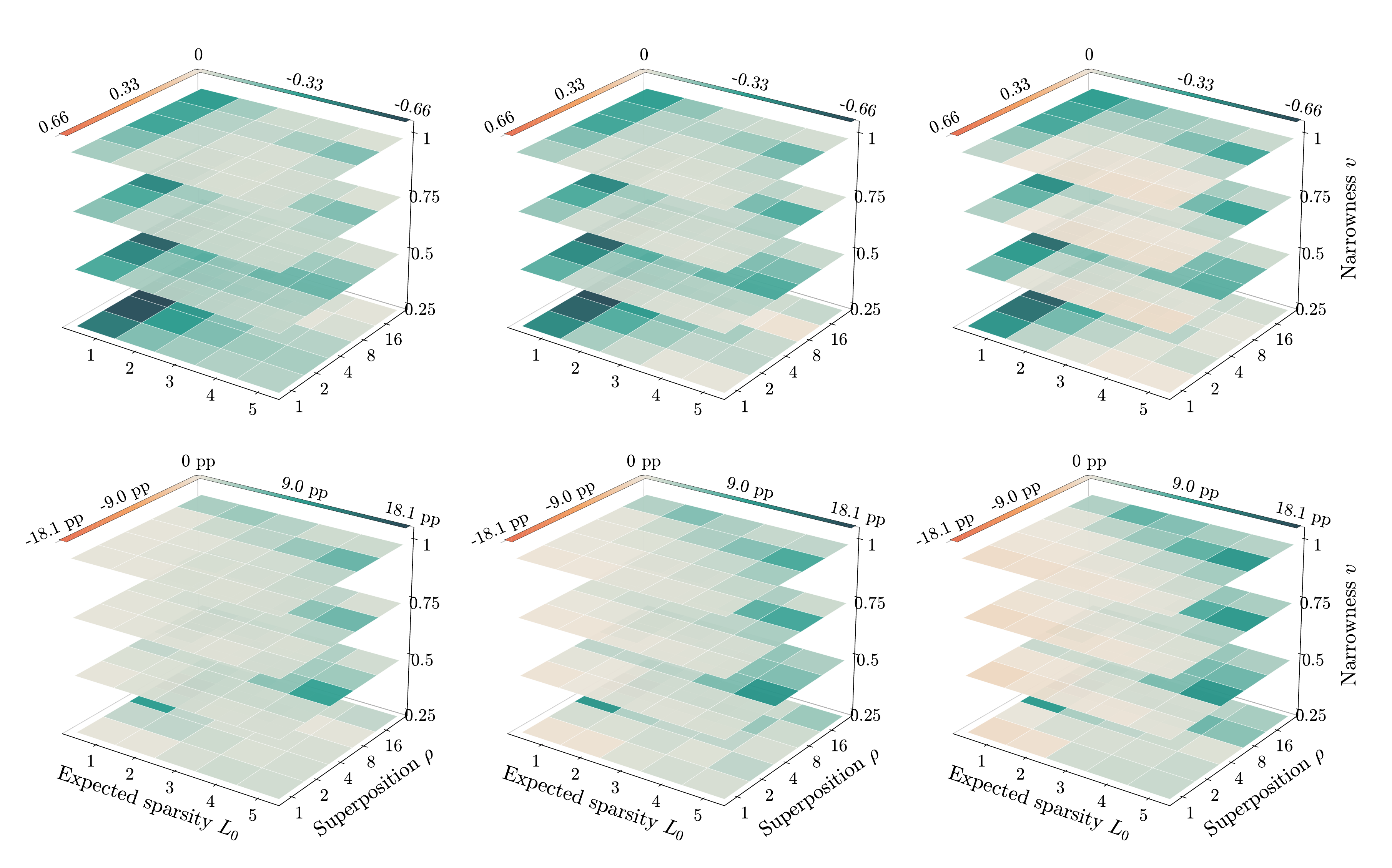}
    \caption{\textbf{Capacity sweep relative to \textbf{B}.} Differences
    between $\textbf{DT}$ and $\textbf{B}$ across the $100$ capacity cells.
    Columns show $\mathrm{L}$, $\mathrm{M}$, and $\mathrm{H}$; the upper row
    shows the polysemanticity difference and the lower row the F1 difference
    in percentage points. Each horizontal slice fixes $\nu$, while the
    in-plane axes vary $L_0$ and $\rho$. Teal denotes an improvement and
    orange a deterioration.}
    \label{fig:app-synth-capacity-vs-b}
\end{figure}

\begin{figure}[t]
    \centering
    \includegraphics[width=\linewidth]{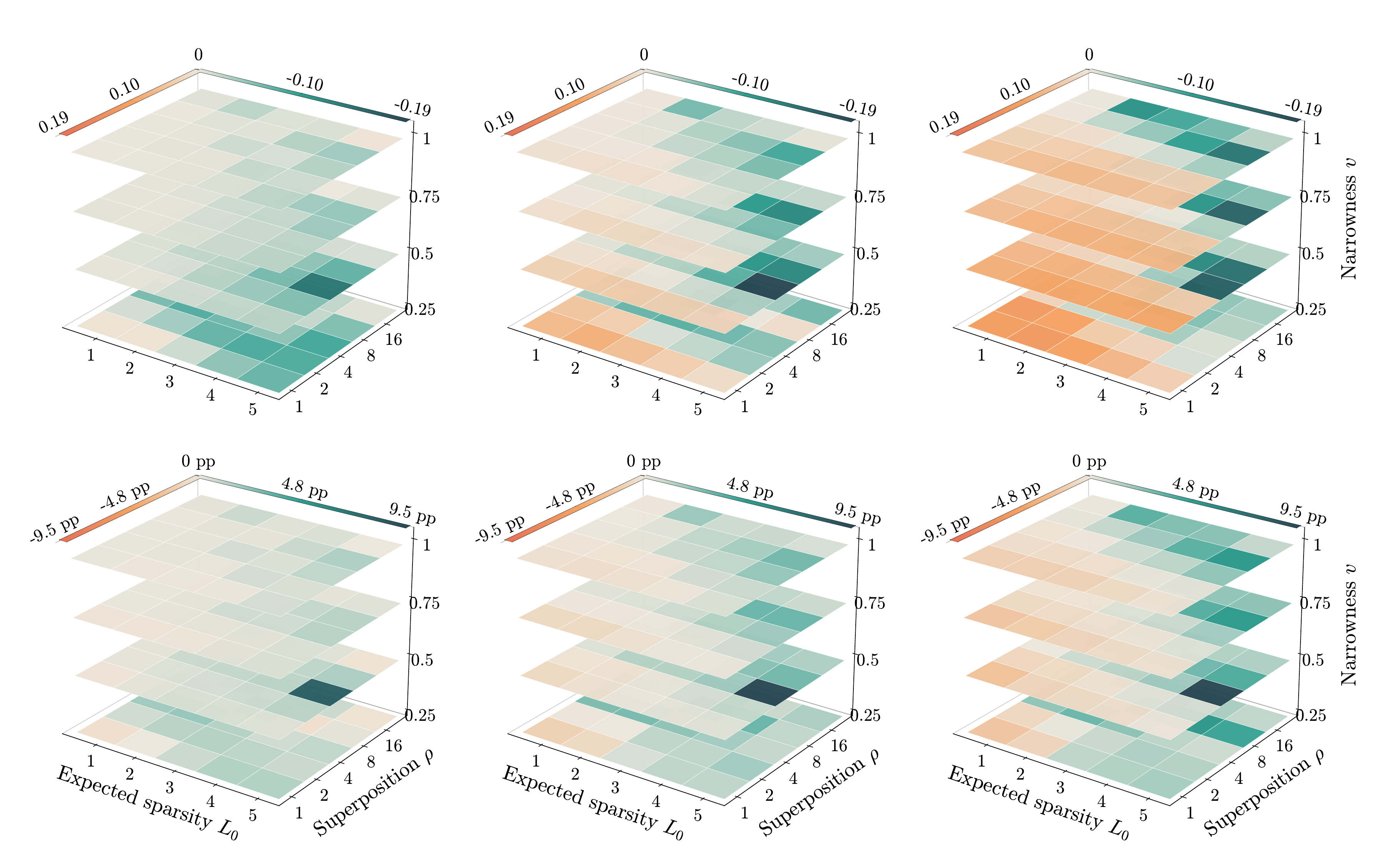}
    \caption{\textbf{Capacity sweep relative to \textbf{D}.} Differences
    between $\textbf{DT}$ and $\textbf{D}$ across the $100$ capacity cells.
    Columns show $\mathrm{L}$, $\mathrm{M}$, and $\mathrm{H}$; the upper row
    shows the polysemanticity difference and the lower row the F1 difference
    in percentage points. Each horizontal slice fixes $\nu$, while the
    in-plane axes vary $L_0$ and $\rho$. Teal denotes an improvement and
    orange a deterioration.}
    \label{fig:app-synth-capacity-vs-d}
\end{figure}

\paragraph{Smooth and innovation streams.}
\label{app:dgp-process}

For each image, we sample an active set $\mathcal A_S\subseteq[c]$ uniformly
without replacement, with $|\mathcal A_S|=m_S$. For each
$a\in\mathcal A_S$, we sample
$\boldsymbol\psi^{(S)}_a\sim\mathcal N(0,\mathbf I_p)$, diffuse it with
$\mathbf T_\tau=(\mathbf I_p+\tau\mathbf L)^{-1}$, rectify it, and retain
the values above its empirical $q$-quantile. Writing
\[
    \mathbf g_a=(\mathbf T_\tau\boldsymbol\psi^{(S)}_a)_+,
    \qquad
    \mathbf u_a=\mathbf g_a\odot\Ind\{\mathbf g_a>Q_q(\mathbf g_a)\},
\]
we set $\mathbf S_{:,a}=\rho_S\mathbf u_a/\|\mathbf u_a\|_\infty$, where
$\rho_S>0$ is the smooth-stream peak amplitude. In practice, we clamp the
normalization denominator from below at $\epsilon=10^{-8}$ and set columns
outside $\mathcal A_S$ to zero. The left half of
\Cref{fig:app-dgp-stream-structure} shows how $\tau$ controls spatial
coherence and $q$ controls support area. When the
threshold is positive and there are no ties, approximately $1-q$ of patches
remain active. At $\tau=0$, the sampler has no preference for adjacent patches.

For the innovation stream, we independently sample an active set
$\mathcal A_R\subseteq[c]$ uniformly without replacement, with
$|\mathcal A_R|=m_R$. We assign each $a\in\mathcal A_R$ a grid center $c_a$,
sampled independently and uniformly from the $p$ patches, with replacement.
Let $\delta(i,c_a)$ denote the Euclidean distance from patch $i$ to that
center. Writing $(x_i,y_i)$ for the grid coordinates of patch $i$, define
\[
    b_i^{(a)}
    =
    \Ind\!\left\{
        (x_i-x_{c_a})+(y_i-y_{c_a})\equiv 0 \pmod 2
    \right\},
\]
so that $\mathbf b^{(a)}\in\{0,1\}^p$ indicates the side of the grid's
bipartition containing $c_a$. A compact triangular envelope
\[
    e_{ia}=\left(1-\frac{\delta(i,c_a)}{r}\right)_+
\]
is restricted by this mask, giving
$R_{ia}=\rho_R b_i^{(a)}e_{ia}$ for active concepts and zero otherwise, where
$\rho_R>0$ is the innovation-stream peak amplitude.
Hence $r$ controls the spatial extent of each packet and the bipartite mask
gives the packet its locally alternating structure.
\Cref{fig:synth-smooth} illustrates the construction of the smooth and
innovation streams side by side, while the right half of
\Cref{fig:app-dgp-stream-structure} visualizes the effects of $r$ and $m_R$.
Because the two streams are sampled independently, their active concept sets
and spatial supports may overlap. Both streams are mapped through the same
dictionary $\mathbf D$. Finally, we form
\[
    \mathbf H=(\mathbf S+\mathbf R)\mathbf D+\boldsymbol\Psi,
    \qquad
    \boldsymbol\Psi_i\overset{\mathrm{iid}}{\sim}
    \mathcal N(0,\sigma^2\mathbf I_d).
\]

\begin{table}[b]
    \centering
    \footnotesize
    \setlength{\tabcolsep}{2.2pt}
    \renewcommand{\arraystretch}{1.2}
    \begin{tabular*}{\linewidth}{@{\extracolsep{\fill}}ccccccccccc@{}}
        \toprule
        & \multirow{2}{*}{\raisebox{-0.5ex}{$\textbf{D}>\textbf{B}$}}
        & \multicolumn{3}{c}{$\textbf{DT}>\textbf{B}$}
        & \multicolumn{3}{c}{$\textbf{DT}>\textbf{D}$}
        & \multicolumn{3}{c}{$\textbf{DT}>\textbf{D}>\textbf{B}$} \\
        \cmidrule(lr){3-5}\cmidrule(lr){6-8}\cmidrule(lr){9-11}
        & & $\mathrm{L}$ & $\mathrm{M}$ & $\mathrm{H}$
        & $\mathrm{L}$ & $\mathrm{M}$ & $\mathrm{H}$
        & $\mathrm{L}$ & $\mathrm{M}$ & $\mathrm{H}$ \\
        \midrule
        \textbf{MCC}
        & $99\%$ & $100\%$ & $99\%$ & $89\%$ & $96\%$ & $83\%$ & $64\%$ & $96\%$ & $82\%$ & $64\%$ \\
        \textbf{Poly.}
        & $99\%$ & $100\%$ & $99\%$ & $89\%$ & $96\%$ & $80\%$ & $64\%$ & $94\%$ & $79\%$ & $63\%$ \\
        \textbf{F1}
        & $98\%$ & $100\%$ & $96\%$ & $87\%$ & $91\%$ & $85\%$ & $79\%$ & $89\%$ & $83\%$ & $77\%$ \\
        $\mathbf{R}^2$
        & $13\%$ & $10\%$ & $1\%$ & $0\%$ & $15\%$ & $1\%$ & $0\%$ & $4\%$ & $0\%$ & $0\%$ \\
        \bottomrule
    \end{tabular*}
    {\footnotesize\par\vspace{0.25em}\noindent
    $\mathrm{L}$, $\mathrm{M}$, and $\mathrm{H}$ denote
    $\mu_S=3\times10^{-4}$, $10^{-3}$, and $3\times10^{-3}$,
    respectively.\par}
    \vspace{-0.6em}
    \caption{\textbf{Ablating \SpatialSAE components.} Each entry is the
    strict win rate of the indicated comparison across the combined Geometry
    and Capacity Sweeps, rounded to the nearest
    integer. The no-trace $\textbf{D}>\textbf{B}$ column isolates
    decomposition. The remaining groups introduce spectral regularization:
    $\textbf{DT}>\textbf{D}$ isolates its incremental effect,
    $\textbf{DT}>\textbf{B}$ measures the combined effect, and the final group
    requires both improvements.}
    \label{tab:app-synth-attribution-wins}
\end{table}

\paragraph{Access levels and sweeps.}
\label{app:dgp-sweeps}
\label{app:synth-protocol}
We compare three models at matched learned width and sparsity. The baseline
$\textbf{B}$ is a standard single-stream \SAE trained on the combined
activations. The decomposed model $\textbf{D}$ is given the
separate smooth and innovation streams but uses no spectral regularization.
Finally, $\textbf{DT}$ has the same inputs and architecture as $\textbf{D}$
and adds the trace penalties of \Cref{eq:spatialsae-objective}. Consequently,
$\textbf{B}\to\textbf{D}$ isolates the effect of access to the decomposition,
whereas $\textbf{D}\to\textbf{DT}$ isolates the effect of the trace prior. In
every $\textbf{DT}$ experiment, we set $\mu_R=\mu_S/3$.

We use two configuration grids. The first varies the geometry of both latent
streams over
\[
    \tau\times q\times r
    =\{4,8,16,32,64\}\times\{0,0.25,0.5,0.75\}\times\{2,4,6\}.
\]
We exclude $q=1$ because thresholding above the empirical maximum would set
the entire smooth stream to zero. Throughout this geometry grid, we fix the
activation dimension at $d=32$, the superposition ratio at $\rho=8$, the
narrowness at $\nu=0.5$, the active-concept counts at $m_S=9$ and $m_R=4$,
and the observation noise at $\sigma=0.2$. For each cell, we estimate the true
mean per-patch $L_0$ and use its nearest integer as the sparsity target.
The second varies the sparsity target, activation dimension, and learned
dictionary narrowness over the $100$-cell capacity grid
\[
    L_0\times d\times\nu
    =\{1,2,3,4,5\}\times\{16,32,64,128,256\}
    \times\{0.25,0.5,0.75,1\}.
\]
Throughout this capacity grid, we fix the true dictionary width at $c=256$,
the spatial geometry at $(\tau,q,r)=(16,0.5,4)$, the active-concept counts at
$m_S=9$ and $m_R=4$, and the observation noise at $\sigma=0.2$. The swept
values of $d$ and $\nu$ therefore give
\[
    \rho=\frac{c}{d}\in\{16,8,4,2,1\},
    \qquad
    w=\nu c\in\{64,128,192,256\}.
\]
\looseness=-1The geometry and capacity grids share the default configuration
$(\tau,q,r)=(16,0.5,4)$ and $(L_0,d,\nu)=(4,32,0.5)$. Their union therefore
contains $60+100-1=159$ distinct configurations.

We run three matched sweeps:
\vspace{-0.5em}
\begin{enumerate}[leftmargin=*]
    \item \textbf{Sparsifier Generalization Sweep.} On all $60$ geometry
    cells, we compare $\textbf{DT}$ with $\textbf{B}$ using ReLU, TopK,
    BatchTopK, and JumpReLU, with five seeds per cell and $\mu_S=10^{-3}$,
    giving $2{,}400$ trained \SAEs and $300$ matched comparisons per
    sparsifier. The results are presented in \Cref{tab:synth-wins}.
    \item \textbf{Geometry Sweep.} On the same $60$ geometry cells, we use
    only TopK and a single shared seed to compare $\textbf{B}$, $\textbf{D}$,
    and $\textbf{DT}$. For $\textbf{DT}$, we use three trace strengths,
    denoted $\mathrm{L}$, $\mathrm{M}$, and $\mathrm{H}$, corresponding to
    $\mu_S=3\times10^{-4}$, $10^{-3}$, and $3\times10^{-3}$, respectively.
    \item \textbf{Capacity Sweep.} We repeat the preceding single-seed
    comparison on all $100$ capacity cells.
\end{enumerate}

\vspace{-0.5em}
\paragraph{Training and evaluation.}
Unless stated otherwise, models are trained for $5{,}000$ Adam steps with
batches of $16$ images and decoder normalization after each update. Matched
models share the generator configuration and seed and are evaluated on the same
$2{,}000$ held-out images. We build our implementation on SynthSAEBench
\citep{chanin2026synthsaebench} and reuse its implementations of \textsc{MCC},
support F1, and held-out $R^2$. We port the polysemanticity metric
(\textsc{Poly.}) from \citet{mencattini2026rate}. A strict increase counts as
a win for \textsc{MCC}, F1, and $R^2$, and a strict decrease for
polysemanticity; ties are not wins. The full-ladder event
$\textbf{DT}>\textbf{D}>\textbf{B}$ requires both strict improvements in the
same configuration.

\subsection{Effects of Decomposition and Spectral Regularization}
\label{app:synth-results}

\paragraph{Component ablation.}
\Cref{tab:app-synth-attribution-wins} summarizes the Geometry and Capacity
Sweeps as strict win rates over their $159$ distinct configurations. The
recovery metrics exhibit a clear ladder: giving the model access to the
decomposition is already highly beneficial, and adding spectral regularization
on top yields further gains. Specifically, $\textbf{D}$ beats $\textbf{B}$ in
$99\%$, $99\%$, and $98\%$ of configurations on \textsc{MCC},
polysemanticity, and F1, respectively. At the low trace strength,
$\textbf{DT}$ then beats $\textbf{D}$ in $96\%$, $96\%$, and $91\%$ of
configurations, so the full $\textbf{DT}>\textbf{D}>\textbf{B}$ ladder holds
in $96\%$, $94\%$, and $89\%$. The ladder weakens as the trace strength
increases, while $R^2$ rarely follows it, emphasizing that the trace
coefficient must be tuned for recovery rather than reconstruction.

\paragraph{Geometry Sweep.}
Across changes in smooth coupling, support, and packet radius, the trace prior
provides a robust benefit
(\Cref{fig:app-synth-geometry-vs-b,fig:app-synth-geometry-vs-d}). Against
$\textbf{B}$, $\textbf{DT}$ improves both polysemanticity and F1 throughout
the geometry grid at every trace strength, producing the predominantly teal
surfaces in \Cref{fig:app-synth-geometry-vs-b}. This advantage persists after
accounting for the decomposition: against $\textbf{D}$, F1 improves throughout,
while polysemanticity improves across the full grid at $\mathrm{L}$ and
$\mathrm{M}$. At $\mathrm{H}$, the polysemanticity benefit becomes less
uniform, but the remaining cells show no change rather than a deterioration.
Thus, once the data have the intended spatial structure, spectral
regularization provides a substantial and robust gain beyond decomposition
alone.

\paragraph{Capacity Sweep.}
Across changes in sparsity, superposition, and narrowness, the benefits of the
trace prior depend more strongly on its strength
(\Cref{fig:app-synth-capacity-vs-b,fig:app-synth-capacity-vs-d}). At
$\mathrm{L}$, $\textbf{DT}$ improves over $\textbf{B}$ on both plotted metrics
in $100\%$ of configurations. By $\mathrm{H}$, however, it performs worse on
polysemanticity in $18\%$ of configurations and on F1 in $20\%$. This
sensitivity becomes more pronounced after accounting for the decomposition:
against $\textbf{D}$, the high trace strength worsens polysemanticity in
$49\%$ of configurations and F1 in $33\%$, making the mean polysemanticity
change slightly unfavorable. Unlike in the Geometry Sweep, a mistuned trace
coefficient can therefore erase the advantage over the baseline and become
actively detrimental once the decomposition is already available.

\clearpage

\vfill\pagebreak
\section{Experiments with Natural Images}
\label{sec:experiment-details}

\begin{figure}[t]
    \centering
    \includegraphics[width=\linewidth]{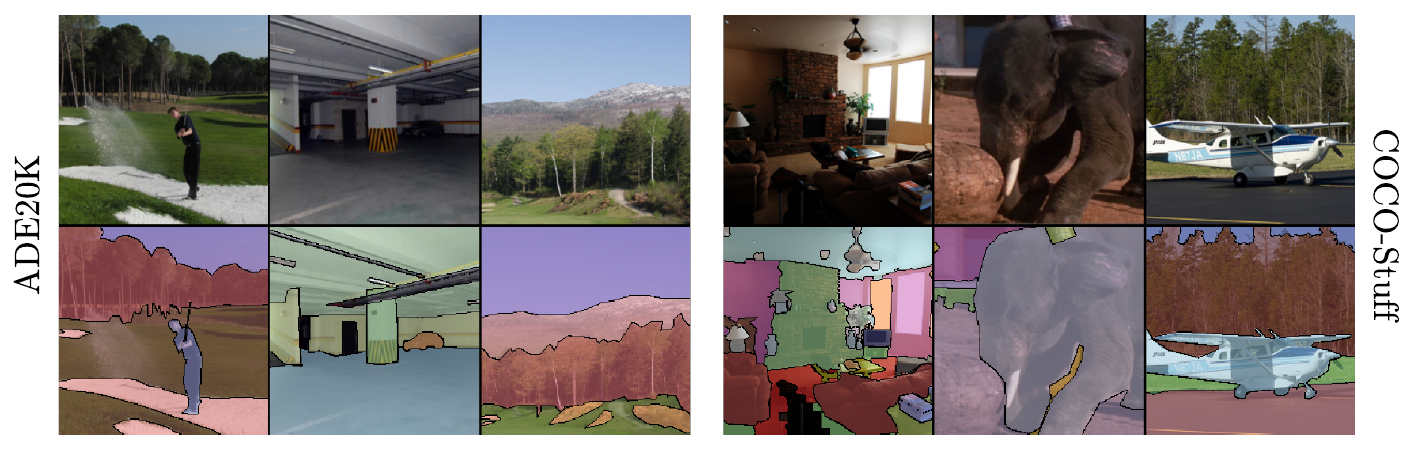}
    \caption{\textbf{Spatial annotations for concept evaluation.}
    Images from \ADETwentyK (left) and \COCOStuff (right), paired with
    their semantic segmentation masks. The annotations identify the
    spatial extent of objects and scene regions, allowing us to evaluate
    whether learned features activate on the concepts they represent.}
    \label{fig:dataset-samples}
\end{figure}

The natural-image experiments in \Cref{sec:wild-eval} test whether spatial
structure improves the concepts recovered from vision-model activations. This
appendix provides the complete experimental account. We first describe the
datasets, frozen backbone, patch graph, and compute environment. In
\Cref{app:real-impl}, we detail optimization, model architecture, sparsification,
and spatial regularization. \Cref{app:metrics} defines the evaluation metrics
and spectral diagnostics, while \Cref{app:real-autointerp} describes the
automated-interpretability protocol. Finally, \Cref{app:real-ablations} reports
the ablations, model-selection procedure, transfer to \COCOStuff, and
additional qualitative results.

\paragraph{Datasets and backbone.}
Judging whether an atom captures a concept requires ground truth that is
localized in space: an image-level label cannot tell whether an atom fires
on the object it is named after or on its surroundings. We therefore
evaluate on densely annotated scene-parsing datasets, where every pixel
carries a semantic label, and use \ADETwentyK \citep{zhou2019semantic} and
\COCOStuff \citep{cocostuff}. \Cref{fig:dataset-samples} shows example images and their
segmentation annotations. \ADETwentyK contains $25{,}574$ training
images spanning $365$ scene categories and thousands of annotated object
categories; its derived SceneParse150 benchmark uses $20{,}210$ training
images and $150$ semantic classes. \COCOStuff augments the $118{,}287$
COCO train2017 images with dense pixel-level annotations over $171$
semantic classes, plus an unlabeled/void class.
Both datasets mix thing classes (person, car) with
stuff classes (sky, wall, road), and a typical image contains several
concepts at different scales. Concepts therefore occupy both large uniform
regions and small or thin structures, which is the mix of smooth and
high-frequency content that the two-stream decomposition targets. The
pixel-level masks let us assign concepts to individual patches, so the same
annotations support patch-level F1, region coherence over annotated
segments, and the concept spectra used as spectral references. The two
datasets also differ in scale, taxonomy, and domain, so \COCOStuff tests
whether configurations selected on \ADETwentyK transfer without re-tuning.
We extract activations from a frozen \DINO ViT-B/14 with four register
tokens \citep{oquab2023dinov2,darcet2024registers}.\footnote{\texttt{timm} model: \texttt{vit\_base\_patch14\_reg4\_dinov2} (pretrained).}
\ADETwentyK is one of the semantic segmentation benchmarks used by
\citet{oquab2023dinov2} to evaluate frozen patch features with a linear
classifier. A $224\times224$ input yields a
$16\times16$ patch grid, so each image provides $p=256$ patch activations
with dimension $d=768$. All spatial models use a fixed, unweighted
graph connecting distinct patches whose Manhattan distance is at most
two, with the same graph used for every image.
As in \citet{fel2025rabbit_hull}, we take the activations of the final layer
after normalization; unlike them, we keep only the $256$ patch tokens,
discarding the \CLS and register tokens, and store the patch activations in
half precision (\texttt{float16}).

\begin{table}[t]
    \centering
    \scriptsize
    \definecolor{budgetaccent}{HTML}{456B73}
    \setlength{\tabcolsep}{2.15pt}
    \renewcommand{\arraystretch}{1.08}
    \newcommand{\budgetstat}[2]{$#1{\scriptstyle\,\pm\,#2}$}
    \begin{tabularx}{\linewidth}{@{}cc*{8}{>{\centering\arraybackslash}X}@{}}
        \toprule
        \multirow{2}{*}{\raisebox{-0.65ex}{$\boldsymbol{\mathbb{E}}\mathbf{[L_0]}$}}
        & \multirow{2}{*}{\raisebox{-0.65ex}{\textbf{Model}}}
        & \multicolumn{2}{c}{\textbf{TopK}}
        & \multicolumn{2}{c}{\textbf{Vanilla}}
        & \multicolumn{2}{c}{\textbf{JumpReLU}}
        & \multicolumn{2}{c}{\textbf{BatchTopK}} \\
        \cmidrule(lr){3-4}\cmidrule(lr){5-6}
        \cmidrule(lr){7-8}\cmidrule(lr){9-10}
        & & $R^2$ & Norm.\ F1 & $R^2$ & Norm.\ F1
            & $R^2$ & Norm.\ F1 & $R^2$ & Norm.\ F1 \\
        \midrule
        & Baseline & \budgetstat{61.1}{0.2} & \budgetstat{58.0}{1.4}
                   & \budgetstat{40.8}{0.8} & \budgetstat{51.5}{1.0}
                   & \budgetstat{38.4}{0.7} & \budgetstat{52.9}{0.7}
                   & \budgetstat{57.9}{0.9} & \budgetstat{54.2}{0.5} \\
        \rowcolor{budgetaccent!8}
        \cellcolor{white}\multirow{-2}{*}{$8$}
        & \SpatialSAE & \budgetstat{45.1}{0.1} & \budgetstat{67.0}{1.2}
                   & \budgetstat{27.7}{0.5} & \budgetstat{55.7}{2.0}
                   & \budgetstat{30.4}{0.3} & \budgetstat{57.2}{1.2}
                   & \budgetstat{43.9}{0.3} & \budgetstat{60.2}{1.6} \\
        \addlinespace[2pt]
        & Baseline & \budgetstat{70.4}{0.2} & \budgetstat{57.5}{1.1}
                   & \budgetstat{53.7}{0.7} & \budgetstat{50.5}{1.0}
                   & \budgetstat{52.2}{0.8} & \budgetstat{52.1}{0.5}
                   & \budgetstat{69.2}{0.9} & \budgetstat{55.9}{0.7} \\
        \rowcolor{budgetaccent!8}
        \cellcolor{white}\multirow{-2}{*}{$16$}
        & \SpatialSAE & \budgetstat{52.4}{0.3} & \budgetstat{65.4}{1.9}
                   & \budgetstat{39.1}{0.7} & \budgetstat{55.5}{2.0}
                   & \budgetstat{41.4}{0.3} & \budgetstat{56.7}{1.6}
                   & \budgetstat{52.4}{0.4} & \budgetstat{60.3}{2.0} \\
        \addlinespace[2pt]
        & Baseline & \budgetstat{78.1}{0.2} & \budgetstat{56.1}{0.7}
                   & \budgetstat{66.1}{0.6} & \budgetstat{49.1}{0.7}
                   & \budgetstat{67.5}{0.5} & \budgetstat{49.8}{0.3}
                   & \budgetstat{78.0}{0.2} & \budgetstat{54.9}{0.7} \\
        \rowcolor{budgetaccent!8}
        \cellcolor{white}\multirow{-2}{*}{$32$}
        & \SpatialSAE & \budgetstat{57.5}{0.1} & \budgetstat{63.0}{0.5}
                   & \budgetstat{51.1}{0.2} & \budgetstat{52.3}{2.0}
                   & \budgetstat{51.6}{0.8} & \budgetstat{54.3}{2.0}
                   & \budgetstat{58.8}{0.4} & \budgetstat{56.6}{0.9} \\
        \addlinespace[2pt]
        & Baseline & \budgetstat{83.3}{0.1} & \budgetstat{52.6}{0.3}
                   & \budgetstat{74.4}{0.6} & \budgetstat{44.4}{0.4}
                   & \budgetstat{76.6}{0.5} & \budgetstat{43.7}{0.5}
                   & \budgetstat{83.5}{0.2} & \budgetstat{52.3}{0.8} \\
        \rowcolor{budgetaccent!8}
        \cellcolor{white}\multirow{-2}{*}{$64$}
        & \SpatialSAE & \budgetstat{61.1}{0.3} & \budgetstat{61.0}{2.3}
                   & \budgetstat{60.5}{0.7} & \budgetstat{48.0}{1.0}
                   & \budgetstat{58.6}{0.2} & \budgetstat{51.2}{1.4}
                   & \budgetstat{63.8}{0.3} & \budgetstat{55.9}{1.1} \\
        \addlinespace[2pt]
        & Baseline & \budgetstat{87.3}{0.1} & \budgetstat{48.3}{0.5}
                   & \budgetstat{80.3}{0.4} & \budgetstat{37.0}{0.6}
                   & \budgetstat{82.7}{0.4} & \budgetstat{36.7}{0.5}
                   & \budgetstat{87.4}{0.3} & \budgetstat{48.3}{1.0} \\
        \rowcolor{budgetaccent!8}
        \cellcolor{white}\multirow{-2}{*}{$128$}
        & \SpatialSAE & \budgetstat{64.9}{0.1} & \budgetstat{58.9}{0.7}
                   & \budgetstat{66.1}{0.5} & \budgetstat{44.8}{3.8}
                   & \budgetstat{62.6}{0.5} & \budgetstat{52.3}{1.9}
                   & \budgetstat{68.4}{0.3} & \budgetstat{55.0}{0.9} \\
        \addlinespace[2pt]
        & Baseline & \budgetstat{91.2}{0.1} & \budgetstat{48.0}{0.4}
                   & \budgetstat{85.4}{0.2} & \budgetstat{27.8}{0.4}
                   & \budgetstat{88.3}{0.3} & \budgetstat{29.4}{0.4}
                   & \budgetstat{91.4}{0.4} & \budgetstat{48.3}{1.0} \\
        \rowcolor{budgetaccent!8}
        \cellcolor{white}\multirow{-2}{*}{$256$}
        & \SpatialSAE & \budgetstat{69.1}{0.1} & \budgetstat{57.3}{1.6}
                   & \budgetstat{71.0}{0.2} & \budgetstat{45.4}{2.5}
                   & \budgetstat{67.4}{0.7} & \budgetstat{52.2}{1.7}
                   & \budgetstat{73.9}{0.3} & \budgetstat{54.8}{1.2} \\
        \bottomrule
    \end{tabularx}
    \caption{\textbf{Sparsity-budget ablation on \ADETwentyK.}
    We measure reconstruction with $R^2$ and concept recovery with Normalized
    F1, relative to the corresponding linear probe. Entries are mean $\pm$
    standard deviation over five seeds and are reported as percentages. We use
    an expected sparsity budget of $16$: compared with a budget of $8$, it
    improves baseline reconstruction for every sparsifier while retaining at
    least $98\%$ of each baseline's best Normalized F1, and is itself optimal
    for BatchTopK. \SpatialSAE has higher Normalized F1 throughout the sweep.}
    \label{tab:app-budget-sweep}
\end{table}

\paragraph{Software and compute.}
We extend the \texttt{overcomplete} library (version 0.3.0) to implement
\SpatialSAE.\footnote{\url{https://github.com/KempnerInstitute/overcomplete}}
The environment uses \texttt{Python} 3.12.12, \texttt{PyTorch} 2.12.1 with \texttt{CUDA} 12.6,
\texttt{torchvision} 0.27.1, \texttt{datasets} 3.6.0, and \texttt{timm} 1.0.27.
Each run uses a single \texttt{NVIDIA A100-SXM4 GPU} with $80$ GB of memory. 
On \ADETwentyK, ten training epochs take approximately
four minutes for static and convolutional spatial models, ten minutes
for context attention, and two to three minutes for baselines.
On \COCOStuff, context attention takes $28$--$35$ minutes and linear
baselines take six to eight minutes. These times cover training epochs.

\subsection{Implementation Details}
\label{app:real-impl}
\label{app:real-models}

\paragraph{Optimization.}
We train all models for ten epochs with Adam, learning rate $10^{-3}$,
$\beta_1=0.9$, $\beta_2=0.999$, and no weight decay. The learning rate
increases linearly from zero over the first $100$ steps, then decays to
$10^{-5}$ with a cosine schedule. We clip the global gradient norm at
$1$. Baselines receive batches of $4096$ shuffled patch activations;
\SpatialSAE receives $16$ complete patch grids, also totaling $4096$
patches. Baselines use float16 autocasting, while spatial models use bfloat16 to
avoid overflow in the Laplacian term; steps with a non-finite loss are
skipped. Data loading uses four workers, but the order in which batches
are consumed is fixed by the seed and independent of worker scheduling:
shards are prefetched concurrently and yielded in a precomputed shuffled
order, and images within a shard are shuffled with a seeded generator.
We do not enforce deterministic GPU kernels.
We use the final checkpoint from each run, without early stopping or
selection across epochs. Configuration sweeps use seed $0$, and the
selected models are retrained with seeds $0$--$4$.

\paragraph{Encoders and dictionaries.}
Encoders and dictionaries are the \texttt{MLPEncoder} and
\texttt{DictionaryLayer} classes of \texttt{overcomplete}, which we
describe here for completeness. We instantiate the encoder with zero
hidden blocks (the \texttt{linear} preset), so it applies a single linear
projection followed by ReLU, after which the chosen sparsifier is applied.
Unless stated otherwise, all models use $c=3072$ dictionary atoms, which
are renormalized to unit $\ell_2$ norm on every forward pass.

For \SpatialSAE, the dictionary partition in \Cref{sec:spatialsae}
assigns $c_S$ atoms to the smooth stream and $c_R$ to the innovation stream,
with corresponding decoder blocks $\mathbf{D}_S$ and $\mathbf{D}_R$.
Here we write $\mathbf{S}\in\R^{p\times c_S}$ and
$\mathbf{R}\in\R^{p\times c_R}$ for the fields' active code blocks. These are
sparsified with stream budgets $k_S$ and $k_R$, enforced by the chosen
sparsifier exactly as $k$ is for the baseline.
Each stream has its own encoder,
$\mathbf{E}_S\in\R^{d\times c_S}$ and
$\mathbf{E}_R\in\R^{d\times c_R}$, applied to its respective input.
The resulting codes decode through the corresponding blocks
$\mathbf{D}_S$ and $\mathbf{D}_R$. To match the baselines,
the two streams share the total dictionary and budget,
$c_S+c_R=c$ and $k_S+k_R=k$, and atoms are allocated in proportion to
budget, $c_S/c=k_S/k$, with largest-remainder rounding. A $14{:}2$ budget
split therefore assigns $2688$ atoms to the smooth stream and $384$ to
the innovation stream. We sweep the seven splits $2{:}14$, $4{:}12$,
$6{:}10$, $8{:}8$, $10{:}6$, $12{:}4$, and $14{:}2$.

\paragraph{Controlling effective sparsity and dead atoms.}
Unless stated otherwise, we target an expected sparsity of $16$ active
features per patch; \Cref{tab:app-budget-sweep} reports the full sparsity-budget
sweep. Comparisons at a fixed budget require models to use a similar number
of active features per patch. TopK enforces the budget directly by
retaining the $k$ largest coefficients. BatchTopK applies a shared
threshold, set to the $(nk)$-th largest pre-activation in a batch of
$n$ patches; at evaluation, it uses an exponential moving average of
the training thresholds with momentum $0.9$.

For ReLU and JumpReLU, the number of active features depends on the
$\ell_1$ penalty, so the same penalty coefficient need not yield the
same sparsity across models. We therefore use a simplified variant of the target-$L_0$
controller of \citet{chanin2026synthsaebench}, without its gain
scheduling, to adjust the coefficient
$\lambda_0\eta$, with $\lambda_0=10^{-2}$. Let $\bar L_0$ be an
exponential moving average of the number of active features per patch,
with smoothing coefficient $0.99$. Starting at step $100$, we update
the multiplier after each optimization step as
\begin{equation*}
    \eta \leftarrow \eta\left[
        1+k_i\tanh\!\left(g\frac{\bar L_0-k}{k}\right)
    \right],
    \qquad k_i=3\times10^{-3},\quad g=10,
\end{equation*}
and clip $\eta$ to $[10^{-2},10^2]$. The penalty increases when
$\bar L_0>k$ and decreases when $\bar L_0<k$. For \SpatialSAE, each
stream has its own controller with target $k_S$ or $k_R$.
JumpReLU additionally learns a threshold for each atom, initialized
at $e^{-3}$, through a straight-through estimator with a Silverman
kernel of bandwidth $0.1$. In every final run targeting
$\mathbb{E}[L_0]=16$, the mean number
of active features per patch on the evaluation pool differs from the
target by at most one.

We also use the reanimation regularizer implemented in the
\texttt{overcomplete} library \citep{fel2025archetypal}
to encourage atoms that never fire in a batch to become active.
We subtract their summed pre-activations from the loss, normalized by
the batch size and dictionary width. The penalty has weight
$10^{-3}$ for each spatial stream and $5\times10^{-3}$ for baselines;
each stream uses its own width for normalization. Fewer than
$0.4\%$ of atoms are inactive on the evaluation pool in every final run
targeting $\mathbb{E}[L_0]=16$.

\begin{figure}[t]
\centering
\begin{minipage}[t]{0.49\linewidth}
\begin{algorithm}[H]
    \captionsetup{font=normalfont,labelfont=bf,aboveskip=1pt,belowskip=1pt}
    \caption{Static forward pass.}
    \label{alg:ssae-static}
    \begin{algorithmic}[1]
        \Require Patch activations $\mathbf{H} \in \mathbb{R}^{p \times d}$.
        \For{each patch $i$}
            \State $\mathbf{S}_i\gets\Call{Sparsify}{\mathbf{H}_i\mathbf{E}_S}$
            \State $\mathbf{R}_i\gets
                \Call{Sparsify}{(\mathbf{H}_i - \mathbf{S}_i\mathbf{D}_S)\mathbf{E}_R}$
            \State $\widehat{\mathbf{H}}_i\gets
                \mathbf{S}_i\mathbf{D}_S+\mathbf{R}_i\mathbf{D}_R$
        \EndFor
        \State \Return $\mathbf{S}$, $\mathbf{R}$, and $\widehat{\mathbf{H}}$
    \end{algorithmic}
\end{algorithm}
\end{minipage}\hfill
\begin{minipage}[t]{0.49\linewidth}
\begin{algorithm}[H]
    \captionsetup{font=normalfont,labelfont=bf,aboveskip=1pt,belowskip=1pt}
    \caption{Convolutional forward pass.}
    \label{alg:ssae-conv}
    \begin{algorithmic}[1]
        \Require Patch activations $\mathbf{H} \in \mathbb{R}^{p \times d}$.
        \For{each patch $i$}
            \State $\widetilde{\mathbf{H}}_i\gets
                \sum_{j\in\mathcal{N}(i)}\mathbf{W}_{j-i}\odot\mathbf{H}_{j}$
            \State $\mathbf{S}_i\gets
                \Call{Sparsify}{\widetilde{\mathbf{H}}_i\mathbf{E}_S}$
            \State $\mathbf{R}_i\gets
                \Call{Sparsify}{(\mathbf{H}_i - \mathbf{S}_i\mathbf{D}_S)\mathbf{E}_R}$
            \State $\widehat{\mathbf{H}}_i\gets
                \mathbf{S}_i\mathbf{D}_S+\mathbf{R}_i\mathbf{D}_R$
        \EndFor
        \State \Return $\mathbf{S}$, $\mathbf{R}$, and $\widehat{\mathbf{H}}$
    \end{algorithmic}
\end{algorithm}
\end{minipage}
{\captionsetup{hypcap=false}
\caption*{\textbf{Controls for smooth-field estimation.}
The static control (left) removes neighbor-based estimation, while the
convolutional control (right) aggregates neighboring patches without the
content-dependent weighting used by context attention. Both retain the same
two-stream innovation and decoding structure, allowing the ablation to
isolate how the smooth field is estimated.}}
\end{figure}
\paragraph{Estimating the smooth field.}
The encoder in \Cref{sec:ssae-encoder} estimates the smooth content
at each patch from its neighbors, without observing the target patch's
activation. It projects the neighbors' mean activation to form a query
and their individual activations to form keys and values. Attention
then combines this information into the smooth code. We call this
estimator \emph{context attention} in the ablations below.
We arrange the neighbors in $12$ fixed positions, add learned
relative-position embeddings to the keys and values, and use a learned
vector for positions outside the grid. The attention layer has four
heads and inner width $c_S/4$, with LayerNorm applied to its inputs.
Its output is added to the query, then rectified and sparsified.

We compare this estimator with two alternatives to examine the role
of neighboring activations and how they are combined. The \emph{static}
estimator obtains the smooth code directly by projecting and
sparsifying the target patch's own activation.
The \emph{convolutional} estimator applies a depthwise convolution
over $12$ neighbors with a learned kernel
$\mathbf{W}\in\R^{12\times d}$, before projecting and sparsifying
the result. The kernel weights each coordinate separately,
so the convolution combines spatial information without mixing coordinates.
We use zero padding and initialize all kernel entries to $1/12$.
The row $\mathbf{W}_{\delta}$ corresponds to relative offset $\delta$.
\Cref{alg:ssae-static,alg:ssae-conv} give both forward passes,
using the encoders and dictionary blocks defined above and the fixed
graph's edge set $\mathcal{E}$.

All three estimators use the same innovation step: the innovation
encoder receives $\mathbf{H}-\sg(\mathbf{S}\mathbf{D}_S)$, with
gradients stopped through the smooth reconstruction.

\paragraph{Spectral regularization and objective.}
We implement a width-normalized variant of the objective in
\Cref{def:spatial-sae}, using separate dictionary blocks for the smooth and
innovation streams. For activations $\mathbf{H}\in\R^{p\times d}$,
let $(\mathbf{S}_{\saeparams},\mathbf{R}_{\saeparams})
=F_{\saeparams}(\mathbf{H}\,;\,\mathbf{L})$ and
$\widehat{\mathbf{H}}_{\saeparams}
=\mathbf{S}_{\saeparams}\mathbf{D}_S+\mathbf{R}_{\saeparams}\mathbf{D}_R$.
We suppress the encoder subscript when defining the loss components.
The reconstruction and sparsity losses are
\begin{equation*}
    \begin{aligned}
        \mathcal{L}_{\mathrm{rec}}(\mathbf{H},\widehat{\mathbf{H}})
        &=\frac{1}{pd}\norm{\mathbf{H}-\widehat{\mathbf{H}}}_F^2,\\
        \mathcal{L}_{\ell_1}(\mathbf{S},\mathbf{R};\lambda_S,\lambda_R)
        &=\frac{\lambda_S}{p}\norm{\mathbf{S}}_1
         +\frac{\lambda_R}{p}\norm{\mathbf{R}}_1.
    \end{aligned}
\end{equation*}
Thus reconstruction is averaged over activation entries, and each
sparsity penalty is averaged over patches. For ReLU and JumpReLU,
the controller described above adjusts $\lambda_S$ and $\lambda_R$
independently. TopK and BatchTopK enforce sparsity through the encoder
and set $\lambda_S=\lambda_R=0$.
The normalizability result in \Cref{prop:mflrh-interaction-signs} applies to
the continuous prior with positive $\ell_1$ coefficients. The TopK and
BatchTopK implementations instead use encoder-enforced sparsity with the
same spatial regularizers.

Let $\mathbf{L}$ be the Laplacian of the fixed patch graph used
throughout the experiments, and define
$\widetilde{\mathbf{L}}=\lambda_{\max}(\mathbf{L})\mathbf{I}_p-\mathbf{L}$.
The spectral loss is
\begin{equation*}
    \mathcal{L}_{\mathrm{spec}}(\mathbf{S},\mathbf{R};\mathbf{L},\mu)
    =\mu\left[
        \frac{\Tr(\mathbf{S}^\top\mathbf{L}\mathbf{S})}{c_S}
        +\frac{\Tr(\mathbf{R}^\top\widetilde{\mathbf{L}}\mathbf{R})}{c_R}
    \right].
\end{equation*}
The two terms penalize high-frequency energy in the smooth stream
and low-frequency energy in the innovation stream, respectively,
with each normalized by its dictionary width. Here $\mu$ weights
these normalized terms; setting $\mu=0$ removes both penalties.
When $c_S\neq c_R$, normalization changes the relative weighting of the two
spectral penalties compared with \Cref{eq:spatialsae-objective}. These positive
rescalings preserve the attraction/repulsion signs and nonnegativity of the
quadratic terms.
We compute the eigendecomposition of $\mathbf{L}$ once and reuse it
for regularization and spectral metrics.

The complete training objective is
\begin{equation}
    \min_{\saeparams,\mathbf{D}}
    \frac{1}{\abs{\mathcal{X}}}
    \sum_{\mathbf{H}\in\mathcal{X}}\Bigl[
        \mathcal{L}_{\mathrm{rec}}(\mathbf{H},\widehat{\mathbf{H}}_{\saeparams})
        +\mathcal{L}_{\ell_1}(\mathbf{S}_{\saeparams},\mathbf{R}_{\saeparams};
            \lambda_S,\lambda_R)
        +\mathcal{L}_{\mathrm{spec}}(\mathbf{S}_{\saeparams},\mathbf{R}_{\saeparams};
            \mathbf{L},\mu)
    \Bigr]+\text{ reanim.}
    \label{eq:app-vision-objective}
\end{equation}

\subsection{Evaluation Metrics and Spectral Diagnostics}
\label{app:metrics}

Because the concepts that generate real vision-model activations are
unobserved, no single score can establish \SAE quality. We therefore evaluate
three complementary properties: reconstruction fidelity, alignment of
individual features with annotated concepts, and spatial coherence of the
complete code. We then use graph-frequency diagnostics to characterize 
the energy distribution of the learned codes relative to the true concept
masks.

\definecolor{rubricaccent}{HTML}{456B73}
\begin{rubric}[t]
\centering
\begin{tcolorbox}[
    enhanced,
    title={Monosemanticity rubric (integer score from 1 to 10)},
    fonttitle=\bfseries,
    coltitle=black,
    colbacktitle=rubricaccent!12,
    colback=white,
    colframe=rubricaccent,
    boxrule=0.7pt,
    arc=1.5mm,
    left=0.9em,
    right=0.9em,
    top=0.65em,
    bottom=0.65em,
    before skip=0.9em,
    after skip=0.9em
]
\small
\emph{How consistently is the active region the same object or concept
across the seven rows?}
\medskip

\renewcommand{\arraystretch}{1.13}
\begin{tabularx}{\linewidth}{@{}>{\bfseries}p{0.08\linewidth}X@{}}
10 & One clear concept; every row fits. \\
8--9 & One clear concept; all rows fit except one near miss of that same
concept. \\
6--7 & One concept covers most rows; the remainder are unrelated or
ambiguous. \\
4--5 & About half the rows share a concept, or two distinct groups have
comparable size. \\
2--3 & Only a loose thematic resemblance, or three or more distinct groups. \\
1 & No discernible commonality. \\
\end{tabularx}
\end{tcolorbox}
\caption{\textbf{Monosemanticity scoring rubric.}
Scores measure whether one coherent concept explains the highlighted regions
across the panel.}
\label{rubric:monosemanticity}
\end{rubric}

\paragraph{Evaluation protocol.}
All metrics are computed on held-out data. For an activation matrix
$\mathbf H\in\R^{n\times d}$, let $\widehat{\mathbf H}$ be its reconstruction
and let $\mathbf Z\in\R^{n\times c}$ be its sparse code. For \SpatialSAE,
$\mathbf Z$ concatenates the smooth and innovation codes, and feature-level
metrics consider features from both streams. We first aggregate each metric
over its evaluation samples and then average the resulting scores across
training seeds.

\paragraph{Reconstruction fidelity.}
We measure the fraction of activation variance preserved by the
reconstruction. For a batch of patch activations, let
$\boldsymbol\mu_{\mathbf H}=n^{-1}\mathbf H^\top\mathbf 1$ be the mean
activation and let
$\overline{\mathbf H}=\mathbf 1\boldsymbol\mu_{\mathbf H}^{\top}$ repeat that
mean at every patch. We define
\begin{equation}
    R^2(\mathbf H,\widehat{\mathbf H})
    =1-\frac{\norm{\mathbf H-\widehat{\mathbf H}}_F^2}
             {\norm{\mathbf H-\overline{\mathbf H}}_F^2}.
    \label{eq:app-real-r2}
\end{equation}
We compute this coefficient separately for each evaluation batch, average it
across batches, and report the result as a percentage.

\paragraph{Concept alignment.}
A feature aligns with a concept when its activation mask agrees with the
concept's segmentation mask. For an arbitrary image, let
$\mathbf Y\in\{0,1\}^{p\times r}$ contain the patch-aligned masks of the $r$
semantic classes and let $\mathbf Z\in\R^{p\times c}$ be its sparse code. We
binarize all feature maps at their native firing threshold,
$\mathbf A=\Ind\{\mathbf Z>0\}$. For class $\ell$ and feature $j$, we collect
the columns $\mathbf Y_{:\ell}$ and $\mathbf A_{:j}$ across the images that
contain class $\ell$, concatenate them across patches, and compute
\begin{equation}
    \mathrm{F1}_{\ell j}=\frac{2\mathrm{TP}_{\ell j}}
    {2\mathrm{TP}_{\ell j}+\mathrm{FP}_{\ell j}+\mathrm{FN}_{\ell j}}.
    \label{eq:app-concept-f1}
\end{equation}
Each class is represented by its best-aligned feature,
$j_\ell^\star=\argmax_j\mathrm{F1}_{\ell j}$. Concept F1 macro-averages these
class-best scores:
\begin{equation}
    \operatorname{Concept\ F1}
    =\frac{100}{\abs{\mathcal C}}
      \sum_{\ell\in\mathcal C}\mathrm{F1}_{\ell j_\ell^\star}.
    \label{eq:app-concept-f1-aggregate}
\end{equation}
Here $\mathcal C$ is the set of annotated classes represented in the
evaluation data. We refer to this aggregate as concept F1, or simply F1 in
figures and tables.

\begin{rubric}[t]
\centering
\begin{tcolorbox}[
    enhanced,
    title={Map-quality rubric (integer score from 1 to 10)},
    fonttitle=\bfseries,
    coltitle=black,
    colbacktitle=rubricaccent!12,
    colback=white,
    colframe=rubricaccent,
    boxrule=0.7pt,
    arc=1.5mm,
    left=0.9em,
    right=0.9em,
    top=0.65em,
    bottom=0.65em,
    before skip=0.9em,
    after skip=0.9em
]
\small
\emph{How well does the activation map model the object or concept named in
the description, across the seven rows?} Judge \textbf{coverage} of the
whole object, \textbf{tightness} at its extent, and \textbf{consistency}
across rows jointly.
\medskip

\renewcommand{\arraystretch}{1.13}
\begin{tabularx}{\linewidth}{@{}>{\bfseries}p{0.08\linewidth}X@{}}
10 & In essentially every row, the active region covers the object completely
and tightly. \\
8--9 & The object is covered well in most rows, with only small spill or
gaps. \\
6--7 & The map is on the object in most rows but is clearly partial or loose. \\
4--5 & The map is on the object about half the time and is otherwise diffuse,
scattered, or on other things. \\
2--3 & The object is visible, but the map mostly misses it or covers the whole
image. \\
1 & No usable localization, such as a whole-image response, border artifact,
or noise; alternatively, no consistent object exists to model. \\
\end{tabularx}
\medskip

\footnotesize\emph{The target is the concept inferred by the rater, not the
ground-truth segmentation class. When no consistent concept exists, map
quality is judged against whatever the description nevertheless identifies.}
\end{tcolorbox}
\caption{\textbf{Map-quality scoring rubric.}
Scores measure the coverage, tightness, and consistency of the highlighted
regions relative to the concept inferred by the rater.}
\label{rubric:map-quality}
\end{rubric}
\paragraph{Normalized concept F1.}
Concept F1 depends partly on whether the frozen backbone makes a class
linearly accessible at all. We therefore train a linear probe on the same
backbone activations and compute its class-wise F1 against the same masks.
For the classes that are reliably decoded by the probe,
$\mathcal C_{\mathrm{dec}}=
\{\ell:\mathrm{F1}^{\mathrm{probe}}_\ell>0.6\}$, we report the ratio between the
class-best feature and probe scores:
\begin{equation}
    \operatorname{Normalized\ F1}
    =\frac{100}{\abs{\mathcal C_{\mathrm{dec}}}}
      \sum_{\ell\in\mathcal C_{\mathrm{dec}}}
      \frac{\mathrm{F1}_{\ell j_\ell^\star}}
           {\mathrm{F1}^{\mathrm{probe}}_\ell}.
    \label{eq:app-normalized-f1}
\end{equation}
Normalized F1, sometimes abbreviated \emph{Norm.\ F1} in figures and tables, thus
measures how much of a class's linearly decodable information is captured by
a single \SAE feature.

\paragraph{Region coherence.}
Concept F1 assesses individual features; region coherence asks whether the
complete sparse code respects the spatial extent of annotated concepts. We
align each segmentation mask to the patch grid, discard unlabeled and
boundary patches, and group the remaining patches by semantic region. For an
arbitrary image, let $\mathcal P$ be the retained patch indices and let
$\mathbf Z_{\mathcal P,:}$ be their code matrix. For \SpatialSAE, this matrix
concatenates the two streams. We center its columns across patches,
\begin{equation}
    \widetilde{\mathbf Z}
    =\left(\mathbf I_{\abs{\mathcal P}}
      -\frac{1}{\abs{\mathcal P}}\mathbf 1\mathbf 1^\top\right)
      \mathbf Z_{\mathcal P,:}.
\end{equation}
We compare every pair of nonzero rows of $\widetilde{\mathbf Z}$ by cosine
similarity and group pairs by their grid distance. Within each image, we
average similarities at distance $\delta$ separately for pairs from the same
and different semantic regions. We then average both quantities across
images, denoting the results by $s_{\mathrm{within}}(\delta)$ and
$s_{\mathrm{between}}(\delta)$. Region coherence is their average
matched-distance difference,
\begin{equation}
    \operatorname{Coherence}
    =\frac{100}{\abs{\Delta}}\sum_{\delta\in\Delta}
      \left[s_{\mathrm{within}}(\delta)
      -s_{\mathrm{between}}(\delta)\right],
    \label{eq:app-region-coherence}
\end{equation}
where $\Delta$ contains the grid distances represented by both pair types.
We average within-region similarities equally across regions and
between-region similarities equally across region pairs before averaging
over images. Matching distance removes the confound that patches from the
same region are generally closer; a larger score indicates that the code more
faithfully preserves semantic regions.

\paragraph{Concept-relative code spectra
(\Cref{fig:wild-main}b and \Cref{fig:app-coco-results}b).}
The preceding metrics measure \SAE performance; we next examine how code
energy is distributed across spatial scales. Our hypothesis predicts a
specific spectral signature of concept fragmentation: when a spatially
coherent concept is split across patch-local features, the code allocates
more energy to high graph frequencies than the corresponding concept masks.
Comparing the two spectra therefore tests whether the learned code reflects
the spatial scales present in the annotated concepts.

Let $\mathbf L=\mathbf U\boldsymbol\Lambda\mathbf U^\top$ be the
eigendecomposition of the combinatorial Laplacian of the patch graph, with
graph-Fourier modes ordered by increasing eigenvalue. For an arbitrary image,
let $\mathbf Z\in\R^{p\times c}$ be its code matrix and let
$\mathbf M\in\R^{p\times r}$ contain its patch-aligned concept masks. The
graph-Fourier energy across all features and concepts is
\begin{align}
    \mathbf e_{\mathrm{code}}(\mathbf Z)
      &=\operatorname{diag}\!\left(
        \mathbf U^\top\mathbf Z\mathbf Z^\top\mathbf U\right),
      \label{eq:app-code-concept-spectrum}\\
    \mathbf e_{\mathrm{concept}}(\mathbf M)
      &=\operatorname{diag}\!\left(
        \mathbf U^\top\mathbf M\mathbf M^\top\mathbf U\right).
      \label{eq:app-concept-spectrum}
\end{align}
Here $\operatorname{diag}$ returns the diagonal as a vector, so each entry
records the energy assigned to one graph-Fourier mode. We compute these
vectors for every evaluation image and sum them componentwise across images.
Denote the resulting dataset-level vectors by $\mathbf e_{\mathrm{code}}$ and
$\mathbf e_{\mathrm{concept}}$; the former includes features from both
\SpatialSAE streams. After discarding their constant-mode entries, let
$\mathbf e_{\mathrm{code}}^\circ$ and
$\mathbf e_{\mathrm{concept}}^\circ$ denote the remaining vectors. We
normalize them to unit mass,
\begin{equation*}
    \boldsymbol\pi_{\mathrm{code}}
    =\frac{\mathbf e_{\mathrm{code}}^\circ}
           {\mathbf 1^\top\mathbf e_{\mathrm{code}}^\circ},
    \qquad
    \boldsymbol\pi_{\mathrm{concept}}
    =\frac{\mathbf e_{\mathrm{concept}}^\circ}
           {\mathbf 1^\top\mathbf e_{\mathrm{concept}}^\circ}.
\end{equation*}
Let $\mathcal S$ denote the common smoothing operator applied to these
discrete distributions. The plotted concept-relative spectrum is the vector
\begin{equation}
    \mathbf r
    =\log_2\!\left(
      \mathcal S(\boldsymbol\pi_{\mathrm{code}})
      \oslash
      \mathcal S(\boldsymbol\pi_{\mathrm{concept}})\right),
    \label{eq:app-relative-spectral-density}
\end{equation}
where both $\log_2$ and $\oslash$ act elementwise. The entry associated with
mode $k$ is plotted at normalized rank $q_k=k/p$. Zero indicates that codes
and concept masks allocate the same fraction of nonconstant energy to a
spatial frequency; positive values indicate excess code energy, and negative
values indicate under-representation. We average $\mathbf r$ across training
seeds and use its componentwise standard deviation for the uncertainty bands.

\paragraph{Per-feature frequency and dictionary visualization
(\Cref{fig:wild-main}d and \Cref{fig:app-coco-results}d).}
The concept-relative spectrum summarizes the code as a whole, but does not
show how spatial frequency is distributed across individual features or
decoder directions. We therefore assign each feature an empirical graph
frequency. For an arbitrary image with code matrix
$\mathbf Z\in\R^{p\times c}$, we first remove each feature's spatial mean,
\begin{equation}
    \widetilde{\mathbf Z}
    =\left(\mathbf I-\frac{1}{p}\mathbf 1\mathbf 1^\top\right)
      \mathbf Z.
\end{equation}
Its per-feature Dirichlet and nonconstant energies are the vectors
\begin{equation}
    \mathbf d(\widetilde{\mathbf Z})
    =\operatorname{diag}(\widetilde{\mathbf Z}^{\top}
      \mathbf L\widetilde{\mathbf Z}),
    \qquad
    \mathbf a(\widetilde{\mathbf Z})
    =\operatorname{diag}(\widetilde{\mathbf Z}^{\top}
      \widetilde{\mathbf Z}).
\end{equation}
We compute both vectors for every evaluation image and sum them componentwise,
obtaining the aggregated vectors $\mathbf d_{\mathrm{agg}}$ and
$\mathbf a_{\mathrm{agg}}$. For features with nonzero accumulated energy, and
with $\lambda_{\max}$ denoting the largest eigenvalue of the patch-graph
Laplacian, the frequency vector is
\begin{equation}
    \mathbf f
    =\frac{1}{\lambda_{\max}}
      \mathbf d_{\mathrm{agg}}\oslash\mathbf a_{\mathrm{agg}},
    \label{eq:app-atom-frequency}
\end{equation}
where $\oslash$ denotes elementwise division. Each entry is the normalized,
activation-energy-weighted mean graph frequency of one feature. Low values
indicate spatially smooth activation maps, whereas high values indicate
rapidly varying maps.

\subsection{Automated Interpretability}
\label{app:real-autointerp}

Concept F1 measures alignment with a predefined semantic class, but does not
establish whether a feature represents one coherent concept across its strongest
examples or whether its activation maps localize that concept cleanly. We
therefore complement concept F1 with an automated interpretability evaluation on
both \ADETwentyK and \COCOStuff. The two evaluations use the same feature
selection, panel construction, anonymization, rater, and scoring rubrics; they
differ only in their class cohorts and held-out image pools.

\begin{table}[t]
\centering
\small
\definecolor{ablationaccent}{HTML}{456B73}
\setlength{\tabcolsep}{5pt}
\renewcommand{\arraystretch}{1.3}
\begin{tabularx}{\linewidth}{c*{4}{>{\centering\arraybackslash}X}}
\toprule
& \multicolumn{3}{c}{\textbf{Two Streams} $-$ \textbf{Single Stream}}
& \textbf{Norm.\ F1} \\
\cmidrule(lr){2-4}\cmidrule(l){5-5}
$\mu$ & $\Delta R^2$ & $\Delta$ Norm.\ F1
& $\Delta$ Region coh. & Win rate $\uparrow$ \\
\midrule
\rowcolor{ablationaccent!10}
$10^{-3}$ & $-0.80\pm0.63$ & $\mathbf{+6.12}\pm1.86$
& $-2.90\pm4.07$ & $\mathbf{100\%}$ \\
$10^{-2}$ & $-7.25\pm3.19$ & $+0.23\pm3.33$
& $-4.49\pm3.79$ & $57\%$ \\
$10^{-1}$ & $-5.12\pm2.59$ & $-5.60\pm4.30$
& $+5.51\pm12.21$ & $14\%$ \\
\bottomrule
\end{tabularx}
\caption{\textbf{Effect of code decomposition at matched regularization
strength.} Static TopK on \ADETwentyK, $\mathbb{E}[L_0]=16$.
The single-stream control applies the attractive Laplacian penalty to the
complete code; the two-stream model additionally represents an innovation
field with the repulsive penalty. Differences are two-stream minus
single-stream at the same $\mu$, reported as mean $\pm$ sample standard
deviation across seven budget allocations and two initializations each.
$R^2$ and Normalized F1 changes are in percentage points; coherence changes
are multiplied by $100$. The last column reports the percentage with higher
Normalized F1, rounded to the nearest integer. Shading marks the selected
regularization strength.}
\label{tab:abl-static}
\end{table}
\paragraph{Panel construction and scoring.}
For both datasets, we evaluate the seed-$0$ baseline and \SpatialSAE runs for
all four sparsifiers. For each semantic class and model family, we select that
model's own class-best feature $j_\ell^\star$ from \Cref{eq:app-concept-f1-aggregate}, taken from its full code (for \SpatialSAE, the concatenated smooth and innovation streams). We
then rank the held-out images by the feature's activation summed over patches
and retain the seven highest-ranked images, ordered from strongest to weakest.

Each row of a panel places the $224\times224$ model input beside the same image
with its $16\times16$ activation map overlaid by nearest-neighbor upsampling.
Red opacity is proportional to activation after normalization by the maximum
activation in that image, and a dark outline marks the active region. Panels
contain no class names, captions, or scores and use randomized class
identifiers.

Claude Haiku 4.5 \citep{anthropic2025haiku45}, at low reasoning effort, views
each panel and returns a short description, an integer \emph{monosemanticity}
score, and an integer \emph{map-quality} score. The rater is instructed to
describe the highlighted regions rather than the whole image and is not told
the semantic class, feature F1, model family, or sparsifier. Monosemanticity
measures whether one inferred concept explains the active regions across the
seven images. Map quality measures the coverage, tightness, and consistency of
those regions relative to the concept inferred by the rater. It is therefore a
self-consistent localization score rather than agreement with the
ground-truth segmentation mask. Both scores use an integer $1$--$10$ scale;
the scoring criteria are given in
\Cref{rubric:monosemanticity,rubric:map-quality}. Full panel-generation and
rater instructions are included in the supplementary material.

\paragraph{Dataset cohorts.}
For \ADETwentyK, we begin with all $150$ semantic classes and exclude the $18$
classes whose best-feature F1 is zero for either model family. This leaves
$132$ classes for each sparsifier and produces
$132\times4\times2=1{,}056$ panels, assigned to $132$ eight-panel rater
instances.

For \COCOStuff, we use all $171$ official semantic classes. Every class has
positive best-feature F1 for both model families under every sparsifier, so no
classes are excluded. This produces $171\times4\times2=1{,}368$ panels,
assigned to $171$ eight-panel rater instances. The panels are drawn from a
held-out activation-cache split of $6{,}218$ images, constructed as a seeded
random $5\%$ subset of the $152$ COCO train2017 shards. This split is not the
official COCO val2017 split. Inputs and activation maps retain the same
$224\times224$ and $16\times16$ geometry used for \ADETwentyK.

\paragraph{Assignment and quality control.}
Within each dataset, panels are assigned to rater instances using a seeded
permutation of the class identifiers. \COCOStuff uses its own seed-$0$
permutation, independent of the \ADETwentyK permutation. Every eight-panel
instance contains distinct classes and therefore never includes both model
families for the same class.

The intended protocol assigns one rating to each panel. Every \ADETwentyK
panel received exactly one rating. For \COCOStuff, the procedure produced
$1{,}370$ valid ratings for $1{,}368$ panels because two panels were
inadvertently rated twice. Each duplicate pair received identical integer
scores, and we average the two ratings before aggregation. No malformed
records were produced. The rater was forbidden from accessing other files or
outputs; although the storage paths contained model and sparsifier tokens,
their meanings were not provided.

\paragraph{Aggregation and uncertainty.}
For each displayed bar, we average over the corresponding class cohort: $132$
classes for \ADETwentyK and $171$ classes for \COCOStuff. We construct $95\%$
percentile intervals from $2{,}000$ class-bootstrap resamples, performed
independently for each model family. These intervals quantify variation across
semantic classes, not variation across training seeds, ratings, or raters.
Scores retain their native $1$--$10$ scale, although the figure axes begin at
zero for display. The resulting \ADETwentyK and \COCOStuff evaluations are
reported in \Cref{fig:wild-main}g--h and
\Cref{fig:app-coco-results}g--h, respectively.

\begin{table}[t]
\centering
\footnotesize
\setlength{\tabcolsep}{2pt}
\renewcommand{\arraystretch}{2.2}
\newcommand{\ablstat}[3]{{\renewcommand{\arraystretch}{0.9}\begin{tabular}[c]{@{}c@{}}$#1$\\[-1pt]{\scriptsize\color{black!65}$[#2,#3]$}\end{tabular}}}
\begin{tabular*}{\linewidth}{@{\extracolsep{\fill}}cccccccc@{}}
\toprule
& \textbf{Transition}
& \multicolumn{2}{c}{\textbf{Reconstruction}}
& \multicolumn{2}{c}{\textbf{Normalized F1}}
& \multicolumn{2}{c}{\textbf{Coherence}} \\
\midrule
\multirow{3}{*}{\rotatebox[origin=c]{90}{\shortstack{\textbf{Regularization}\\\textbf{Strength $\mu$}}}}
& $0\to10^{-3}$
& \ablstat{-5.52}{-5.77}{-5.09} & $0\%$
& \ablstat{+6.53}{+4.77}{+6.94} & $\mathbf{100\%}$
& \ablstat{+2.51}{-2.69}{+4.10} & $86\%$ \\
& $10^{-3}\to10^{-2}$
& \ablstat{-25.91}{-29.85}{-23.46} & $0\%$
& \ablstat{+6.75}{+2.60}{+8.61} & $86\%$
& \ablstat{+0.13}{-1.23}{+2.10} & $64\%$ \\
& $10^{-2}\to10^{-1}$
& \ablstat{-27.94}{-28.20}{-27.23} & $0\%$
& \ablstat{-43.14}{-47.46}{-35.75} & $0\%$
& \ablstat{+16.04}{+7.95}{+24.17} & $\mathbf{100\%}$ \\
\cmidrule(lr){2-8}
\multirow{3}{*}{\rotatebox[origin=c]{90}{\shortstack{\textbf{Smooth Field}\\\textbf{Estimator}}}}
& $\mathrm{S}\to\mathrm{C}$
& \ablstat{-3.52}{-6.90}{-1.78} & $0\%$
& \ablstat{-0.34}{-1.79}{+1.01} & $43\%$
& \ablstat{+2.23}{+1.21}{+3.00} & $86\%$ \\
& $\mathrm{S}\to\mathrm{CA}$
& \ablstat{-2.55}{-5.16}{-1.48} & $0\%$
& \ablstat{+1.95}{+0.09}{+5.11} & $71\%$
& \ablstat{+2.69}{+2.14}{+3.88} & $93\%$ \\
& $\mathrm{C}\to\mathrm{CA}$
& \ablstat{+0.85}{+0.42}{+1.55} & $79\%$
& \ablstat{+2.67}{+1.51}{+5.23} & $86\%$
& \ablstat{+0.79}{+0.16}{+0.94} & $79\%$ \\
\bottomrule
\end{tabular*}
\caption{\textbf{Effects of regularization and smooth-field estimation.}
TopK on \ADETwentyK, $\mathbb{E}[L_0]=16$. Differences are destination minus
source, matched by budget allocation and initialization ($14$ pairs).
For each metric, the left column reports the median change ($\Delta$),
with its interquartile range beneath; the right reports the win rate (WR),
the percentage of positive differences rounded to the nearest integer.
$R^2$ and Normalized F1 changes are in percentage
points; coherence changes are multiplied by $100$. Regularization
transitions use the static encoder; estimator transitions fix $\mu=10^{-3}$.
S: static; C: convolutional; CA: context attention.}
\label{tab:abl-transitions}
\end{table}

\subsection{Ablations, Model Selection, and Additional Results}
\label{app:real-ablations}

\SpatialSAE introduces four design choices whose contributions we isolate in
turn: the two-stream decomposition, the regularization strength, the
smooth-field estimator, and the allocation of the sparsity budget between the
two streams. Unless stated otherwise, these ablations use \ADETwentyK and a
TopK sparsifier with expected sparsity $16$ and dictionary width $c=3072$.
Each two-stream setting covers all seven budget allocations, with two parameter
initializations per allocation; the single-stream control has no allocation
parameter. We report $R^2$ and Normalized F1 as percentages and multiply region
coherence by $100$.

\begin{figure}[t]
    \centering
    \includegraphics[width=\linewidth,height=0.79\textheight,keepaspectratio]{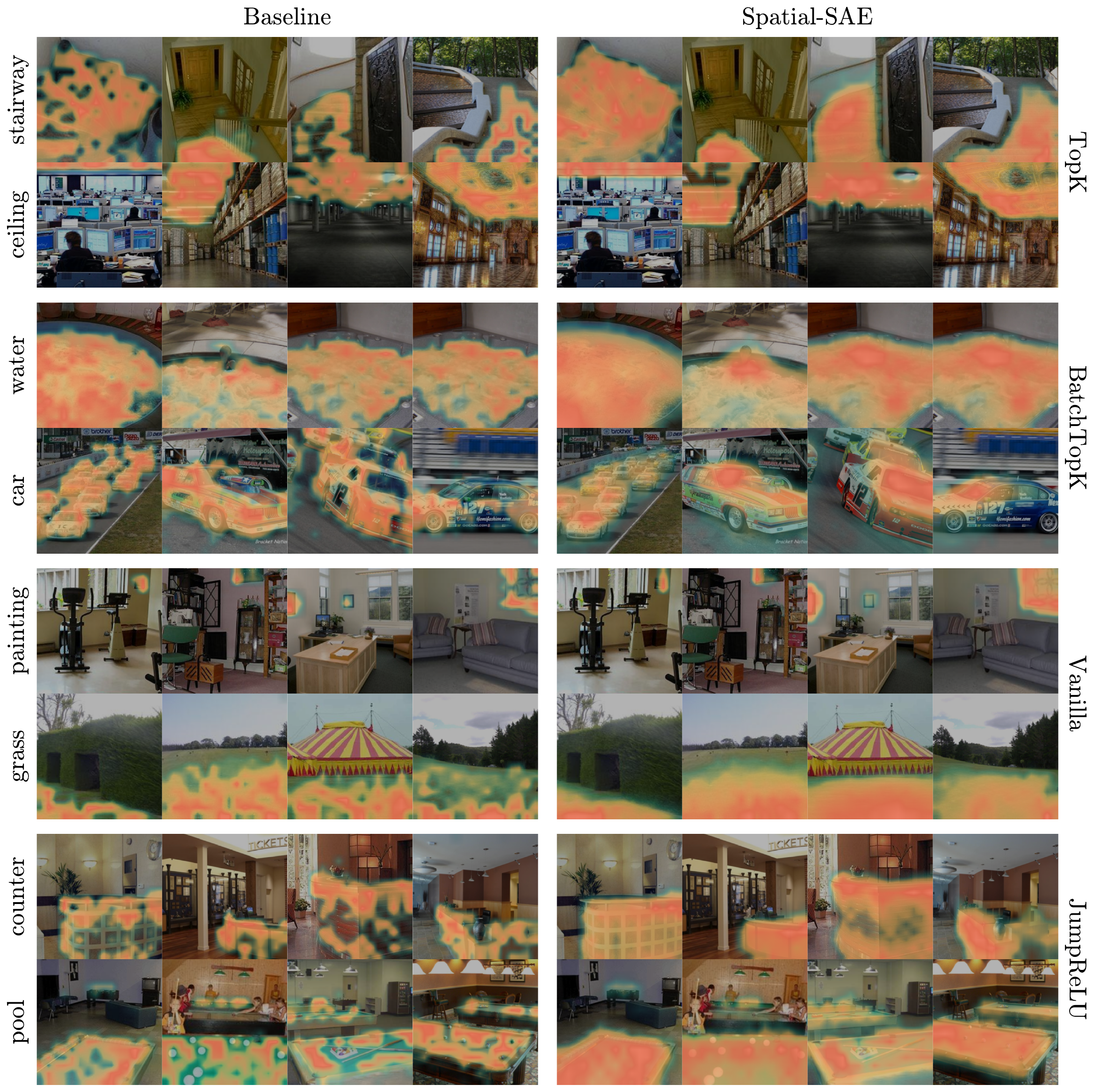}
    \caption{\textbf{Class-best feature maps on \ADETwentyK.}
    For each sparsifier (labels at right), two representative classes (labels
    at left) compare the baseline (left block) with \SpatialSAE (right
    block). Each model contributes its highest-concept-F1 feature for
    the labeled class. The four displayed images are selected for high
    per-image F1 of the baseline feature and then held fixed across models, so
    the selection deliberately favors the baseline. Warmer colors indicate
    stronger activation.}
    \label{fig:app-ade-maps}
\end{figure}
\paragraph{Effect of code decomposition under spectral regularization.}
The gains in \Cref{sec:wild-eval} could in principle come from smoothing the
code rather than from separating it into smooth and innovation fields. To
distinguish these explanations, \Cref{tab:abl-static} compares \SpatialSAE
with a single-stream control that applies the same attractive Laplacian penalty
to the entire code but has neither an innovation stream nor its repulsive
penalty. At the selected strength, $\mu=10^{-3}$, the decomposition increases
Norm.\ F1 by $6.12\pm1.86$ percentage points on average, and this difference is
positive in all $14$ matched comparisons. The corresponding mean differences
are $-0.80\pm0.63$ points in $R^2$ and $-2.90\pm4.07$ points in region
coherence. The latter is expected: the single-stream control applies the
attractive prior to its entire code and thus directly optimizes all of its
features for smoothness, whereas \SpatialSAE also maintains an innovation
stream under the repulsive prior. The purpose of the decomposition is therefore
not to make the entire code maximally smooth, but to improve concept alignment
while retaining a separate channel for high-frequency residual structure.

This advantage depends on the regularization regime; at $\mu=10^{-2}$, the
mean Normalized F1 difference narrows to $+0.23$ points, and at $10^{-1}$ it
favors the single-stream control. Even so, sweeping
$\mu\in\{10^{-3},10^{-2},10^{-1}\}$ reveals a higher attainable peak for the
decomposed model: the best single-stream setting reaches $66.0\%$ Normalized
F1 at $R^2=41.9\%$, whereas the best two-stream setting reaches $71.0\%$ at
$R^2=35.1\%$ (both at $\mu=10^{-2}$). The decomposition can therefore recover
more concept-aligned features, although reaching that peak incurs additional
reconstruction cost.

\paragraph{Effect of regularization strength.}
The upper block of \Cref{tab:abl-transitions} varies $\mu$ for the
two-stream model with the static estimator. Introducing spectral
regularization at $\mu=10^{-3}$ yields median changes of $+6.53$ points in
Normalized F1, $+2.51$ points in region coherence, and $-5.52$ points in
$R^2$. Across the matched pairs, the regularized model has higher Normalized
F1 in $100\%$ of comparisons and higher region coherence in $86\%$, but higher
$R^2$ in none. Increasing the strength from $10^{-3}$ to $10^{-2}$ yields
median changes of $+6.75$ points in Normalized F1, $+0.13$ points in region
coherence, and $-25.91$ points in $R^2$; the stronger setting is better on
these metrics in $86\%$, $64\%$, and $0\%$ of matched comparisons,
respectively. Pushing to $10^{-1}$ degenerates the model: the median changes
are $-43.14$ points in Normalized F1 and $-27.94$ points in $R^2$, with no
positive comparisons on either metric, even though region coherence increases
by a median $16.04$ points and is higher in every comparison.
We therefore use $\mu=10^{-3}$ in subsequent comparisons, choosing the first
large and fully consistent improvement in concept alignment without accepting
the much steeper reconstruction loss at $10^{-2}$.

\begin{figure}[t]
    \centering
    \includegraphics[width=\linewidth,height=0.79\textheight,keepaspectratio]{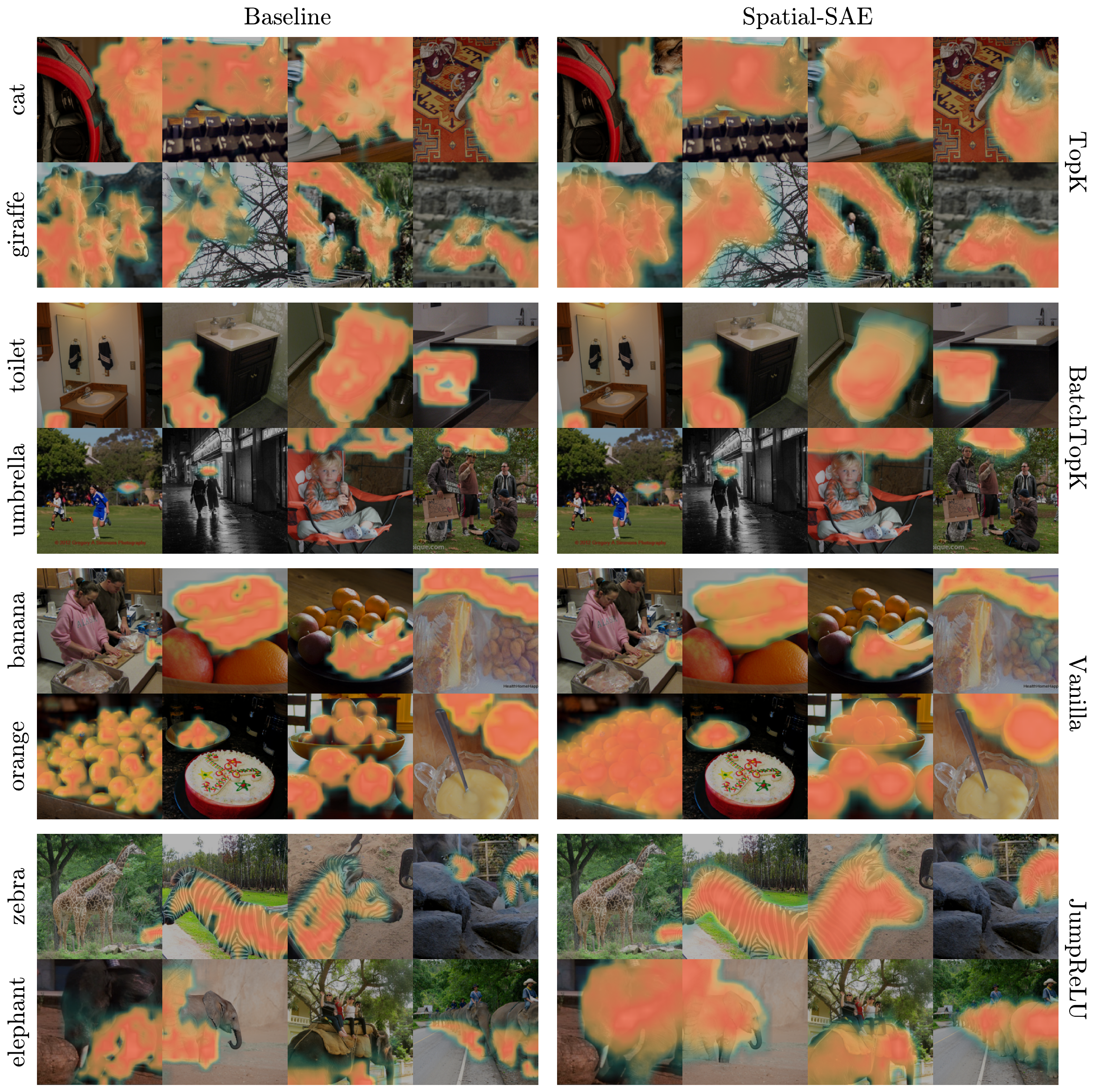}
    \caption{\textbf{Class-best feature maps on \COCOStuff.}
    For each sparsifier (labels at right), two representative classes (labels
    at left) compare the baseline (left block) with \SpatialSAE (right
    block). Each model contributes its highest-concept-F1 feature for
    the labeled class. Images are selected for high per-image F1 of the
    baseline feature and then held fixed across models, so the selection
    deliberately favors the baseline. Warmer colors indicate stronger activation.}
    \label{fig:app-coco-maps}
\end{figure}
\paragraph{Effect of the smooth-field estimator.}
The lower block of \Cref{tab:abl-transitions} compares the static,
convolutional, and context-attention estimators at $\mu=10^{-3}$. Relative to
the static estimator, convolution yields median changes of $-3.52$ points in
$R^2$, $-0.34$ points in Normalized F1, and $+2.23$ points in region
coherence. It outperforms the static estimator in $0\%$, $43\%$, and $86\%$ of
matched comparisons on these three metrics, respectively, so its benefit is
largely confined to coherence. Context attention offers a better tradeoff.
Relative to the static estimator, its median changes are $-2.55$ points in
$R^2$, $+1.95$ points in Normalized F1, and $+2.69$ points in coherence; it
outperforms the static estimator in $0\%$, $71\%$, and $93\%$ of comparisons,
respectively. Relative to convolution, all three median differences are
positive---$+0.85$, $+2.67$, and $+0.79$ points---and context attention is
better in $79\%$, $86\%$, and $79\%$ of comparisons. It therefore provides
the most consistent improvement in concept alignment and coherence, and we use
it in the final models.

\paragraph{Budget selection for each sparsifier.}
After fixing context attention and $\mu=10^{-3}$, we run a separate
seven-allocation sweep for each sparsifier on \ADETwentyK, evaluating every
split from $2{:}14$ through $14{:}2$ at total budget $k=16$. We choose from
each sweep in two stages: first retain every allocation whose Normalized F1 is
within $0.5$ percentage points of that sparsifier's best value, then select the
candidate with the highest $R^2$. The procedure selects
$k_S{:}k_R=14{:}2$ for TopK, $12{:}4$ for BatchTopK, $10{:}6$ for ReLU, and $14{:}2$
for JumpReLU. We retrain the selected configurations with seeds $0$--$4$; in
the \ADETwentyK sparsity-budget sweep, we preserve these selected proportions
while scaling the two stream budgets with the total budget.

\paragraph{Transfer to \COCOStuff.}
We transfer the model design and all training choices selected on \ADETwentyK
to \COCOStuff, including the context-attention estimator, $\mu=10^{-3}$,
dictionary width, total sparsity budget, per-sparsifier $k_S/k_R$ allocation,
and optimization procedure. Without further tuning on \COCOStuff, we retrain
each transferred configuration from scratch over five random seeds. Under
this direct transfer, \SpatialSAE improves Normalized F1 for all four
sparsifiers by $1.4$--$7.1$ points and region coherence by $6.5$--$9.0$ points
(\Cref{fig:app-coco-results}). Map quality and monosemanticity also improve in
every comparison, by $0.34$--$0.74$ and $0.14$--$0.73$, respectively, on their
native $1$--$10$ scales.

\paragraph{Qualitative feature maps.}
The examples in \Cref{fig:app-ade-maps,fig:app-coco-maps} are deliberately
selected in the baselines' favor. Each family contributes its own class-best
feature, but we select the displayed images specifically for high per-image F1
of the \emph{baseline}'s feature and then hold those images fixed for
\SpatialSAE. This is effectively reverse cherry-picking: the baseline is shown
on images where its feature already aligns well with the target, whereas
\SpatialSAE receives no such selection advantage.
Despite this baseline-favoring protocol, \SpatialSAE produces visibly more
complete and connected maps. On \ADETwentyK, baseline activations for extended
regions such as ceiling, water, grass, and pool often contain holes or break
into disconnected patches, while \SpatialSAE follows a larger continuous
portion of the region; the same contrast appears for bounded objects such as
cars and counters. On \COCOStuff, the spatial maps likewise follow more of
objects with varied shapes and scales, including cats, giraffes, fruit,
zebras, and elephants, instead of isolating a few discriminative fragments.
Because image selection is tilted toward the baselines, these qualitative
comparisons are a conservative visual counterpart to the aggregate
region-coherence and automated-interpretability gains.

\begin{figure}[t]
    \centering
    \includegraphics[width=\linewidth]{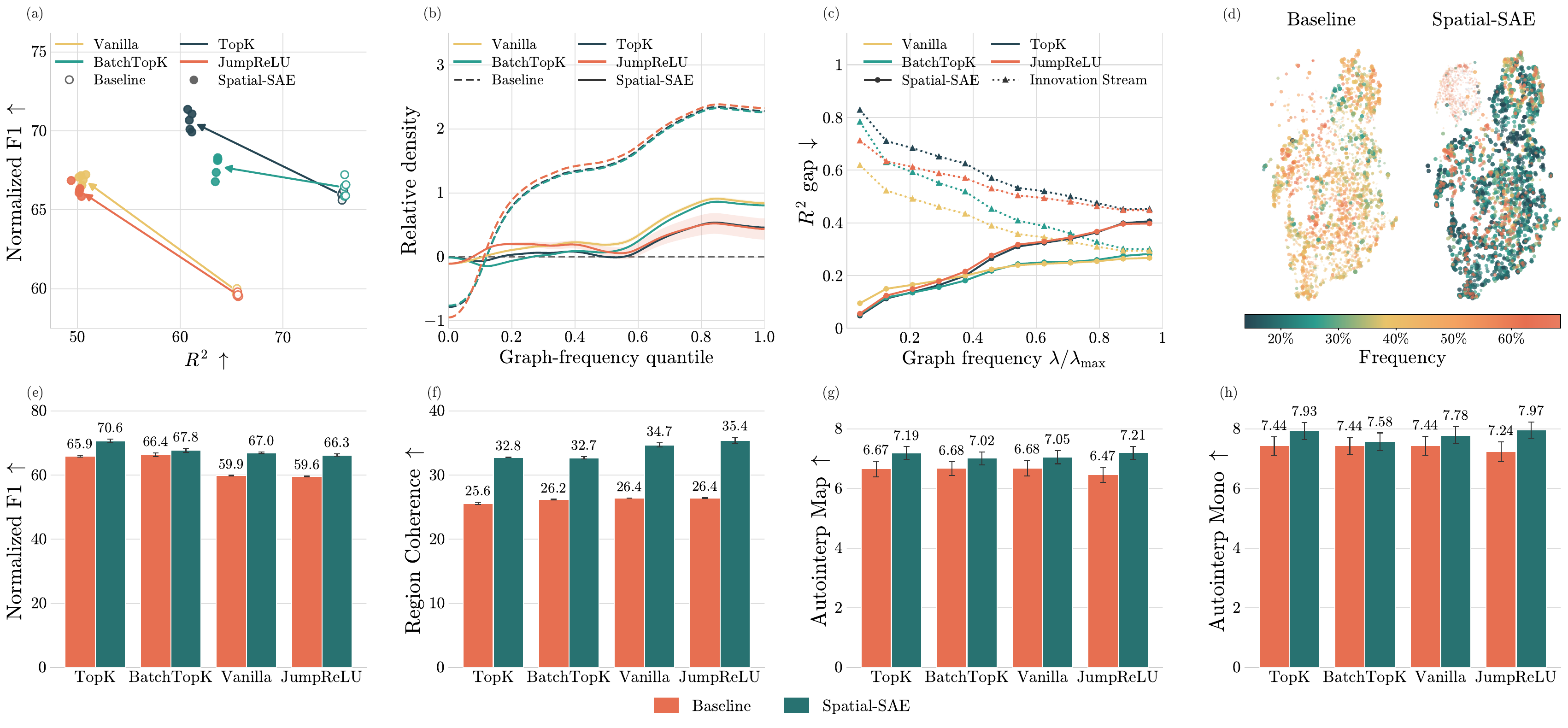}
    \caption{\textbf{Quantitative results on \COCOStuff.}
    (a) Normalized F1 against reconstruction $R^2$; points show training seeds
    and arrows connect model means.
    (b) Base-$2$ log ratio between code and concept-mask spectral density;
    zero denotes agreement, with dashed baseline and solid \SpatialSAE curves.
    (c) The baseline's $R^2$ advantage over the combined \SpatialSAE code
    (solid) and its innovation stream alone (dotted), resolved by graph
    frequency.
    (d) TopK decoder atoms in a joint \textsc{umap}, colored by activation-map
    frequency.
    (e--h) Normalized F1, region coherence, automated map quality, and
    monosemanticity. Error bars show one standard deviation across five training seeds in (e--f) and
    $95\%$ class-bootstrap intervals in (g--h).}
    \label{fig:app-coco-results}
\end{figure}

\paragraph{A shared spectral signature across datasets.}
\Cref{fig:wild-main}b and \Cref{fig:app-coco-results}b reveal a remarkably
stable failure mode. Across all four sparsifiers and on both datasets, the
baseline curves nearly coincide: relative to the concept masks, standard
\SAE codes underallocate spectral density to the lowest graph frequencies and
progressively overallocate it toward higher frequencies. At the upper end of
the spectrum, the baseline log ratios exceed $2$, corresponding to more than
four times the normalized spectral density of the concept masks. \SpatialSAE
moves every sparsifier substantially toward zero across the spectrum, largely
removing the low-frequency deficit and sharply reducing the high-frequency
excess. The correction survives tuning-free transfer to \COCOStuff, showing
that it is neither sparsifier- nor dataset-specific: spatial fragmentation has
a consistent spectral signature in standard \SAEs, and the structured prior
directly counteracts it.

\end{document}